\documentclass[twoside,letterpaper]{article}
\usepackage[accepted]{aistats2016}
\usepackage[pass]{geometry}
\usepackage{fancyhdr}
\usepackage{url}
\usepackage{mdwmath}
\usepackage{amssymb}
\usepackage{amsmath}
\usepackage{amsthm}
\usepackage{algorithm,algorithmic}
\usepackage[tight,footnotesize]{subfigure}
\usepackage{array}
\usepackage{amsfonts}
\usepackage{courier}
\usepackage{eqparbox}
\usepackage{enumerate}
\usepackage{framed}
\usepackage{lipsum}
\usepackage{mathrsfs}
\usepackage{multirow}
\usepackage{mdwtab}
\usepackage{stfloats}
\usepackage{enumitem}
\usepackage{cases}
\usepackage{epsfig}

\newtheorem{definition}{Definition}
\newtheorem{theorem}{Theorem}
\newtheorem{lemma}{Lemma}

\newtheorem{property}{Property}
\newtheorem{assumption}{Assumption}

\begin{document}

\twocolumn[

\aistatstitle{Tractable and Scalable Schatten Quasi-Norm Approximations for Rank Minimization}

\aistatsauthor{ Fanhua Shang \And Yuanyuan Liu \And James Cheng }
\aistatsaddress{ Department of Computer Science and Engineering, The Chinese University of Hong Kong } ]

\begin{abstract}
The Schatten quasi-norm was introduced to bridge the gap between the trace norm and rank function. However, existing algorithms are too slow or even impractical for large-scale problems. Motivated by the equivalence relation between the trace norm and its bilinear spectral penalty, we define two tractable Schatten norms, i.e.\ the bi-trace and tri-trace norms, and prove that they are in essence the Schatten-$1/2$ and $1/3$ quasi-norms, respectively. By applying the two defined Schatten quasi-norms to various rank minimization problems such as MC and RPCA, we only need to solve much smaller factor matrices. We design two efficient linearized alternating minimization algorithms to solve our problems and establish that each bounded sequence generated by our algorithms converges to a critical point. We also provide the restricted strong convexity (RSC) based and MC error bounds for our algorithms. Our experimental results verified both the efficiency and effectiveness of our algorithms compared with the state-of-the-art methods.
\end{abstract}

\section{Introduction}
The rank minimization problem has a wide range of applications in matrix completion (MC)~\cite{candes:emc}, robust principal component analysis (RPCA)~\cite{candes:rpca}, low-rank representation~\cite{liu:lrr}, multivariate regression~\cite{hsieh:nnm} and multi-task learning~\cite{argyriou:mtl}. To efficiently solve these problems, a principled way is to relax the rank function by its convex envelope~\cite{fazel:rmh, recht:nnm}, i.e., the trace norm (also known as the nuclear norm), which also leads to a convex optimization problem. In fact, the trace norm penalty is an $\ell_{1}$-norm regularization of the singular values, and thus it motivates a low-rank solution. However, \cite{fan:ve} pointed out that the $\ell_{1}$-norm over-penalizes large entries of vectors, and results in a biased solution. Similar to the $\ell_{1}$-norm case, the trace norm penalty shrinks all singular values equally, which also leads to over-penalize large singular values. In other words, the trace norm may make the solution deviate from the original solution as the $\ell_{1}$-norm does. Compared with the trace norm, although the Schatten-${p}$ quasi-norm for $0\!<\!p\!<\!1$ is non-convex, it gives a closer approximation to the rank function. Therefore, the Schatten-${p}$ quasi-norm minimization has attracted a great deal of attention in images recovery~\cite{lu:irsvm, lu:lrm}, collaborative filtering~\cite{nie:rmc} and MRI analysis~\cite{majumdar:mri}.

\cite{mohan:mrm} and \cite{lai:irls} proposed iterative reweighted lease squares (IRLS) algorithms to approximate associated Schatten-${p}$ quasi-norm minimization problems. In addition, \cite{lu:lrm} proposed an iteratively reweighted nuclear norm (IRNN) algorithm to solve non-convex surrogate minimization problems. In some recent work~\cite{marjanovic:mc, nie:lrmr, nie:rmc, lu:irsvm, lu:lrm}, the Schatten-${p}$ quasi-norm has been shown to be empirically superior to the trace norm. Moreover, \cite{zhang:ncmr} theoretically proved that the Schatten-${p}$ quasi-norm minimization with small $p$ requires significantly fewer measurements than the convex trace norm minimization. However, all existing algorithms have to be solved iteratively and involve singular value decomposition (SVD) or eigenvalue decomposition (EVD) in each iteration. Thus they suffer from high computational cost and are even not applicable for large-scale problems~\cite{shang:snm}.

In contrast, the trace norm has a scalable equivalent formulation, the bilinear spectral regularization~\cite{srebro:mmmf, recht:nnm}, which has been successfully applied in many large-scale applications, such as collaborative filtering~\cite{mitra:lsmf, aravkin:rpca}. Since the Schatten-${p}$ quasi-norm is equivalent to the $\ell_{p}$ quasi-norm on the singular values, it is natural to ask the following question: can we design an equivalent matrix factorization form to some cases of the Schatten-${p}$ quasi-norm, e.g., $p\!=\!1/2$ or $1/3$?

In this paper we first define two tractable Schatten norms, the bi-trace (Bi-tr) and tri-trace (Tri-tr) norms. We then prove that they are in essence the Schatten-${1/2}$ and ${1/3}$ quasi-norms, respectively, for solving whose minimization we only need to perform SVDs on much smaller factor matrices to replace the large matrices in the algorithms mentioned above. Then we design two efficient linearized alternating minimization algorithms with guaranteed convergence to solve our problems. Finally, we provide the sufficient condition for exact recovery, and the restricted strong convexity (RSC) based and MC error bounds.

\section{Notations and Background}
The Schatten-${p}$ norm ($0\!<\!p\!<\!\infty$) of a matrix $X\!\in\! \mathbb{R}^{m\times n}$ ($m\geq n$) is defined as
\vspace{-3mm}
\begin{displaymath}
\|X\|_{S_{p}}=\left(\sum\nolimits_{i=1}^{n}\sigma^{p}_{i}(X)\right)^{1/p},
\end{displaymath}
\vspace{-6mm}

where $\sigma_{i}(X)$ denotes the $i$-th singular value of $X$. For $p\!\geq\!1$ it defines a natural norm, for instance, the Schatten-${1}$ norm is the so-called trace norm, $\|X\|_{\textup{tr}}$, whereas for $p<1$ it defines a quasi-norm. As the non-convex surrogate for the rank function, the Schatten-${p}$ quasi-norm with $0\!<\!p\!<\!1$ is the better approximation of the matrix rank than the trace norm~\cite{zhang:ncmr} (analogous to the superiority of the $\ell_{p}$ quasi-norm to the $\ell_{1}$-norm~\cite{lai:irls, foucart:lp}).

We mainly consider the following Schatten quasi-norm minimization problem to recover a low-rank matrix from a small set of linear observations, $b\in\mathbb{R}^{l}$,
\vspace{-2mm}
\begin{equation}\label{Nb1}
\min_{X\in\mathbb{R}^{m\times n}} \left\{\|X\|^{p}_{S_{p}}:\,\mathcal{A}(X)=b\right\},
\end{equation}
\vspace{-6mm}

where $\mathcal{A}:\mathbb{R}^{m\times n}\!\rightarrow\! \mathbb{R}^{l}$ is a linear measurement operator. Alternatively, the Lagrangian version of \eqref{Nb1} is
\vspace{-2mm}
\begin{equation}\label{Nb2}
\min_{X\in\mathbb{R}^{m\times n}} \left\{\|X\|^{p}_{S_{p}}+\frac{1}{\mu}f\!\left(\mathcal{A}(X)-b\right)\right\},
\end{equation}
\vspace{-6mm}

where $\mu\!>\!0$ is a regularization parameter, and the loss function $f(\cdot):\mathbb{R}^{l}\!\rightarrow\! \mathbb{R}$ generally denotes certain measurement for characterizing the loss term $\mathcal{A}(X)-b$ (for instance, $\mathcal{A}$ is the linear projection operator $\mathcal{P}_{\Omega}$, and $f(\cdot)\!=\!\|\!\cdot\!\|^{2}_{2}$ in MC problems~\cite{marjanovic:mc, mohan:mrm, liu:nnr, lu:lrm}).

The Schatten-$p$ quasi-norm minimization problems \eqref{Nb1} and \eqref{Nb2} are non-convex, non-smooth and even non-Lipschitz~\cite{bian:ipa}. Therefore, it is crucial to develop efficient algorithms that are specialized to solve some alternative formulations of Schatten-${p}$ quasi-norm minimization \eqref{Nb1} or \eqref{Nb2}. So far, only few algorithms, such as IRLS~\cite{lai:irls, mohan:mrm} and IRNN~\cite{lu:lrm}, have been developed to solve such problems. In addition, since all existing Schatten-${p}$ quasi-norm minimization algorithms involve SVD or EVD in each iteration, they suffer from a high computational cost of $O(n^{2}m)$, which severely limits their applicability to large-scale problems.

\section{Tractable Schatten Quasi-Norm Minimization}
\cite{srebro:mmmf} and \cite{recht:nnm} pointed out that the trace norm has the following equivalent non-convex formulations.
\begin{lemma}\label{lem1}
Given a matrix $X\in \mathbb{R}^{m\times n}$ with $\textrm{rank}(X)=r\leq d$, the following holds:
\vspace{-2mm}
\begin{equation*}
\begin{split}
\|X\|_{\textup{tr}}&=\min_{U\in\mathbb{R}^{m\times d},V\in\mathbb{R}^{n\times d}:X=UV^{T}}\|U\|_{F}\|V\|_{F}\\
&=\min_{U,V:X=UV^{T}}\frac{\|U\|^{2}_{F}+\|V\|^{2}_{F}}{2}.
\end{split}
\end{equation*}
\end{lemma}

\subsection{Bi-Trace Quasi-Norm}
Motivated by the equivalence relation between the trace norm and its bilinear spectral regularization form stated in Lemma \ref{lem1}, our bi-trace (Bi-tr) norm is naturally defined as follows~\cite{shang:snm}.
\begin{definition}\label{def1}
For any matrix $X\in \mathbb{R}^{m\times n}$ with $\textrm{rank}(X)=r\leq d$, we can factorize it into two much smaller matrices $U\in \mathbb{R}^{m\times d}$ and $V\in \mathbb{R}^{n\times d}$ such that $X=UV^{T}$. Then the bi-trace norm of $X$ is defined as
\vspace{-1mm}
\begin{equation*}
\|X\|_{\textup{Bi-tr}}:=\min_{U,V:X=UV^{T}}\|U\|_{\textup{tr}}\|V\|_{\textup{tr}}.
\end{equation*}
\end{definition}
\vspace{-2mm}

In fact, the bi-trace norm defined above is not a real norm, because it is non-convex and does not satisfy the triangle inequality of a norm. Similar to the well-known Schatten-${p}$ quasi-norm ($0\!<\!p\!<\!1$), the bi-trace norm is also a quasi-norm, and their relationship is stated in the following theorem~\cite{shang:snm}.
\begin{theorem}\label{the1}
The bi-trace norm $\|\!\cdot\!\|_{\textup{Bi-tr}}$ is a quasi-norm. Surprisingly, it is also the Schatten-${1/2}$ quasi-norm, i.e.,
\vspace{-1mm}
\begin{equation*}
\|X\|_{\textup{Bi-tr}}=\|X\|_{S_{1/2}},
 \end{equation*}
where $\|X\|_{S_{1/2}}$ is the Schatten-${1/2}$ quasi-norm of $X$.
\end{theorem}

The proof of Theorem \ref{the1} can be found in the Supplementary Materials. Due to such a relationship, it is easy to verify that the bi-trace quasi-norm possesses the following properties.

\begin{property}\label{pro1}
For any matrix $X\in \mathbb{R}^{m\times n}$ with $\textrm{rank}(X)=r\leq d$, the following holds:
\vspace{-2mm}
\begin{equation*}
\begin{split}
\|X\|_{\textup{Bi-tr}}&\!=\!\!\!\min_{U,\!V:X\!=\!U\!V^{\!T}}\!\!\|U\|_{\textup{tr}}\!\|V\|_{\textup{tr}}\!=\!\!\!\min_{U,\!V:X\!=\!U\!V^{\!T}}\!\!\!\frac{\|U\|^{2}_{\textup{tr}}\!+\!\|V\|^{2}_{\textup{tr}}}{2}\\
&=\min_{U,\!V:X\!=\!U\!V^{T}}\!\left(\frac{\|U\|_{\textup{tr}}\!+\!\|V\|_{\textup{tr}}}{2}\right)^{2}.
\end{split}
\end{equation*}
\end{property}

\begin{property}\label{pro2}
The bi-trace quasi-norm satisfies the following properties:
\vspace{-2mm}
\begin{enumerate}
\item $\|X\|_{\textup{Bi-tr}}\geq0$, with equality iff $X=0$.
\vspace{-2mm}
\item $\|X\|_{\textup{Bi-tr}}$ is unitarily invariant, i.e., $\|X\|_{\textup{Bi-tr}}=\|PXQ^{T}\|_{\textup{Bi-tr}}$, where $P\in\mathbb{R}^{m\times m}$ and $Q\in\mathbb{R}^{n\times n}$ have orthonormal columns.
\end{enumerate}
\end{property}

\subsection{Tri-Trace Quasi-Norm}
Similar to the definition of the bi-trace quasi-norm, our tri-trace (Tri-tr) norm is naturally defined as follows.

\begin{definition}\label{def2}
For any matrix $X\in\mathbb{R}^{m\times n}$ with $\textrm{rank}(X)\!=\!r\leq d$, we can factorize it into three much smaller matrices $U\in \mathbb{R}^{m\times d}$, $V\in \mathbb{R}^{d\times d}$ and $W\in \mathbb{R}^{n\times d}$ such that $X=UVW^{T}$. Then the tri-trace norm of $X$ is defined as
\vspace{-2mm}
\begin{equation*}
\|X\|_{\textup{Tri-tr}}:=\min_{U,V,W:X=UVW^{T}}\|U\|_{\textup{tr}}\|V\|_{\textup{tr}}\|W\|_{\textup{tr}}.
\end{equation*}
\end{definition}
\vspace{-2mm}

Like the bi-trace quasi-norm, the tri-trace norm is also a quasi-norm, as stated in the following theorem.
\begin{theorem}\label{the2}
The tri-trace norm $\|\!\cdot\!\|_{\textup{Tri-tr}}$ is a quasi-norm. In addition, it is also the Schatten-${1/3}$ quasi-norm, i.e.,
\begin{equation*}
\|X\|_{\textup{Tri-tr}}=\|X\|_{S_{1/3}}.
 \end{equation*}
\end{theorem}

The proof of Theorem \ref{the2} is very similar to that of Theorem \ref{the1} and is thus omitted. According to Theorem \ref{the2}, it is easy to verify that the tri-trace quasi-norm possesses the following properties.

\begin{property}\label{pro3}
For any matrix $X\in \mathbb{R}^{m\times n}$ with $\textrm{rank}(X)=r\leq d$, the following holds:
\vspace{-2mm}
\begin{equation*}
\begin{split}
&\|X\|_{\textup{Tri-tr}}\!=\!\!\min_{X\!=\!UVW^{T}}\!\left(\frac{\|U\|_{\textup{tr}}\!+\!\|V\|_{\textup{tr}}\!+\!\|W\|_{\textup{tr}}}{3}\right)^{3}\\
=&\!\!\min_{X\!=\!U\!V\!W^{\!T}}\!\!\|U\|_{\textup{tr}}\!\|V\|_{\textup{tr}}\!\|W\|_{\textup{tr}}\!\!=\!\!\!\!\min_{X\!=\!U\!V\!W^{\!T}}\!\!\!\frac{\|U\|^{3}_{\textup{tr}}\!\!+\!\!\|V\|^{3}_{\textup{tr}}\!\!+\!\!\|W\|^{3}_{\textup{tr}}}{3}.
\end{split}
\end{equation*}
\end{property}

\begin{property}\label{pro4}
The tri-trace quasi-norm satisfies the following properties:
\vspace{-2mm}
\begin{enumerate}
\item $\|X\|_{\textup{Tri-tr}}\geq0$, with equality iff $X=0$.
\vspace{-2mm}
\item $\|X\|_{\textup{Tri-tr}}$ is unitarily invariant, i.e., $\|X\|_{\textup{Tri-tr}}=\|PXQ^{T}\|_{\textup{Tri-tr}}$, where $P\in\mathbb{R}^{m\times m}$ and $Q\in\mathbb{R}^{n\times n}$ have orthonormal columns.
\end{enumerate}
\end{property}

The following relationship between the trace-norm and Frobenius norm is well known: $\|X\|_{F}\leq\|X\|_{\textup{tr}}\leq \sqrt{r}\|X\|_{F}$. Similarly, the analogous bounds hold for the bi-trace and tri-trace quasi-norms, as stated in the following property.

\begin{property}\label{pro5}
For any matrix $X\in \mathbb{R}^{m\times n}$ with $\textrm{rank}(X)= r$, the following inequalities hold:
\vspace{-2mm}
\begin{displaymath}
\begin{split}
\|X\|_{\textup{tr}}\leq\|X&\|_{\textup{Bi-tr}}\leq r\|X\|_{\textup{tr}},\\
\|X\|_{\textup{tr}}\leq\|X\|_{\textup{Bi-tr}}&\leq \|X\|_{\textup{Tri-tr}}\leq r^{2}\|X\|_{\textup{tr}}.
\end{split}
\end{displaymath}
\end{property}
\vspace{-4mm}

\begin{proof}
The proof of this property involves the following properties of the $\ell_{p}$ quasi-norm. For any vectors $x$ and $y$ in $\mathbb{R}^{n}$ and $0<p_{2}\leq p_{1}\leq1$, we have
\vspace{-2mm}
\begin{displaymath}
\|x\|_{1}\leq\|x\|_{p_{1}},\;\; \|x\|_{p_{1}}\leq \|x\|_{p_{2}}\leq n^{1/p_{2}-1/p_{1}}\|x\|_{p_{1}}.
\end{displaymath}
Suppose $X\!\in\!\mathbb{R}^{m\times n}$ is of rank $r$, and denote its skinny SVD by $X\!=\!U\Sigma V^{T}$. By Theorems \ref{the1} and \ref{the2}, and the properties of the $\ell_{p}$ quasi-norm, we have
\vspace{-2mm}
\begin{displaymath}
\begin{split}
&\,\|X\|_{\textup{tr}}\!=\!\|\textrm{diag}(\Sigma)\|_{{1}}\!\!\leq\! \|\textrm{diag}(\Sigma)\|_{{\frac{1}{2}}}\!\!=\!\|X\|_{\textup{Bi-tr}}\!\leq\! r\|X\|_{\textup{tr}},\\
&\|X\|_{\textup{tr}}\!=\!\|\textrm{diag}(\Sigma)\|_{{1}}\!\!\leq\! \|\textrm{diag}(\Sigma)\|_{{\frac{1}{3}}}\!\!=\!\|X\|_{\textup{Tri-tr}}\!\leq\! r^{2}\|X\|_{\textup{tr}}.
\end{split}
\end{displaymath}
In addition,

$\|X\|_{\textup{Bi-tr}}\!=\!\|\textrm{diag}(\Sigma)\|_{{\frac{1}{2}}}\!\!\leq\!\|\textrm{diag}(\Sigma)\|_{{\frac{1}{3}}}\!\!=\!\|X\|_{\textup{Tri-tr}}$.
\end{proof}
\vspace{-2mm}

It is easy to see that Property \ref{pro5} in turn implies that any low bi-trace or tri-trace quasi-norm approximation is also a low trace norm approximation.

\subsection{Problem Formulations}
Bounding the Schatten quasi-norm of $X$ in \eqref{Nb1} by the bi-trace or tri-trace quasi-norm defined above, the noiseless low-rank structured matrix factorization problem is given by
\vspace{-1mm}
\begin{equation}\label{Fm1}
\min_{U,V}\!\left\{\mathcal{R}(U,\!V)\!=\!(\|U\|_{\textup{tr}}\!+\!\|V\|_{\textup{tr}})/2\!:\mathcal{A}(UV^{T}\!)\!=\!b\right\},
\end{equation}
where $\mathcal{R}(\cdot)$ can also denote $(\|U\|_{\textup{tr}}\!+\!\|V\|_{\textup{tr}}\!+\!\|W\|_{\textup{tr}})/3$, and $\mathcal{A}(UV^{T})$ is replaced by $\mathcal{A}(UVW^{T})$.
In addition, \eqref{Fm1} has the following Lagrangian forms,
\begin{equation}\label{Fm2}
F(U,\!V)\!:=\!\min_{U,V}\!\left\{\!\frac{\|U\|_{\textup{tr}}\!\!+\!\!\|V\|_{\textup{tr}}}{2}\!+\!\frac{f(\mathcal{A}(UV^{T}\!)\!-\!b)}{\mu}\!\right\},
\end{equation}
\begin{equation}\label{Fm3}
\min_{U,V,W}\!\!\left\{\!\frac{\|U\!\|_{\textup{tr}}\!+\!\|\!V\!\|_{\textup{tr}}\!+\!\|\!W\!\|_{\textup{tr}}}{3}\!+\!\frac{f\!(\!\mathcal{A}(U\!V\!W^{T}\!)\!-\!b)}{\mu}\!\right\}.
\end{equation}

The formulations \eqref{Fm1}, \eqref{Fm2} and \eqref{Fm3} can address a wide range of problems, such as MC~\cite{mohan:mrm, lu:lrm}, RPCA~\cite{candes:rpca,shang:rpca,shang:rbf} ($\mathcal{A}$ is the identity operator, and $f(\cdot)\!=\!\|\!\cdot\!\|_{1}$ or $\!\|\!\cdot\!\|_{{p}}\,(0\!<\!p\!<\!1)$), and low-rank representation~\cite{liu:lrr} or multivariate regression~\cite{hsieh:nnm} ($\mathcal{A}(X)\!=\!AX$ with $A$ being a given matrix, and $f(\cdot)\!=\!\|\!\cdot\!\|_{2,1}$ or $\|\!\cdot\!\|^{2}_{F}$). In addition, $f(\cdot)$ may be also chosen as the Hinge loss in~\cite{srebro:mmmf} or the structured atomic norms in~\cite{jaggi:fw}.

\section{Optimization Algorithms}
In this section, we mainly propose two efficient algorithms to solve the challenging bi-trace quasi-norm regularized problem \eqref{Fm2} with a smooth or non-smooth loss function, respectively. In other words, if $f(\cdot)$ is a smooth loss function, e.g., $f(\cdot)\!=\!\frac{1}{2}\!\|\!\cdot\!\|^{2}_{2}$, we employ the proximal alternating linearized minimization (PALM) method as in~\cite{bolte:palm} to solve \eqref{Fm2}. In contrast, to solve efficiently \eqref{Fm2} with a non-smooth loss function, e.g., $f(\cdot)\!=\!\|\!\cdot\!\|_{1}$, we need to introduce an auxiliary variable $e$ and obtain the following equivalent form:
\vspace{-1mm}
\begin{equation}\label{Al1}
\min_{U,V,e}\!\left\{\frac{\|U\|_{\textup{tr}}\!+\!\|V\|_{\textup{tr}}}{2}\!+\!\frac{f(e)}{\mu}\!: e\!=\!\mathcal{A}(UV^{T})\!-\!b\right\}.
\end{equation}

\subsection{LADM Algorithm}
To avoid introducing more auxiliary variables, inspired by~\cite{yang:adm}, we propose a linearized alternating direction method (LADM) to solve~\eqref{Al1}, whose augmented Lagrangian function is given by
\vspace{-2mm}
\begin{equation*}
\begin{split}
\mathcal{L}(U,V,e,\lambda,\beta)\!=&\frac{1}{2}\!(\|U\|_{\textup{tr}}\!+\!\|V\|_{\textup{tr}})\!+\!\frac{f(e)}{\mu}\!+\!\langle \lambda,\mathcal{A}(UV^{T})\\
&-b-e\rangle+({\beta}/{2})\|\mathcal{A}(UV^{T})\!-b-e\|^{2}_{2},
\end{split}
\end{equation*}
\vspace{-5mm}

where $\lambda\!\in\!\mathbb{R}^{l}$ is the Lagrange multiplier, $\langle\cdot,\cdot\rangle$ denotes the inner product, and $\beta\!>\!0$ is a penalty parameter. By applying the classical augmented Lagrangian method to \eqref{Al1}, we obtain the following iterative scheme:
\vspace{-1mm}
\addtocounter{equation}{1}
\begin{align}
&U_{k+\!1}\!=\!\mathop{\arg\min}_{U}\!\frac{\!\|U\|_{\textup{tr}}}{2}\!+\!\frac{\beta_{k}}{2}\!\|\mathcal{A}(\!U\!V^{T}_{k}\!)\!-\!e_{k}\!-\!\widetilde{b}_{k}\|^{2}_{2}, \tag{\theequation a} \label{aleq1}\\
&V_{k+\!1}\!\!=\!\mathop{\arg\min}_{V}\!\frac{\!\|V\|_{\textup{tr}}}{2}\!+\!\frac{\beta_{k}}{2}\!\|\mathcal{A}(\!U_{k\!+\!1}\!V^{T}\!)\!-\!e_{k}\!-\!\widetilde{b}_{k}\|^{2}_{2}, \tag{\theequation b} \label{aleq2}\\
&e_{k+\!1}\!=\!\mathop{\arg\min}_{e}\!\frac{f(e)}{\mu}\!+\!\frac{\beta_{k}}{2}\|\mathcal{A}(\!U_{k+\!1}\!V^{T}_{k+\!1}\!)\!-\!e\!-\!\widetilde{b}_{k}\|^{2}_{2},\tag{\theequation c} \label{aleq3}\\
&\lambda_{k+1}\!=\!\lambda_{k}+\beta_{k}(\mathcal{A}(U_{k+1}V^{T}_{k+1})-b-e_{k+1}),\tag{\theequation d} \label{aleq4}
\end{align}
\vspace{-6mm}

where $\widetilde{b}_{k}\!=\!b\!-\!{\lambda_{k}}/{\beta_{k}}$. In many machine learning problems~\cite{marjanovic:mc, liu:lrr, hsieh:nnm}, $\mathcal{A}$ is not identity, e.g., the operator $\mathcal{P}_{\Omega}$. Due to the presence of $V_{k}$ and $U_{k+\!1}$, thus we usually need to introduce some auxiliary variables to achieve closed-form solutions to \eqref{aleq1} and \eqref{aleq2}. To avoid introducing additional auxiliary variables, we propose the following linearization technique for \eqref{aleq1} and \eqref{aleq2}.

\subsubsection{Updating $U_{k+1}$ and $V_{k+1}$}
Let $\varphi_{k}(U)\!:=\!\|\mathcal{A}(UV^{T}_{k})\!-\!b\!-\!e_{k}\!+\!\lambda_{k}/\!\beta_{k}\|^{2}_{2}/2$, then we can know that the gradient of $\varphi_{k}(U)$ is Lipschitz continuous with the constant $t^{\varphi}_{k}$, i.e., $\|\nabla\! \varphi_{k}(U_{1})\!-\!\nabla\! \varphi_{k}(U_{2})\|_{F}\!\leq\! t^{\varphi}_{k}\|U_{1}\!-\!U_{2}\|_{F}$ for any $U_{1}, U_{2}\!\in\! \mathbb{R}^{m\times d}$. By linearizing $\varphi_{k}(U)$ at $U_{k}$ and adding a proximal term, we have
\vspace{-2mm}
\begin{equation}\label{Al7}
\!\!\widehat{\varphi}_{k}\!(U,U_{k}\!)\!=\!\varphi_{k}\!(U_{k}\!)\!+\!\langle\nabla\! \varphi_{k}\!(U_{k}\!),U\!-\!U_{k}\rangle\!+\!\frac{t^{\varphi}_{k}}{2}\|U\!-\!U_{k}\!\|^{2}_{F}.
\end{equation}
Therefore, we have
\vspace{-2mm}
\begin{equation}\label{Al2}
\begin{split}
&U_{k+1}\!=\!\mathop{\arg\min}_{U}\frac{1}{2}\|U\|_{\textup{tr}}\!+\!\beta_{k}\widehat{\varphi}_{k}(U,U_{k})\\
=&\!\mathop{\arg\min}_{U}\!\frac{1}{2}\|U\|_{\textup{tr}}\!+\!\frac{\beta_{k}t^{\varphi}_{k}}{2}\|U\!-\!U_{k}\!+\!\frac{\nabla\! \varphi_{k}(U_{k})}{t^{\varphi}_{k}}\|^{2}_{F}.
\end{split}
\end{equation}
\vspace{-5mm}

Similarly, we have
\vspace{-2mm}
\begin{equation}\label{Al3}
V_{k+1}\!\!=\!\mathop{\arg\min}_{V}\!\frac{1}{2}\!\|V\|_{\textup{tr}}\!+\!\frac{\beta_{k}t^{\psi}_{k}}{2}\!\|V\!-\!V_{k}\!+\!\frac{\nabla\! \psi_{k}(V_{k})}{t^{\psi}_{k}}\|^{2}_{F},
\end{equation}
where $\psi_{k}(V)\!:=\!\|\mathcal{A}(U_{k+1}\!V^{T})\!-\!b\!-\!e_{k}\!+\!\lambda_{k}/\beta_{k}\|^{2}_{2}/2$ with the Lipschitz constant $t^{\psi}_{k}$. Using the so-called matrix shrinkage operator~\cite{cai:svt}, we can obtain a closed-form solution to \eqref{Al2} and \eqref{Al3}, respectively. Additionally, if $f(\cdot)\!=\!\|\!\cdot\!\|_{1}$, the optimal solution to \eqref{aleq3} can be obtained by the well-known soft-thresholding operator~\cite{daubechies:it}.

\subsubsection{Computing Step Sizes}
There are two step sizes, i.e., the Lipschitz constants $t^{\varphi}_{k}$ in \eqref{Al2} and $t^{\psi}_{k}$ in \eqref{Al3}, need to be set during the iteration.
\vspace{-1mm}
\begin{equation*}
\begin{split}
&\|\nabla\! \varphi_{k}(U_{1})\!-\!\nabla\! \varphi_{k}(U_{2})\|_{F}\!=\!\|\mathcal{A}^{\ast}\{\mathcal{A}[(U_{1}-U_{2})V^{T}_{k}]\}V_{k}\|_{F}\\
&\leq  \|\mathcal{A}^{\ast}\mathcal{A}\|_{2}\|V^{T}_{k}V_{k}\|_{2}\|U_{1}\!-\!U_{2}\|_{F},\\
&\|\nabla\! \psi_{k}\!(V_{1})\!-\!\!\nabla\! \psi_{k}\!(V_{2})\|_{F}\!=\!\!\|U^{T}_{k\!+\!1}\mathcal{A}^{\ast}\{\mathcal{A}[U_{k\!+\!1}(V_{1}\!\!-\!V_{2})^{T}]\}\|_{F}\\
&\leq \|\mathcal{A}^{\ast}\mathcal{A}\|_{2}\|U^{T}_{k+\!1}U_{k+\!1}\|_{2}\|V_{1}\!-\!V_{2}\|_{F},
\end{split}
\end{equation*}
where $\mathcal{A}^{\ast}$ denotes the adjoint operator of $\mathcal{A}$. Thus, both step sizes are defined in the following way:
\vspace{-1mm}
\begin{equation}\label{Al4}
\left\{
\begin{aligned}
t^{\varphi}_{k}&\geq\|\mathcal{A}^{\ast}\!\mathcal{A}\|_{2}\|V^{T}_{k}\!V_{k}\|_{2},\\
t^{\psi}_{k}&\geq\|\mathcal{A}^{\ast}\!\mathcal{A}\|_{2}\|U^{T}_{k+\!1}\!U_{k+\!1}\|_{2}.
\end{aligned}
\right.
\end{equation}
\vspace{-6mm}

Based on the description above, we develop an efficient LADM algorithm to solve the Bi-tr quasi-norm regularized problem \eqref{Fm2} with a non-smooth loss function (e.g., RPCA problems), as outlined in \textbf{Algorithm} \ref{alg1}. To further accelerate the convergence of the algorithm, the penalty parameter $\beta$ is adaptively updated by the strategy as in~\cite{lin:ladmm}, as well as $\rho$. Moreover, Algorithm \ref{alg1} can be used to solve the noiseless problem \eqref{Fm1} and also extended to solve the Tri-tr quasi-norm regularized problem \eqref{Fm3} with a non-smooth loss function.

\begin{algorithm}[t]
\caption{LADM for \eqref{Fm2} with non-smooth loss}
\label{alg1}
\renewcommand{\algorithmicrequire}{\textbf{Input:}}
\renewcommand{\algorithmicensure}{\textbf{Initialize:}}
\renewcommand{\algorithmicoutput}{\textbf{Output:}}
\begin{algorithmic}[1]
\REQUIRE $b$, the given rank $d$  and $\mu$.
\ENSURE $\beta_{0}\!=\!10^{-4}$, $\beta_{\max}\!=\!10^{20}$ and $\varepsilon\!=\!10^{-4}$.\\
\WHILE {not converged}
\STATE {Update $t^{\varphi}_{k}$, $U_{k+1}$, $t^{\psi}_{k}$, and $V_{k+1}$ by \eqref{Al4}, \eqref{Al2}, \eqref{Al4}, and \eqref{Al3}, respectively.}
\STATE {Update $e_{k+1}$ and $\lambda_{k+1}$ by \eqref{aleq3} and \eqref{aleq4}.}
\STATE {Update $\beta_{k+1}$ by $\beta_{k+1}=\min(\rho\beta_{k},\,\beta_{\max})$.}
\STATE {Check the convergence condition,\\ $\quad\quad \|\mathcal{A}(U_{k+1}V^{T}_{k+1})-b-e_{k+1}\|_{2}<\varepsilon$.}
\ENDWHILE
\OUTPUT $U_{k+1}$, $V_{k+1}$, $e_{k+1}$.
\end{algorithmic}
\end{algorithm}

\subsection{PALM Algorithm}
By using the similar linearization technique in \eqref{Al2} and \eqref{Al3}, we design an efficient PALM algorithm to solve \eqref{Fm2} with a smooth loss function, e.g., MC problems. Specifically, by linearizing the smooth loss function $\varphi_{k}(U)\!:=\!\|\mathcal{A}(UV^{T}_{k})\!-\!b\|^{2}_{2}/2$ at $U_{k}$ and adding a proximal term, we have the following approximation:
\vspace{-2mm}
\begin{equation}\label{Al5}
\begin{split}
&U_{k+1}\\
\!\!\!=&\!\mathop{\arg\min}_{U}\!\!\frac{\|U\|_{\textup{tr}}}{2}\!\!+\!\!\langle\frac{\nabla\! \varphi_{k}(U_{k})}{\mu}, U\!\!-\!\!U_{k}\rangle\!\!+\!\!\frac{t^{\varphi}_{k}}{2\mu}\!\|U\!\!-\!\!U_{k}\|^{2}_{F}\\
\!\!\!=&\!\mathop{\arg\min}_{U}\!\frac{\|U\|_{\textup{tr}}}{2}\!+\!\frac{t^{\varphi}_{k}}{2\mu}\|U\!-\!U_{k}\!+\!\frac{\nabla\! \varphi_{k}(U_{k})}{t^{\varphi}_{k}}\|^{2}_{F},
\end{split}
\end{equation}
where $\nabla\! \varphi_{k}(U_{k})=\mathcal{A}^{\ast}[\mathcal{A}(U_{k}V^{T}_{k})-b]V_{k}$. Similarly,
\vspace{-2mm}
\begin{equation}\label{Al6}
V_{k+1}\!=\!\mathop{\arg\min}_{V}\!\frac{\|V\|_{\textup{tr}}}{2}\!+\!\frac{t^{\psi}_{k}}{2\mu}\|V\!-\!V_{k}\!+\!\frac{\nabla\! \psi_{k}(V_{k})}{t^{\psi}_{k}}\|^{2}_{F},
\end{equation}
\vspace{-5mm}

where $\nabla\! \psi_{k}(V_{k})=\{\mathcal{A}^{\ast}[\mathcal{A}(U_{k+1}V^{T}_{k})-b]\}^{T}U_{k+1}$.

\subsection{Convergence Analysis}
In the following, we provide the convergence analysis of our algorithms. First, we analyze the convergence of our LADM algorithm for solving \eqref{Fm2} with a non-smooth loss function, e.g., $f(\cdot)=\|\!\cdot\!\|_{1}$.

\begin{theorem}\label{the3}
Let $\{(U_{k},V_{k},e_{k})\}$ be a sequence generated by Algorithm \ref{alg1}, then we have
\vspace{-3mm}
\begin{enumerate}
\item $\{(U_{k},V_{k},e_{k})\}$ are all Cauchy sequences;
\vspace{-2mm}
\item If $\lim_{k\rightarrow\infty}\|\lambda_{k+1}\!-\!\lambda_{k}\|_{2}=0$, then the accumulation point of the sequence $\{(U_{k}, V_{k}, e_{k})\}$ satisfies the KKT conditions for \eqref{Al1}.
\end{enumerate}
\end{theorem}
\vspace{-4mm}

The proof of Theorem \ref{the3} is provided in the Supplementary Materials. From Theorem \ref{the3}, we can know that under mild conditions each sequence generated by our LADM algorithm converges to a critical point, similar to the LADM algorithms for solving convex problems as in~\cite{lin:ladmm}.

Moreover, we provide the global convergence of our PALM algorithm for solving \eqref{Fm2} with a smooth loss function, e.g., $f(\cdot)\!=\!\frac{1}{2}\!\|\!\cdot\!\|^{2}_{2}$.

\begin{theorem}\label{the4}
Let $\{(U_{k},V_{k})\}$ be a sequence generated by our PALM algorithm, then it is a Cauchy sequence and converges to a critical point of \eqref{Fm2} with the squared loss, $\|\!\cdot\!\|^{2}_{2}$.
\end{theorem}
\vspace{-2mm}

The proof of Theorem \ref{the4} can be found in the Supplementary Materials. Theorem \ref{the4} shows the global convergence of our PALM algorithm. We emphasize that, different from the general subsequence convergence property, the global convergence property is given by $(U_{k},V_{k})\!\rightarrow\!(\widehat{U}, \widehat{V})$ as the number of iteration $k\!\rightarrow\!+\infty$, where $(\widehat{U}, \widehat{V})$ is a critical point of \eqref{Fm2}. On the contrary, existing algorithms for solving non-convex and non-smooth problems, such as \cite{lai:irls} and \cite{lu:lrm}, have only subsequence convergence property.

By the Kurdyka-{\L}ojasiewicz (KL) property (for more details, see~\cite{bolte:palm}) and Theorem 2 in~\cite{attouch:cr}, our PALM algorithm has the following convergence rate:

\begin{theorem}\label{the5}
The sequence $\{(U_{k},V_{k})\}$ generated by our PALM algorithm converges to a critical point $(\widehat{U}, \widehat{V})$ of $F$ with $f(\cdot)\!=\!\frac{1}{2}\!\|\!\cdot\!\|^{2}_{2}$, which satisfies the KL property at each point of $\textup{dom}\,\partial F$ with $\phi(s)\!=\!cs^{1-\theta}$ for $c\!>\!0$ and $\theta\!\in\![0,1)$. We have
\vspace{-3mm}
\begin{itemize}
\item If $\theta=0$, $\{(U_{k},V_{k})\}$ converges to $(\widehat{U}, \widehat{V})$ in finite steps;
\vspace{-2mm}
\item If $\theta\!\in\!(0,1/2]$, then $\exists C\!>\!0$ and $\gamma\in[0,1)$ such that $\|[U^{T}_{k}\!,V^{T}_{k}]-[\widehat{U}^{T}\!,\widehat{V}^{T}]\|_{F}\leq C\gamma^{k}$;
\vspace{-2mm}
\item If $\theta\in(1/2,1)$, then $\exists C>0$ such that $\|[U^{T}_{k}\!,V^{T}_{k}]-[\widehat{U}^{T}\!,\widehat{V}^{T}]\|_{F}\leq Ck^{-\frac{1-\theta}{2\theta-1}}$.
\end{itemize}
\end{theorem}
\vspace{-3mm}

Theorem \ref{the5} shows us the convergence rate of our PALM algorithm for solving the non-convex and non-smooth bi-trace quasi-norm problem \eqref{Fm2} with the squared loss $\|\!\cdot\!\|^{2}_{2}$. Moreover, we can see that the convergence rate of our PALM algorithm is at least sub-linear.

\section{Recovery Guarantees}
We provide theoretical guarantees for our Bi-tr quasi-norm minimization in recovering low-rank matrices from small sets of linear observations. By using the null-space property (NSP), we first provide a sufficient condition for exact recovery of low-rank matrices. We then establish the restricted strong convexity (RSC) condition based and MC error bounds.

\subsection{Null Space Property}
The wide use of NSP for recovering sparse vectors and low-rank matrices can be found in~\cite{foucart:lp, oymak:lrm}. We give the sufficient and necessary condition for exact recovery via our bi-trace quasi-norm model \eqref{Fm1} that improves the NSP condition for the Schatten-$p$ quasi-norm in~\cite{oymak:lrm}. Let $U_{\star}\!\!=\!\!L_{(d)}\Sigma^{1/2}_{(d)}\!\!\in\!\mathbb{R}^{m\times d}$, $V_{\star}\!\!=\!\!R_{(d)}\Sigma^{1/2}_{(d)}\!\!\in\!\mathbb{R}^{n\times d}$ and $\Sigma_{(d)}\!\!=\!\!\textup{diag}([\sigma_{1}(X_{0}),\!\ldots\!,\sigma_{r}(X_{0}),0,\!\ldots\!,0])\!\!\in\!\!\mathbb{R}^{d\times d}$, where $L_{(d)}$ and $R_{(d)}$ denote the matrices consisting the top $d$ left and right singular vectors of the true matrix $X_{0}$ (which satisfies $\mathcal{A}(X_{0})\!=\!b$) with rank at most $r$ $(r\!\leq\! d)$. $\mathcal{N}(\mathcal{A})\!:=\!\{X \!\in\!\mathbb{R}^{m\times n}\!:\!\mathcal{A}(X)\!=\!\mathbf{0}\}$ denotes the null space of the linear operator $\mathcal{A}$. Then we have the following theorem, the proof of which is provided in the Supplementary Materials.

\begin{theorem}\label{Recovery Th1}
$X_{0}$ can be uniquely recovered by \eqref{Fm1}, if and only if for any $Z=U_{\star}W^{T}_{2}+W_{1}V^{T}_{\star}+W_{1}W^{T}_{2}\in \mathcal{N}(\mathcal{A})\setminus\{\mathbf{0}\}$, where $W_{1}\in \mathbb{R}^{m\times d}$, $W_{2}\in \mathbb{R}^{n\times d}$, we have
\vspace{-2mm}
\begin{equation}
\sum^{r}_{i=1}\!\sigma_{i}(W_{1})\!+\!\sigma_{i}(W_{2})<\sum^{d}_{i=r+1}\!\sigma_{i}(W_{1})\!+\!\sigma_{i}(W_{2}).
\end{equation}
\end{theorem}
\vspace{-2mm}

\textbf{Remark}: Since $\Gamma\!\subset\! \mathcal{N}(\mathcal{A})$, where $\Gamma\!=\!\{Z|Z\!=\!U_{\star}W^{T}_{2}\!+\!W_{1}V^{T}_{\star}\!+\!W_{1}W^{T}_{2},\,Z\!\in\! \mathcal{N}(\mathcal{A})\!\setminus\!\{\mathbf{0}\}\}$, the sufficient condition in Theorem \ref{Recovery Th1} is weaker than the corresponding sufficient condition for the Schatten-$p$ quasi-norm in~\cite{oymak:lrm}.

\subsection{RSC based Error Bound}
Unlike most of existing recovery guarantees as in~\cite{zhang:ncmr, oymak:lrm}, we do not impose the restricted isometry property (RIP) on the general operator $\mathcal{A}$, rather, we require the operator $\mathcal{A}$ to satisfy a weaker and more general condition known as restricted strong convexity (RSC)~\cite{negahban:rsc}, as shown in the following.

\begin{assumption}[RSC]
We suppose that there is a positive constant $\kappa(\mathcal{A})$ such that the general operator $\mathcal{A}:\mathbb{R}^{m\times n}\!\rightarrow\!\mathbb{R}^{l}$ satisfies the following inequality
\vspace{-2mm}
\begin{displaymath}
\frac{1}{\sqrt{l}}\|\mathcal{A}(\Delta)\|_{2}\geq\kappa(\mathcal{A})\|\Delta\|_{F}
\end{displaymath}
for all $\Delta\in\mathbb{R}^{m\times n}$.
\end{assumption}
\vspace{-2mm}

We mainly provide the RSC based error bound for robust recovery via our bi-trace quasi-norm algorithm with noisy measurements. To our knowledge, our recovery guarantee analysis is the first one for solutions generated by Schatten quasi-norm algorithms, not for the global optima\footnote{It is well known that the Schatten-$p$ quasi-norm ($0\!\!<\!\!p\!\!<\!\!1$) problems in~\cite{marjanovic:mc, nie:rmc, lai:irls, lu:lrm, lu:irsvm} are non-convex, non-smooth and non-Lipschitz~\cite{bian:ipa}. The recovery guarantees in~\cite{rohde:lrm, zhang:ncmr, oymak:lrm} are naturally based on the global optimal solution of associated models.} of \eqref{Fm2} as in~\cite{rohde:lrm, zhang:ncmr, oymak:lrm}.

\begin{theorem}\label{Recovery Th3}
Assume $X_{0}\in \mathbb{R}^{m\times n}$ is a true matrix and the corrupted measurements $\mathcal{A}(X_{0})+e=b$, where $e$ is noise with $\|e\|_{2}\leq \epsilon$. Let $(\hat{U},\hat{V})$ be a critical point of \eqref{Fm2} with the squared loss $\|\!\cdot\!\|^{2}_{2}$, and suppose the operator $\mathcal{A}$ satisfies the RSC condition with a constant $\kappa(\mathcal{A})$. Then
\vspace{-2mm}
\begin{displaymath}
\frac{\|X_{0}\!-\!\hat{U}\hat{V}^{T}\|_{F}}{\sqrt{m n}}\!\leq\! \frac{\epsilon}{\kappa(\mathcal{A})\sqrt{lmn}}\!+\!\frac{\mu\sqrt{d}}{2C_{1}\kappa(\mathcal{A})\sqrt{lmn}},
\end{displaymath}
where $C_{1}=\frac{\|\mathcal{A}^{*}(b-\mathcal{A}(\hat{U}\hat{V}^{T}))\hat{V}\|_{F}}{\|b-\mathcal{A}(\hat{U}\hat{V}^{T})\|_{2}}$.
\end{theorem}
\vspace{-2mm}

The proof of Theorem \ref{Recovery Th3} and the analysis of lower-boundedness of $C_{1}$ is provided in the Supplementary Materials.

\subsection{Error Bound on Matrix Completion}
Although the MC problem is a practically important application of \eqref{Fm2}, the projection operator $\mathcal{P}_{\Omega}$ in \eqref{ebmc1} does not satisfy the standard RIP and RSC conditions in general~\cite{candes:emc, candes:mcn, jain:svp}. Therefore, we also need to provide the recovery guarantee for performance of our Bi-tr quasi-norm minimization for solving the following MC problem.
\vspace{-2mm}
\begin{equation}\label{ebmc1}
\min_{U,V} \!\left\{\!\frac{\|U\|_{\textup{tr}}\!\!+\!\!\|V\|_{\textup{tr}}}{2}\!+\!\frac{1}{2\mu}\!\|\mathcal{P}_{\Omega}(U\!V^{T}\!)\!\!-\!\!\mathcal{P}_{\Omega}(D)\|^{2}_{F}\!\right\}.
\end{equation}
\vspace{-4mm}

Without loss of generality, assume that the observed matrix $D\!\in\! \mathbb{R}^{m\times n}$ can be decomposed as a true matrix $X_{0}$ of rank $r\!\leq\! d$ and a random Gaussian noise $E$, i.e., $D\!=\!X_{0}\!+\!E$. We give the following recovery guarantee for our Bi-tr quasi-norm minimization \eqref{ebmc1}.

\begin{theorem}\label{Recovery Th4}
Let $(\widehat{U},\widehat{V})$ be a critical point of the problem \eqref{ebmc1} with given rank $d$, and $m\geq n$. Then there exists an absolute constant $C_{2}$, such that with probability at least $1-2\exp(-m)$,
\vspace{-2mm}
\begin{equation*}
\frac{\|\!X_{0}\!-\!\widehat{U}\widehat{V}^{T}\!\|_{F}}{\sqrt{m n}}\!\!\leq\!\! \frac{\|E\|_{F}}{\sqrt{mn}}\!+\!C_{2}\delta\!\!\left(\!\frac{md\log(m)}{|\Omega|}\!\right)^{1/4}\!\!+\!\frac{\mu\sqrt{d}}{2C_{3}\!\sqrt{|\Omega|}},
\end{equation*}
\vspace{-2mm}

where $\delta\!=\!\max_{i,j}|D_{i,j}|$ and $C_{3}\!=\!\frac{\|\mathcal{P}_{\Omega}(D-\hat{U}\hat{V}^{T}\!)\hat{V}\|_{F}}{\|\mathcal{P}_{\Omega}(D-\hat{U}\hat{V}^{T}\!)\|_{F}}$.
\end{theorem}

The proof of Theorem \ref{Recovery Th4} and the analysis of lower-boundedness of $C_{3}$ can be found in the Supplementary Materials. When the samples size $|\Omega|\!\gg\! md\log(m)$, the second and third terms diminish, and the recovery error is essentially bounded by the ``average" magnitude of entries of noise $E$. In other words, only $O(md\log(m))$ observed entries are needed, significantly lower than $O(mr\log^{2}(m))$ in standard matrix completion theories~\cite{candes:mcn, keshavan:mc, recht:nnm}, which will be confirmed by the following experimental results.

\section{Experimental Results}
We evaluate both the effectiveness and efficiency of our methods (i.e., the Bi-tr and Tri-tr methods) for solving MC and RPCA problems, such as collaborative filtering and text separation. All experiments were conducted on an Intel Xeon E7-4830V2 2.20GHz CPU with 64G RAM.

\begin{figure*}[t]
\begin{center}
\setlength{\belowcaptionskip}{-6pt}
{\includegraphics[width=0.4925\columnwidth]{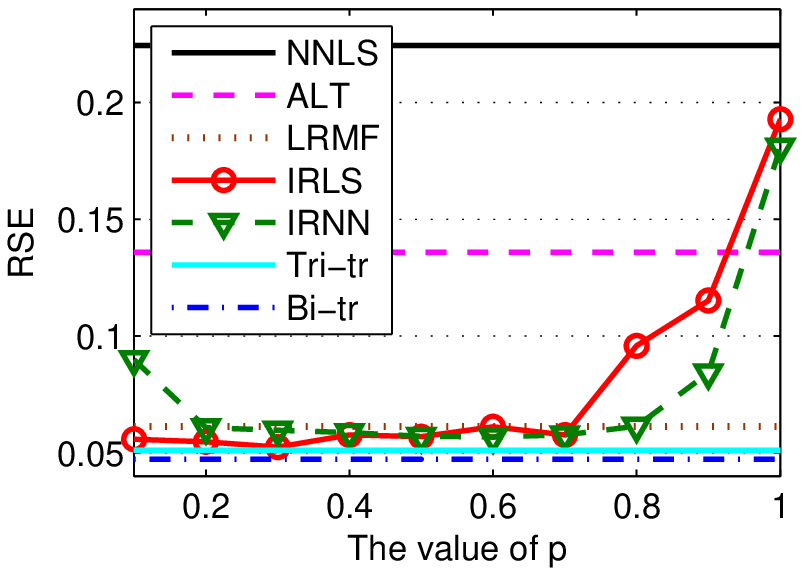}}\,
{\includegraphics[width=0.4925\columnwidth]{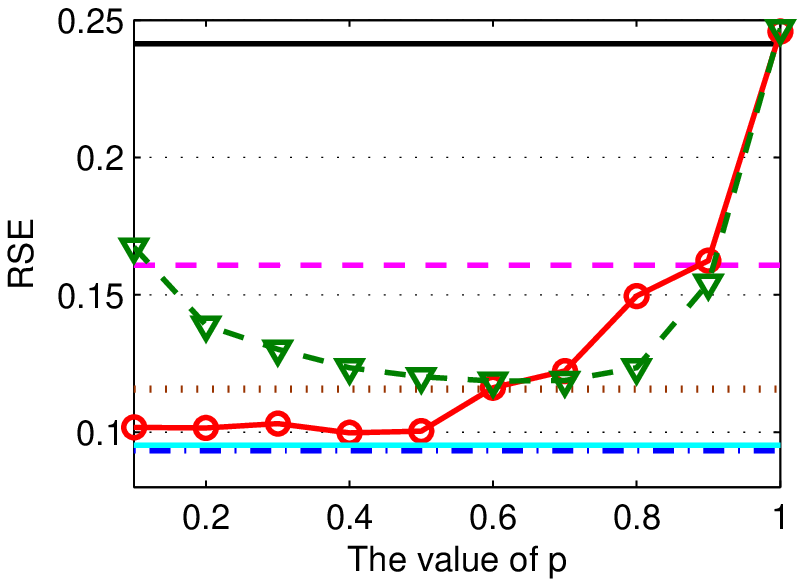}}\,
{\includegraphics[width=0.4925\columnwidth]{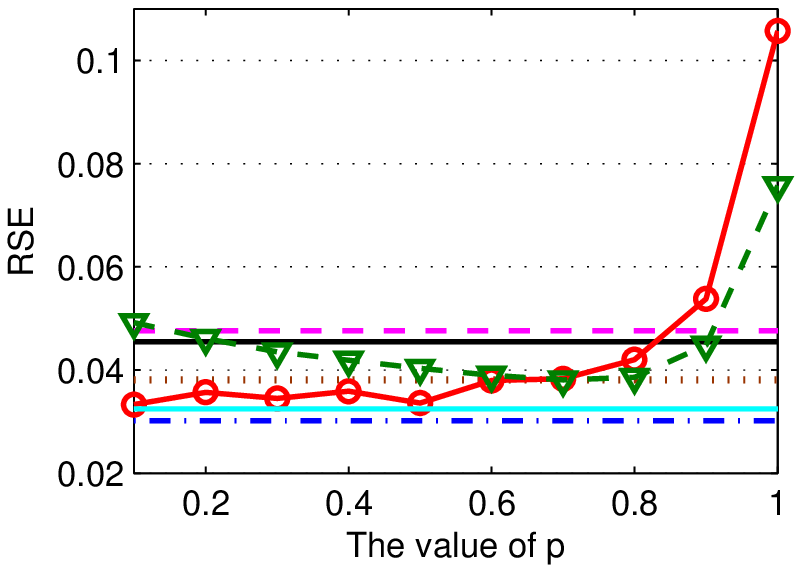}}\,
{\includegraphics[width=0.4925\columnwidth]{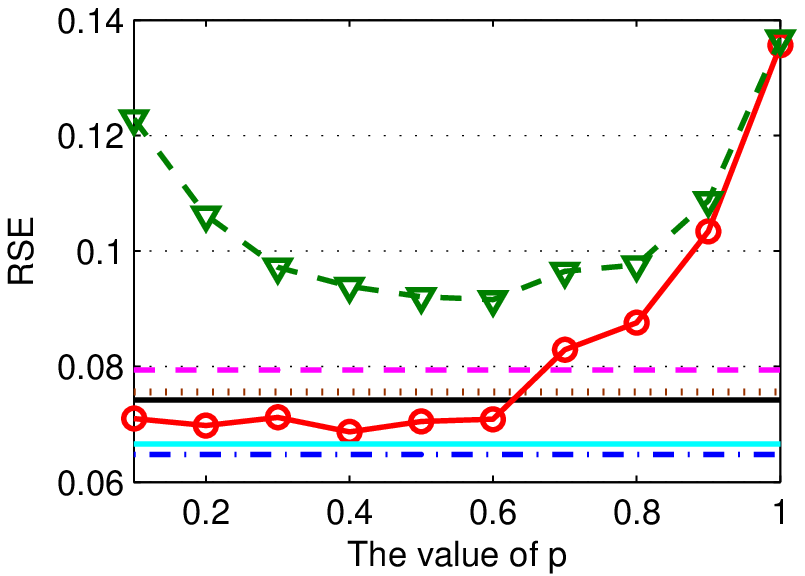}}
\subfigure[20\% SR and $nf\!=\!0.1$]{\includegraphics[width=0.4925\columnwidth]{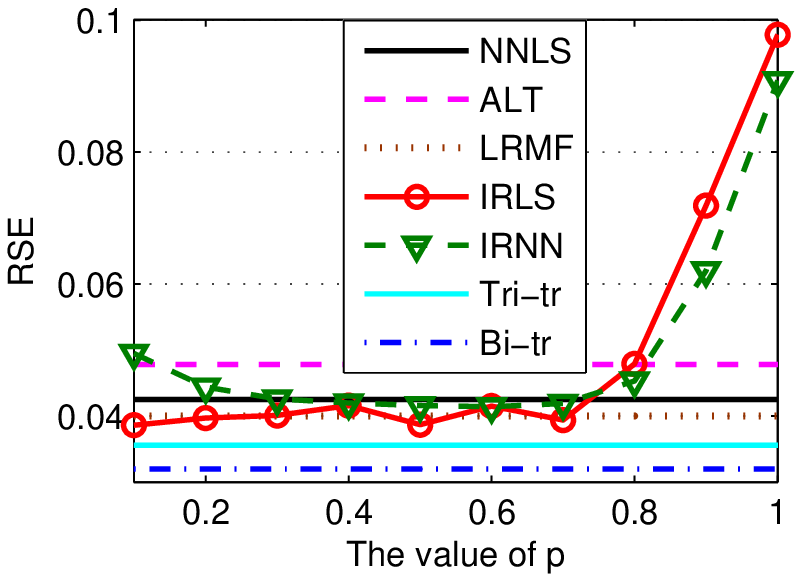}}\,
\subfigure[20\% SR and $nf\!=\!0.2$]{\includegraphics[width=0.4925\columnwidth]{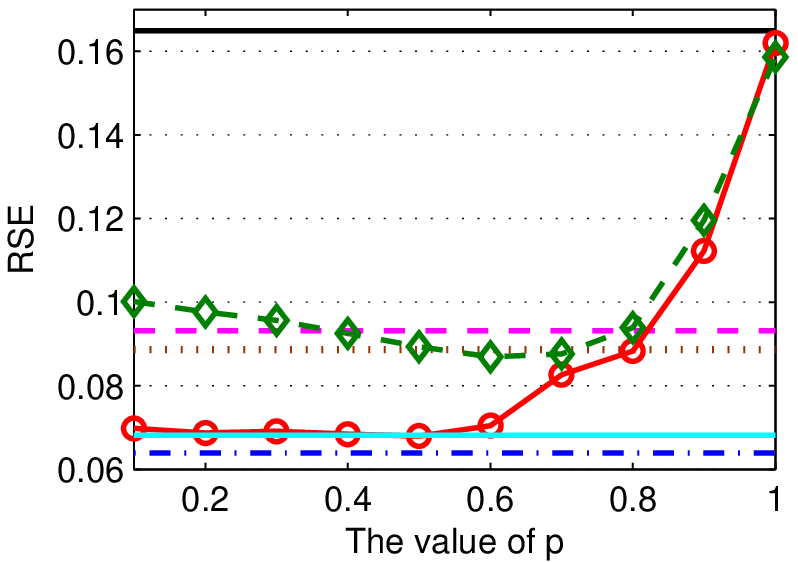}}\,
\subfigure[30\% SR and $nf\!=\!0.1$]{\includegraphics[width=0.4925\columnwidth]{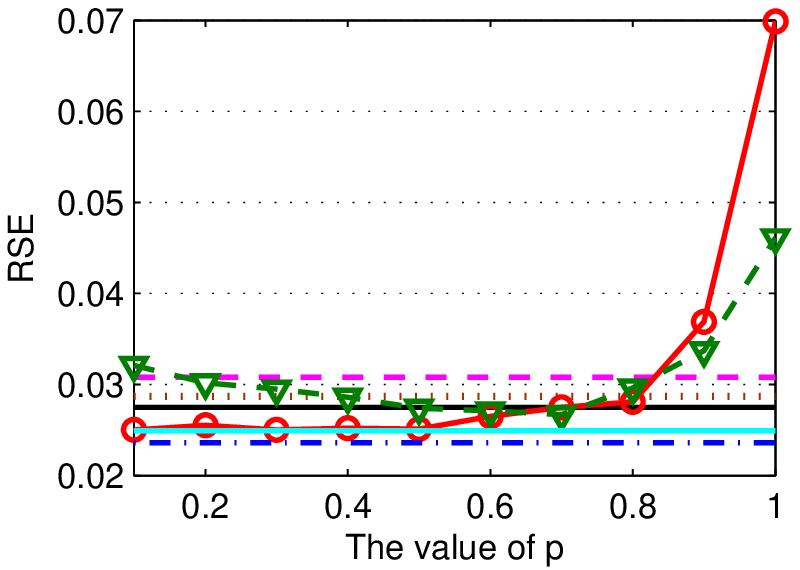}}\,
\subfigure[30\% SR and $nf\!=\!0.2$]{\includegraphics[width=0.4925\columnwidth]{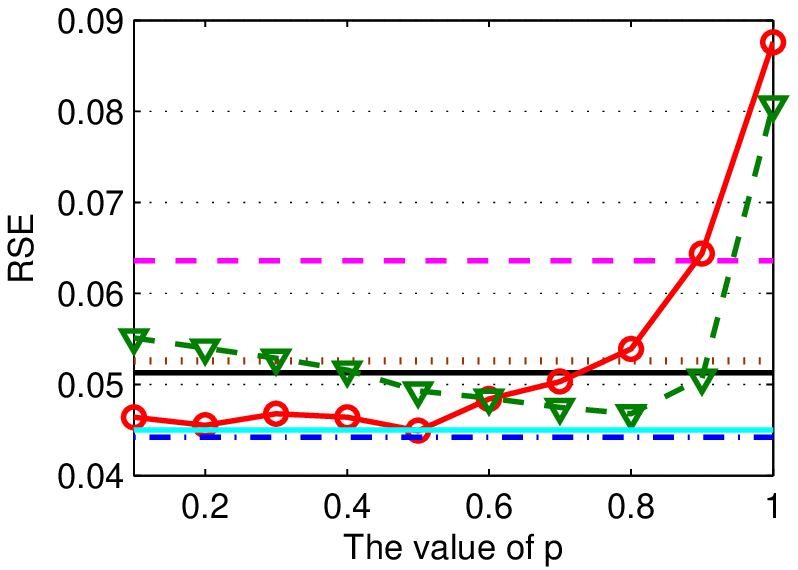}}
\caption{The recovery accuracy of NNLS, ALT, LRMF, IRLS, IRNN, and our Tri-tr and Bi-tr methods on noisy random matrices of size $100\times100$ (the first row) or $200\times200$ (the second row).}
\label{fig_sim1}
\end{center}
\end{figure*}

\subsection{Synthetic Matrix Completion}
The synthetic matrices $X_{0}\!\in\!\mathbb{R}^{m\times n}$ with rank $r$ are generated randomly by the following procedure: the entries of both random matrices $P\!\in\!\mathbb{R}^{m\times r}$ and $Q\!\in\!\mathbb{R}^{n\times r}$ are first generated as independent and identically distributed (i.i.d.) numbers, and then $X_{0}\!=\!PQ^{T}$ is assembled. The experiments are conducted on random matrices with different noise factors, $nf\!=\!0.1$ or $0.2$, where the observed subset is corrupted by i.i.d.\ standard Gaussian random variables as in~\cite{shang:snm}. In both cases, the sampling ratio (SR) is set to 20\% or 30\%. We use the relative standard error ($\textup{RSE}\!:=\!\|X\!-\!X_{0}\|_{F}/\|X_{0}\|_{F}$) as the evaluation measure, where $X$ denotes the recovered matrix.

We compare our methods with two trace norm solvers: NNLS~\cite{toh:apg} and ALT~\cite{hsieh:nnm}, one bilinear spectral regularization method, LRMF~\cite{mitra:lsmf}, and two Schatten-$p$ norm methods, IRLS~\cite{lai:irls} and IRNN~\cite{lu:lrm}. The recovery results of IRLS and IRNN ($p\!\in\!\{0.1,0.2,\ldots,1\}$) on noisy random matrices are shown in Figure~\ref{fig_sim1}, from which we can observe that as a scalable alternative to trace norm regularization, LRMF with relatively small ranks often obtains more accurate solutions than its trace norm counterparts, i.e., NNLS and ALT. If $p$ is chosen from the range of $\{0.3,0.4,0.5,0.6\}$, IRLS and IRNN have similar performance, and usually outperform NNLS, ALT and LRMF in terms of RSE, otherwise they sometimes perform much worse than the latter three methods, especially $p\!=\!1$. This means that both our methods (which are in essence the Schatten-$1/2$ and $1/3$ quasi-norm algorithms) should perform better than them. As expected, the RSE results of both our methods under all of these settings are consistently much better than those of the other approaches. This clearly justifies the usefulness of our Bi-tr and Tri-tr quasi-norm penalties. Moreover, the running time of all these methods on random matrices with different sizes is provided in the Supplementary Materials, which shows that our methods are much faster than the other methods. This confirms that both our methods have very good scalability and can address large-scale problems.

\subsection{Collaborative Filtering}
We test our methods on the real-world recommendation system datasets: MovieLens1M, MovieLens10M and MovieLens20M\footnote{\url{http://www.grouplens.org/node/73}}, and Netflix~\cite{kdd:cup}. We randomly choose 90\% as the training set and the remaining as the testing set, and the experimental results are reported over 10 independent runs. Besides those methods used above, we also compare our methods to one of the fastest methods, LMaFit~\cite{wen:nsor}, and use the root mean squared error (RMSE) as evaluation measure.

\begin{figure}[t]
\begin{center}
\setlength{\belowcaptionskip}{-2pt}
\subfigure[MovieLens1M]{\includegraphics[width=0.492\columnwidth]{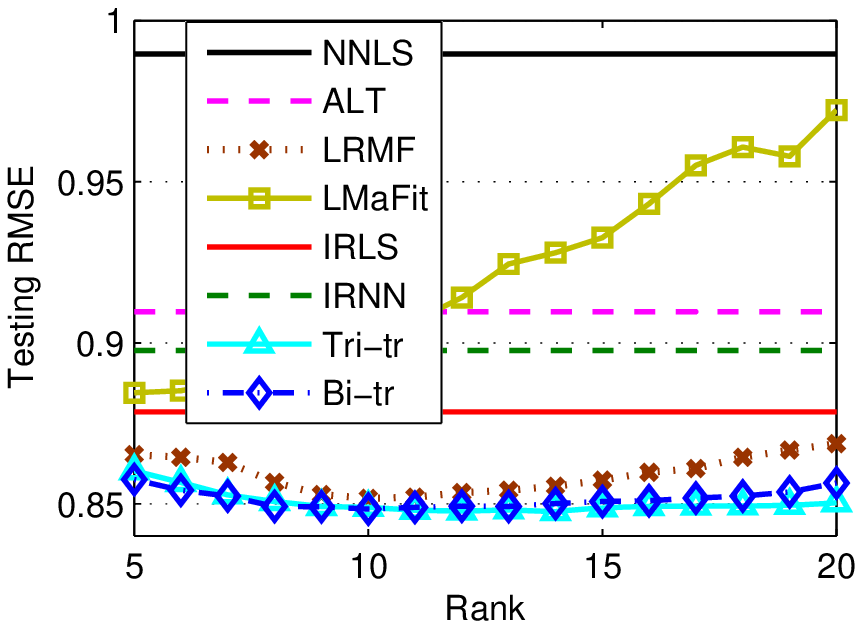}}
\subfigure[MovieLens10M]{\includegraphics[width=0.492\columnwidth]{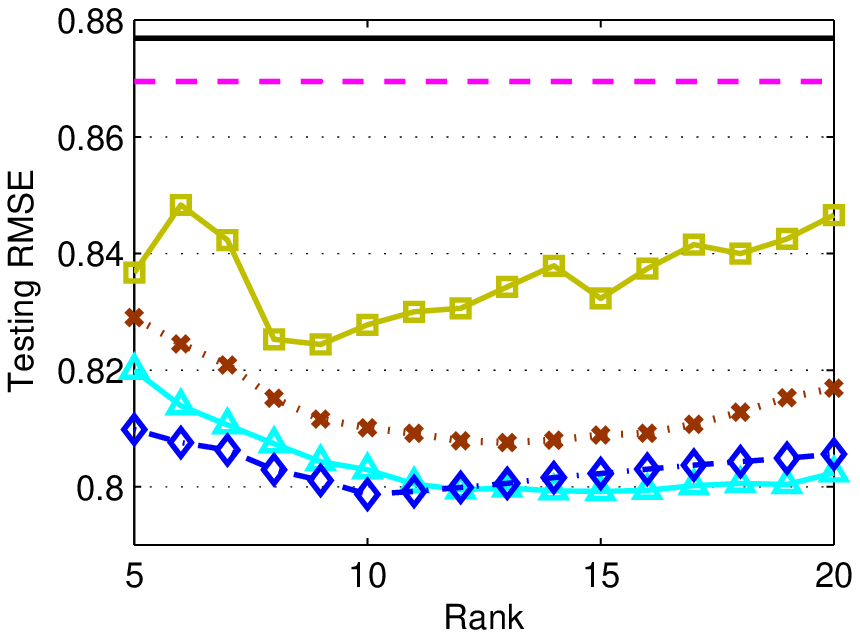}}
\subfigure[MovieLens20M]{\includegraphics[width=0.492\columnwidth]{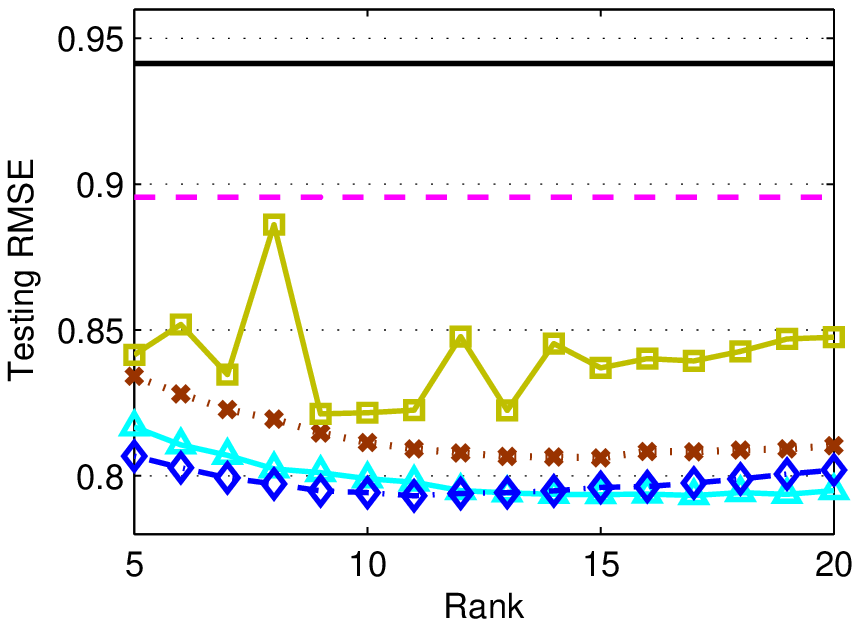}}
\subfigure[Netflix]{\includegraphics[width=0.492\columnwidth]{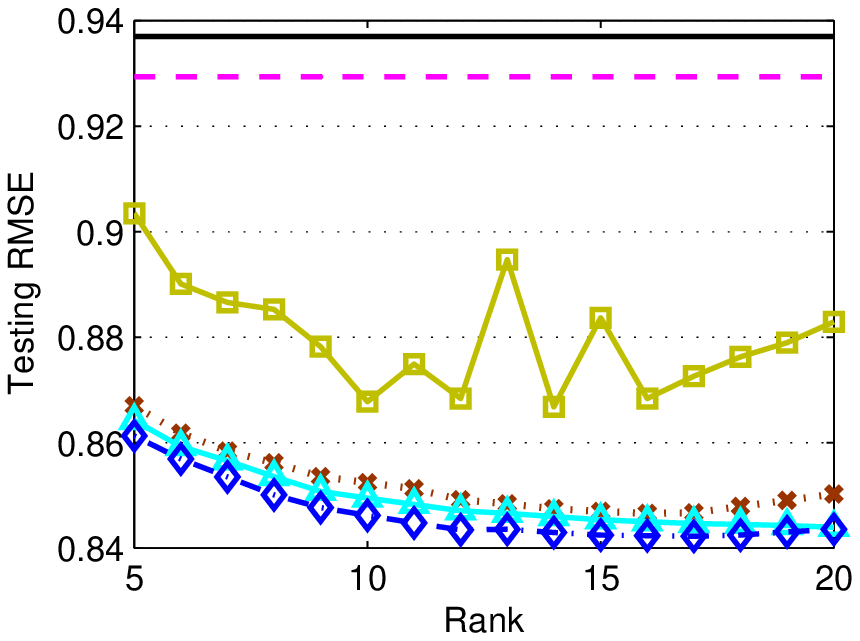}}
\caption{Evolution of the testing RMSE of different methods with ranks varying from 5 to 20.}
\label{fig_sim2}
\end{center}
\end{figure}

\begin{figure*}[t]
\centering
\includegraphics[width=0.137\linewidth]{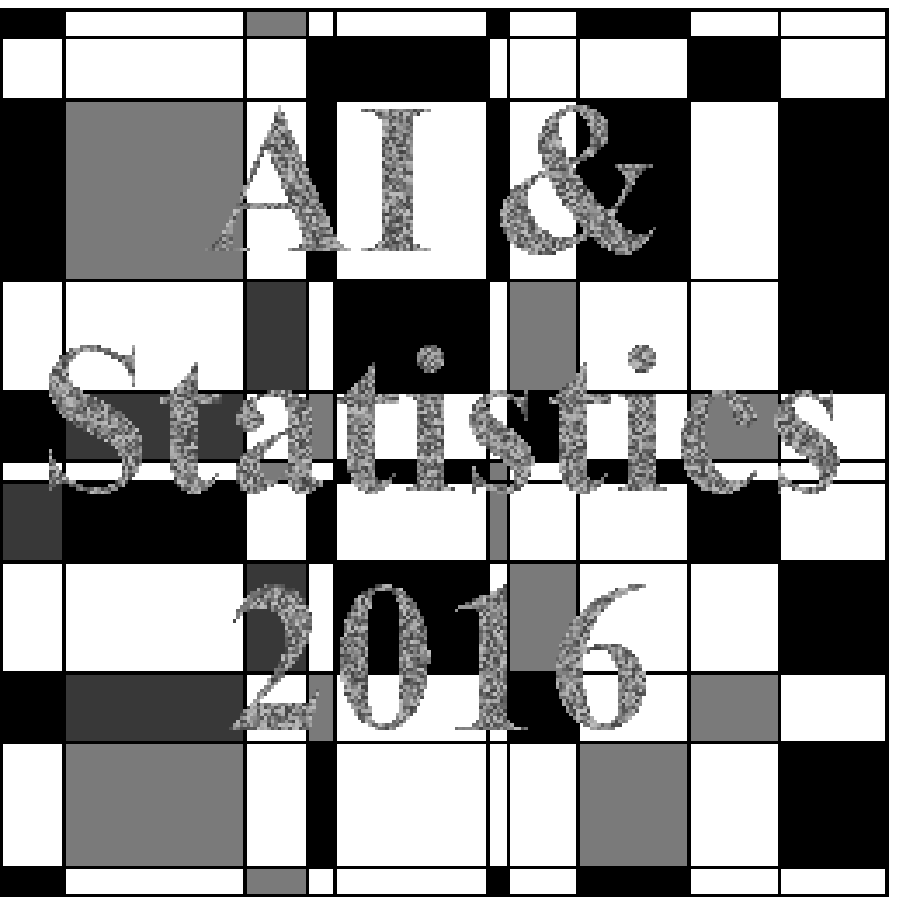}
\includegraphics[width=0.137\linewidth]{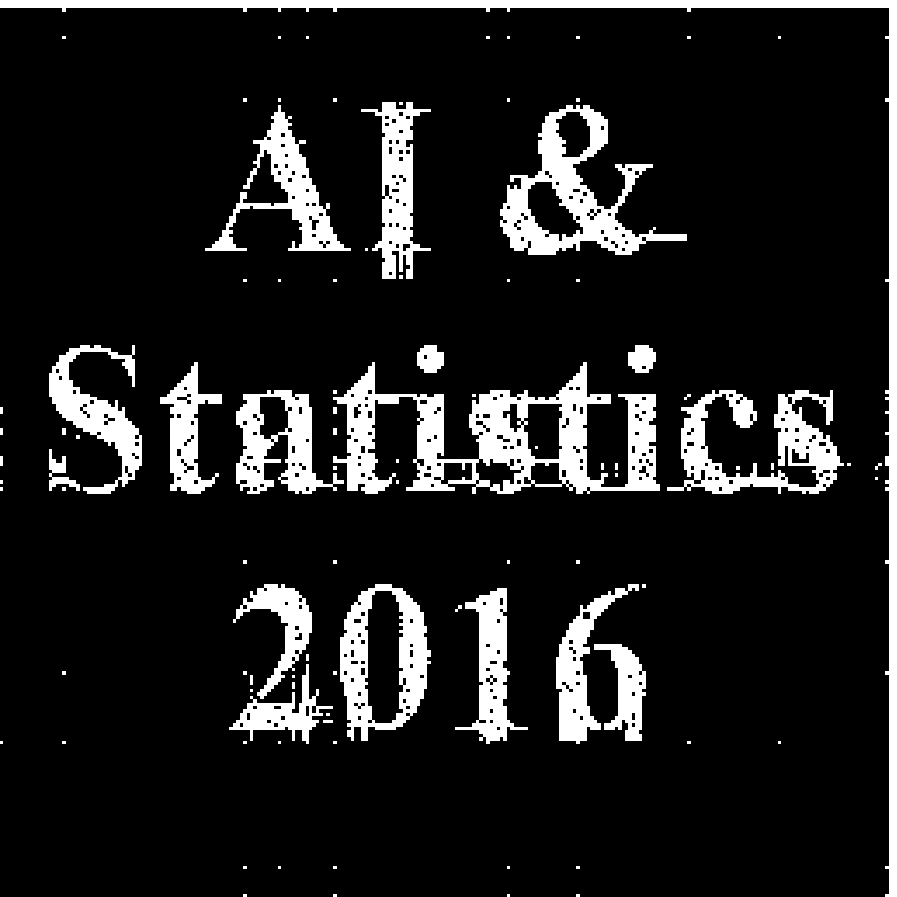}
\includegraphics[width=0.137\linewidth]{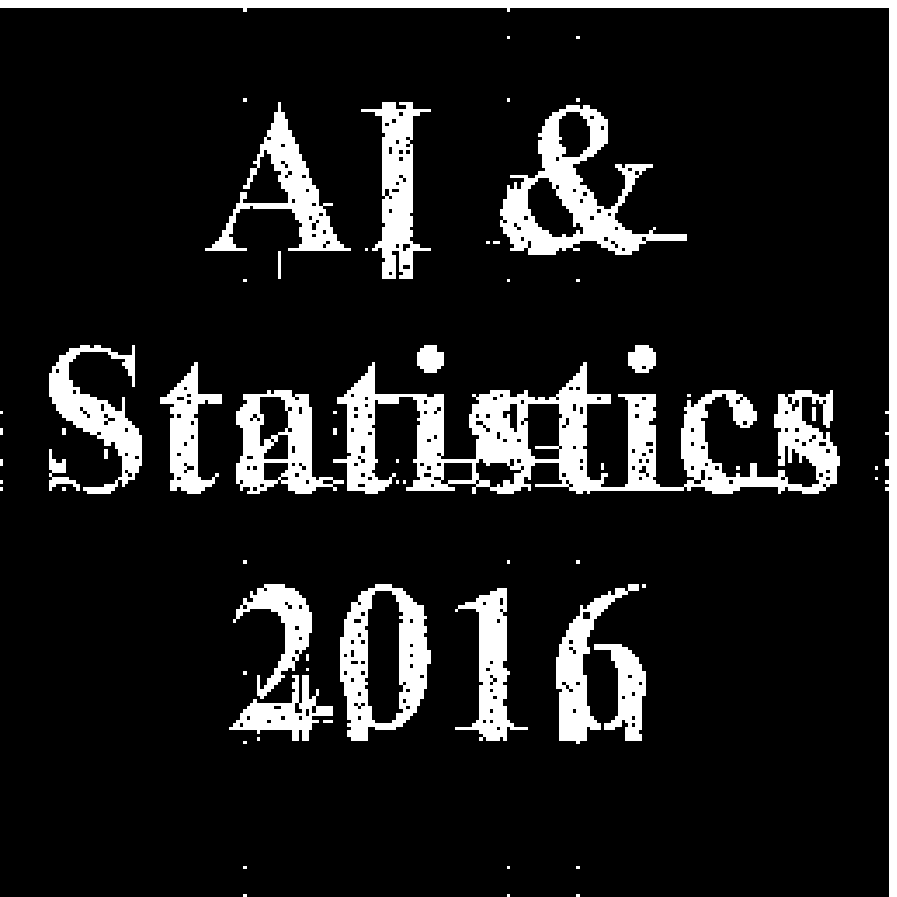}
\includegraphics[width=0.137\linewidth]{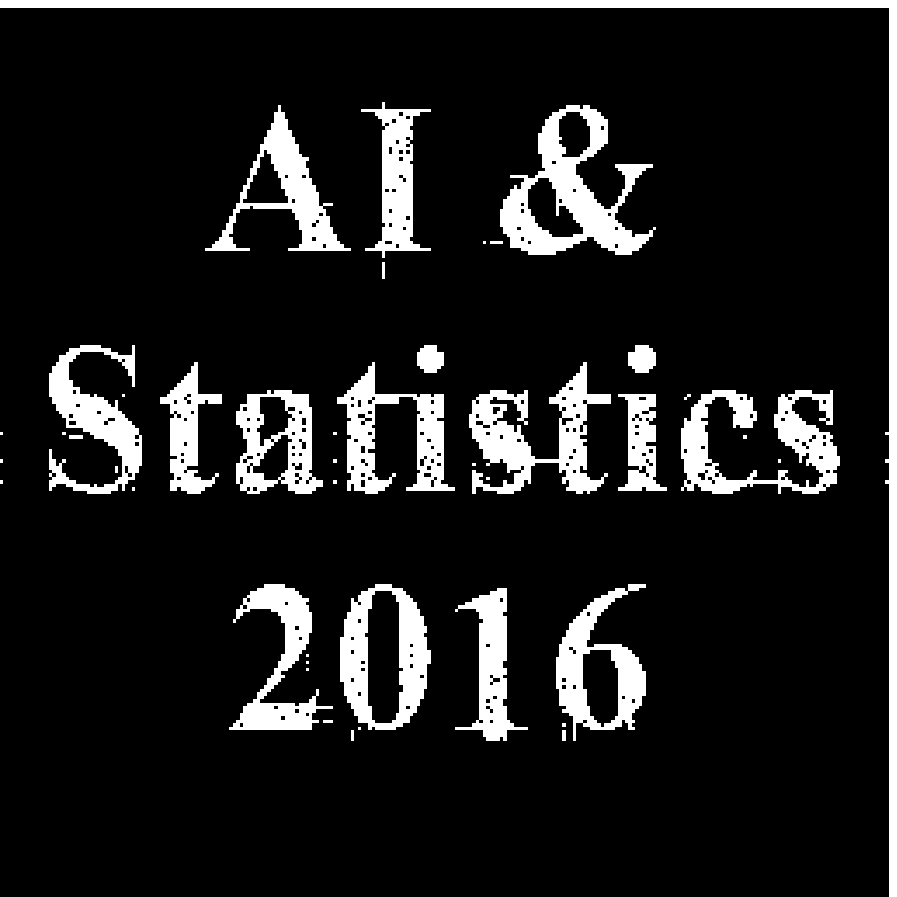}
\includegraphics[width=0.137\linewidth]{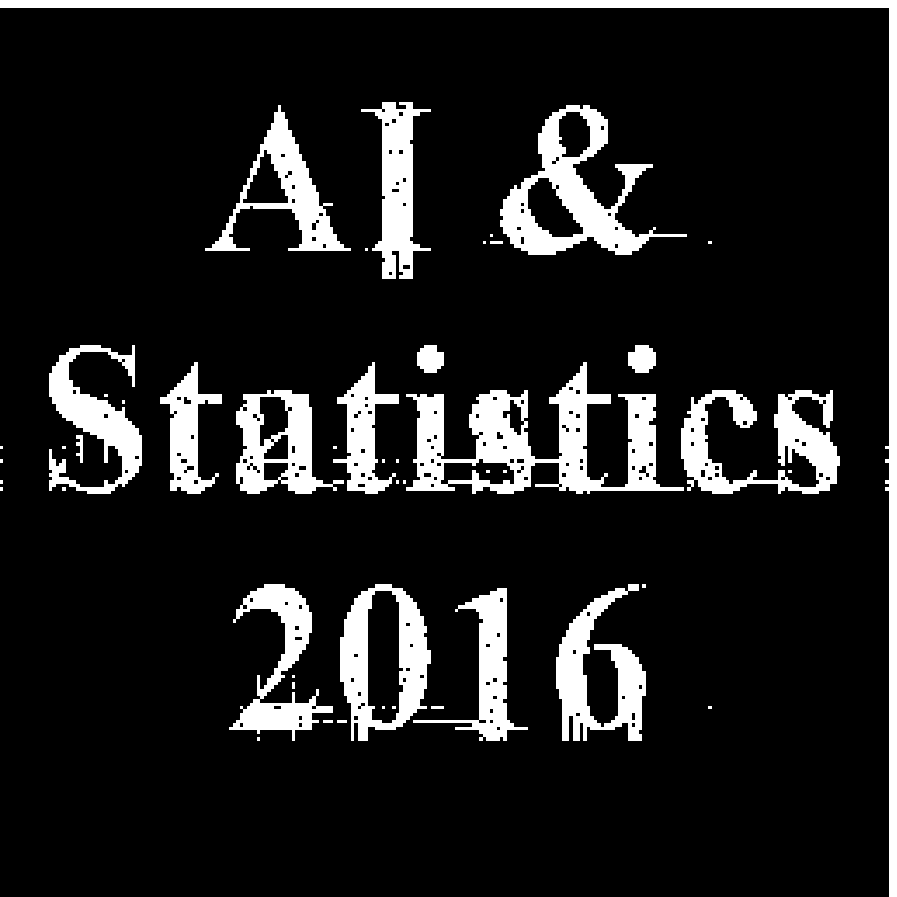}
\includegraphics[width=0.137\linewidth]{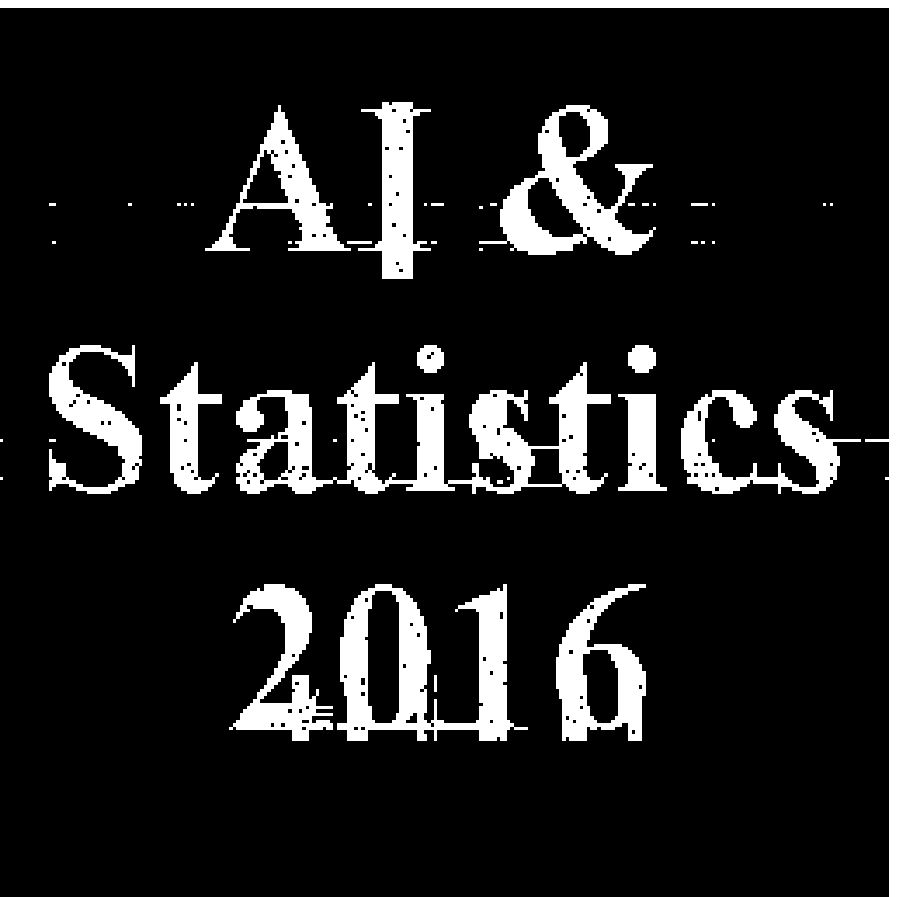}
\includegraphics[width=0.137\linewidth]{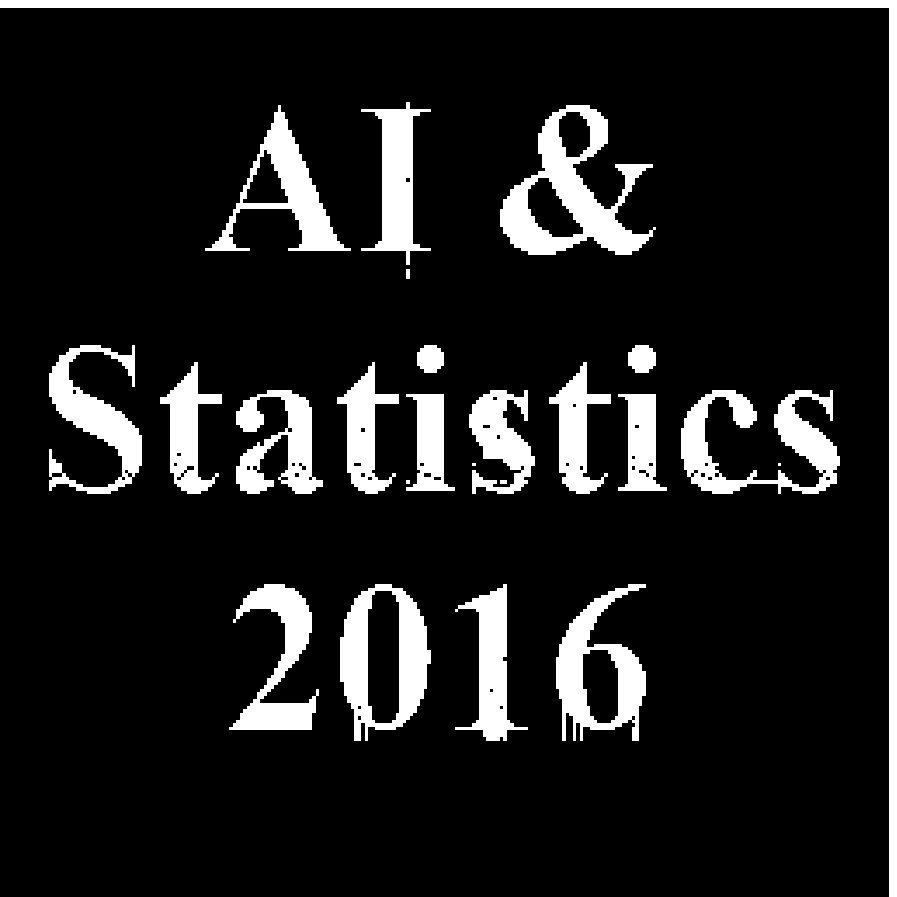}
\subfigure[Images]{\includegraphics[width=0.137\linewidth]{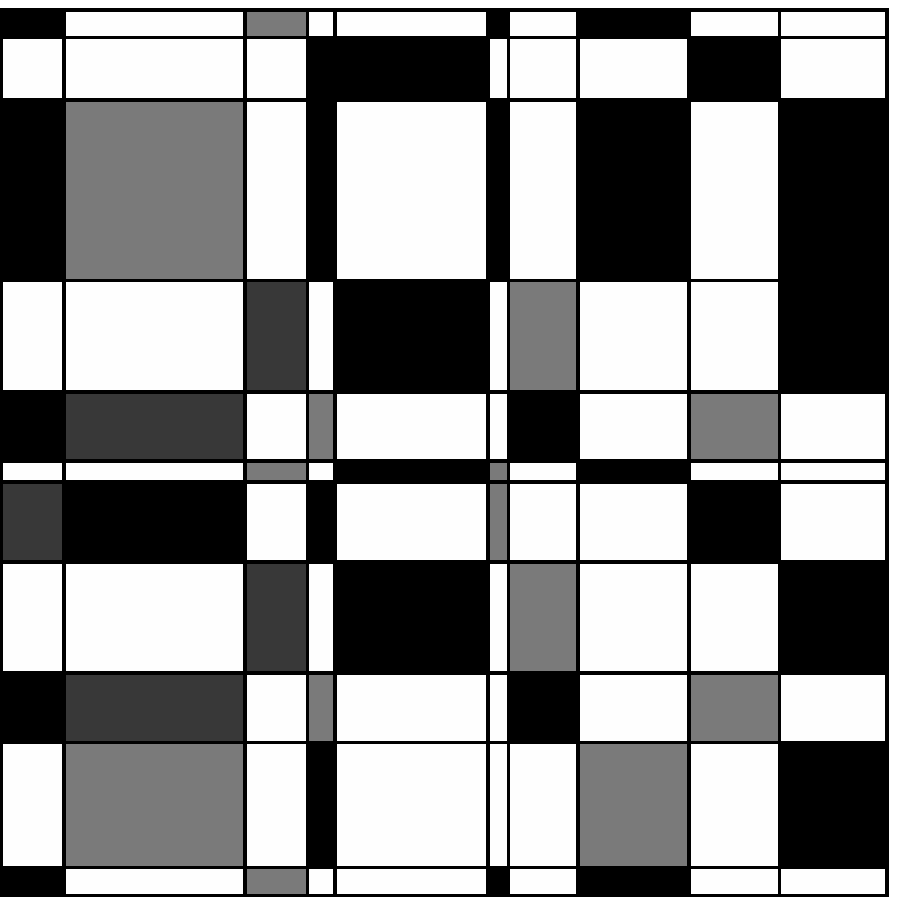}}
\subfigure[PCP]{\includegraphics[width=0.137\linewidth]{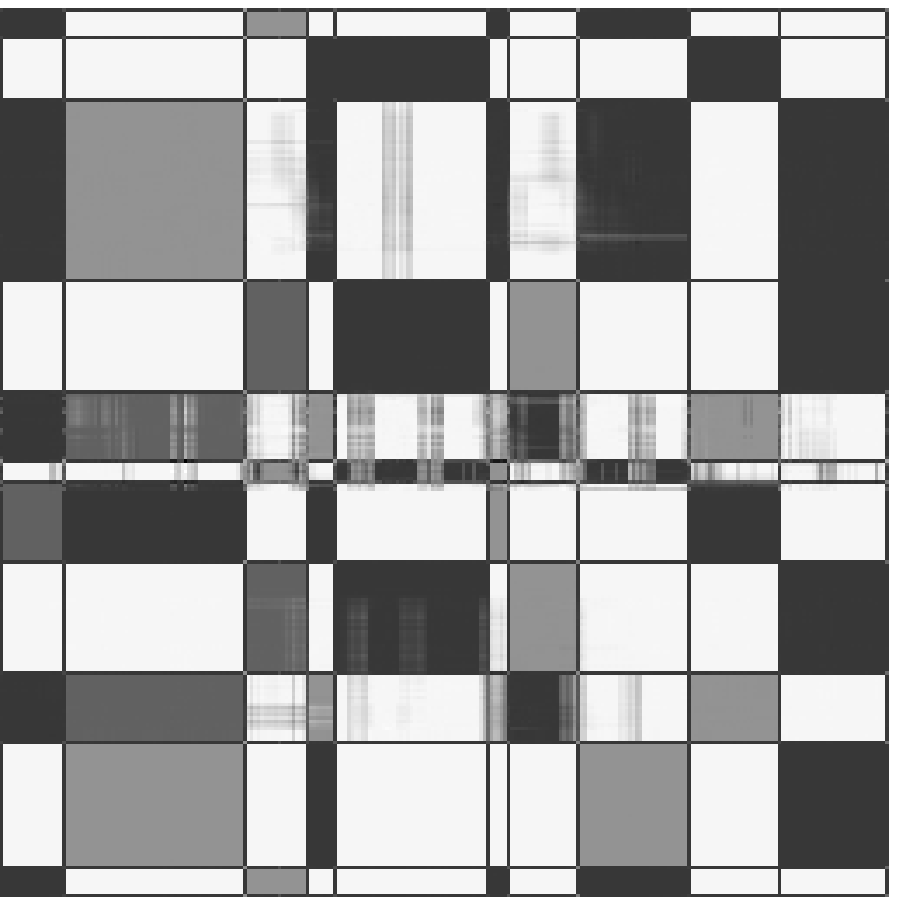}}
\subfigure[LRMF+$\ell_{1}$]{\includegraphics[width=0.137\linewidth]{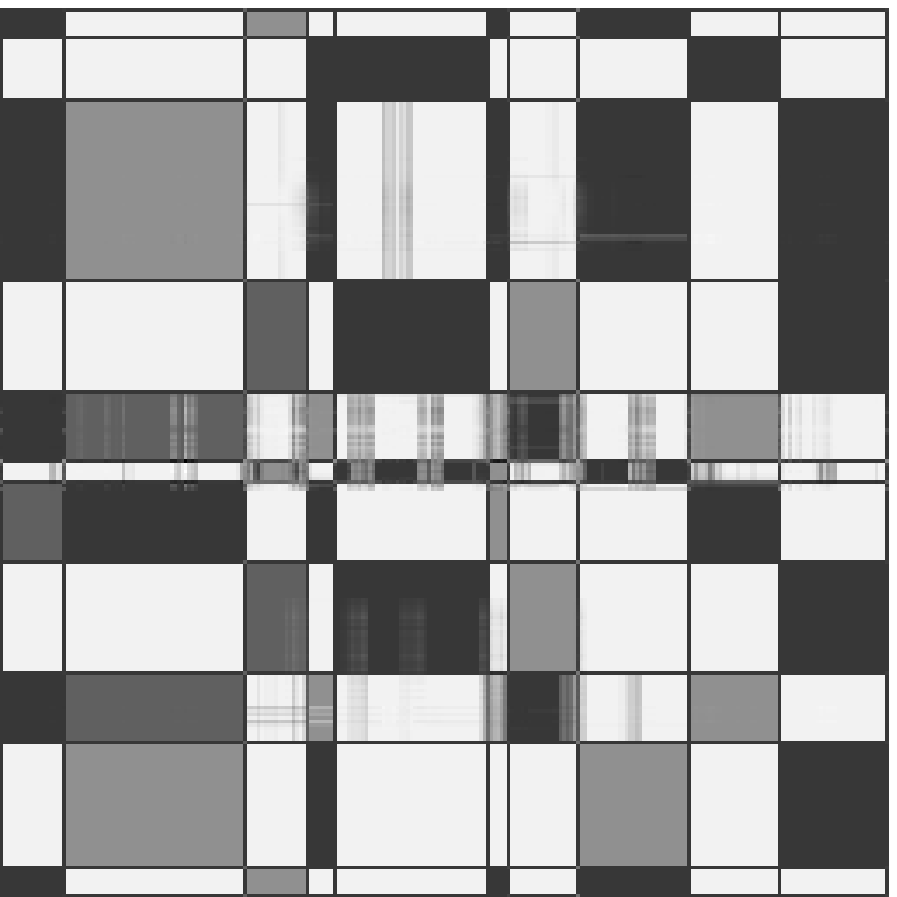}}
\subfigure[$\textup{S}_{p}$+$\ell_{p}$]{\includegraphics[width=0.137\linewidth]{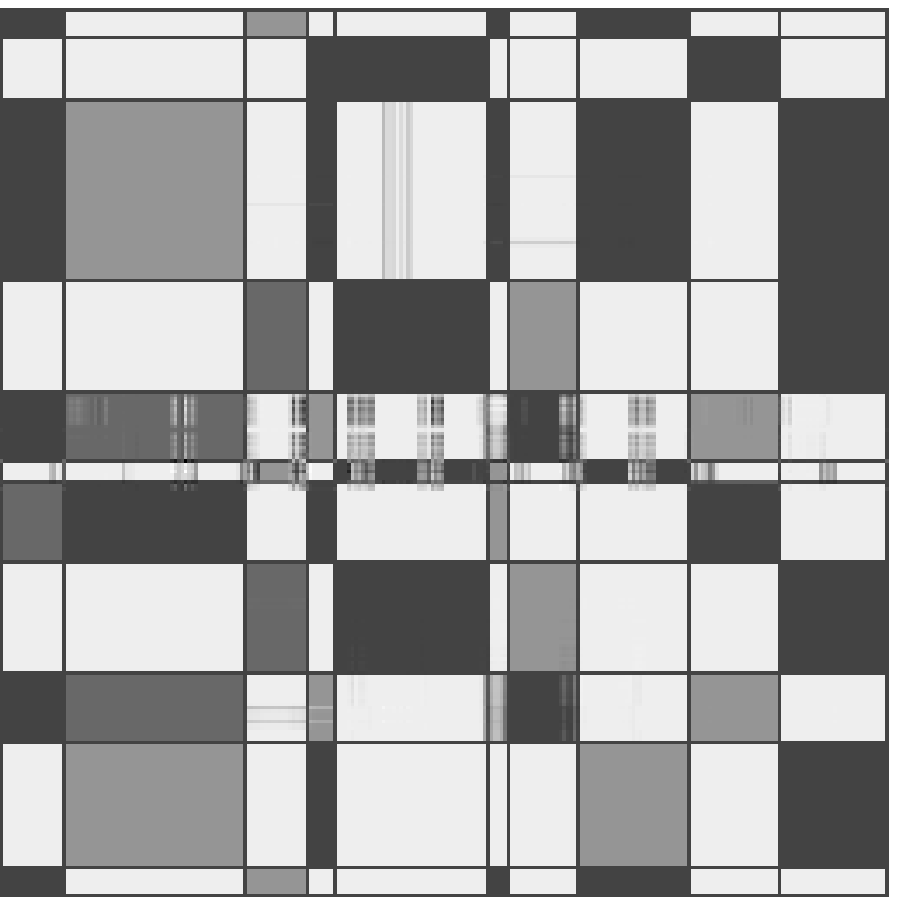}}
\subfigure[Tri-tr+$\ell_{1}$]{\includegraphics[width=0.137\linewidth]{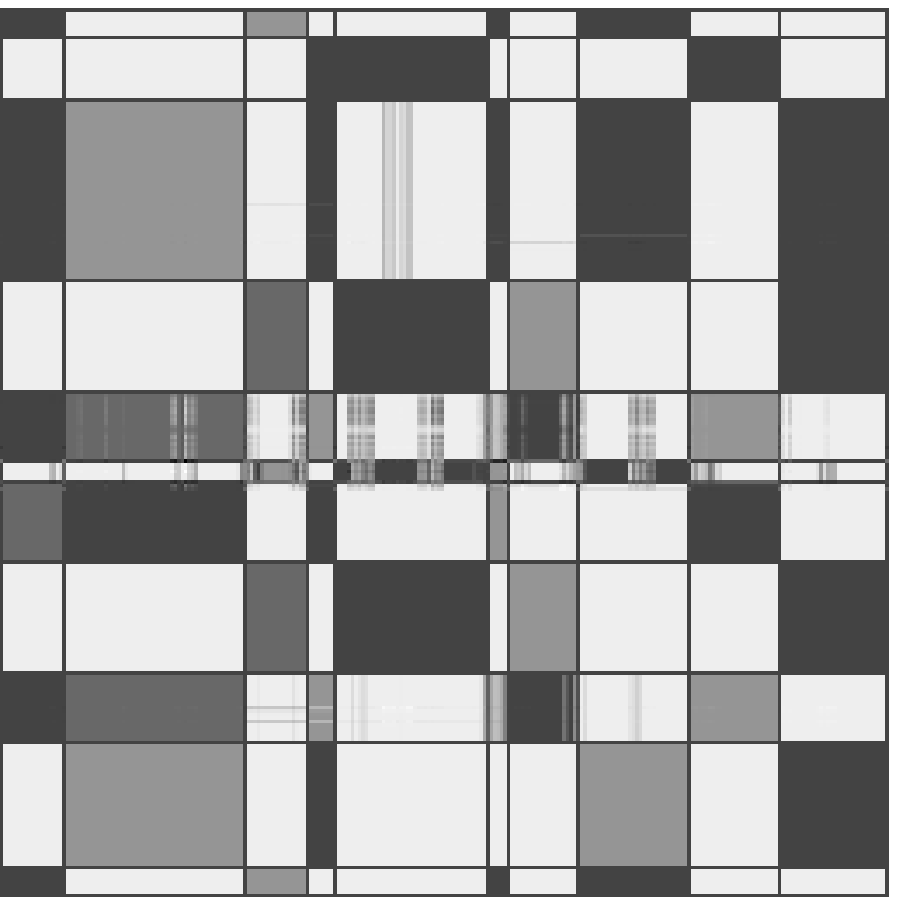}}
\subfigure[Bi-tr+$\ell_{1}$]{\includegraphics[width=0.137\linewidth]{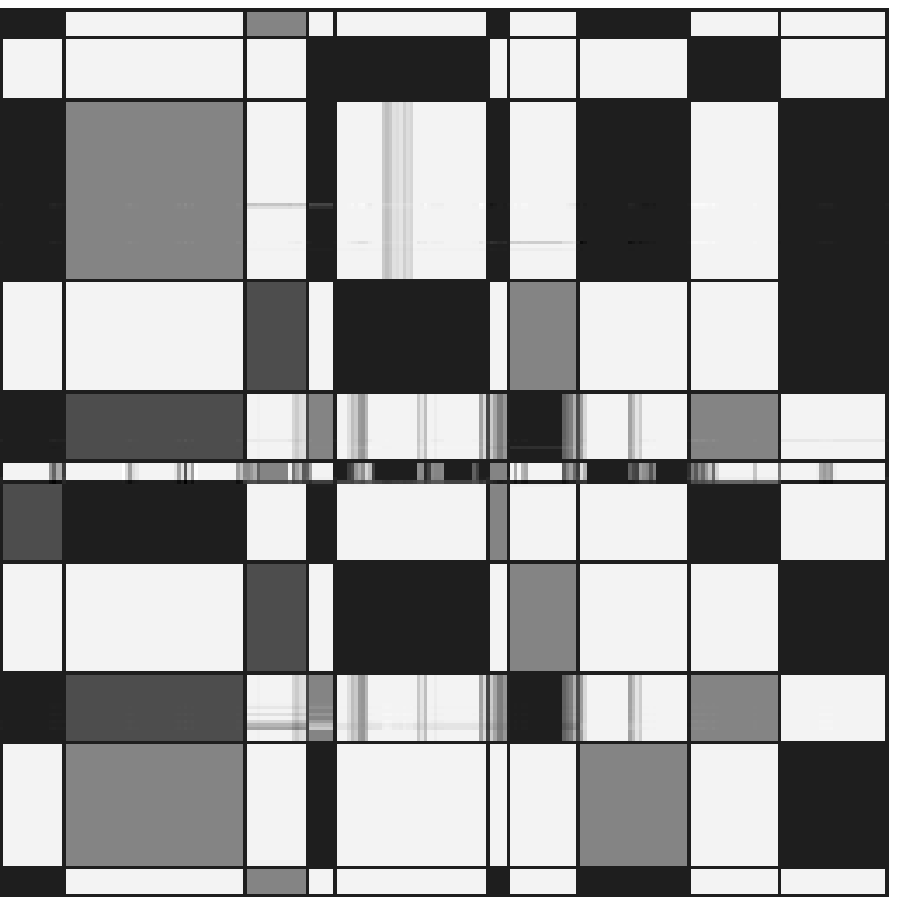}}
\subfigure[Bi-tr+$\ell_{1/2}$]{\includegraphics[width=0.137\linewidth]{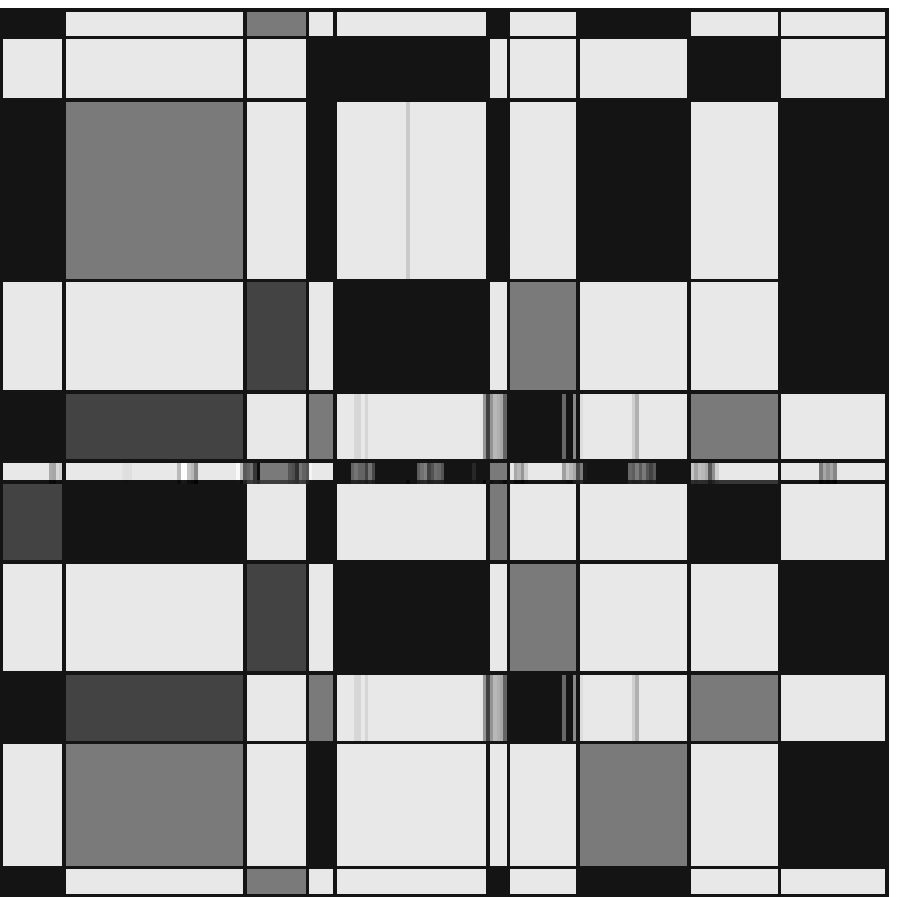}}
\caption{Text separation results. The first and second rows mainly show the detected texts and the recovered background images: (a) Input image (upper) and original image (bottom); (b) AUC: 0.8939, RSE: 0.1494; (c) AUC: 0.9058, RSE: 0.1406; (d) AUC: 0.9425, RSE: 0.1342; (e) AUC: 0.9356, RSE: 0.1320; (f) AUC: 0.9389, RSE: 0.1173; (g) AUC: \textbf{0.9731}, RSE: \textbf{0.0853}.}
\label{fig_sim3}
\end{figure*}

The testing RMSE of all those methods on the four datasets is reported in Figure \ref{fig_sim2}, where the rank varies from 5 to 20 (the running time of all methods are provided in Supplementary Materials). From all these results, we can observe that for these fixed ranks, the matrix factorization methods including LMaFit, LRMF and our methods significantly perform better than the trace norm solvers including NNLS and ALT in terms of RMSE, especially on the three larger datasets, as shown in Figures \ref{fig_sim2}(b)-(d). In most cases, the sophisticated matrix factorization based approaches outperform LMaFit as a baseline method without any regularization term. This suggests that those regularized models can alleviate the over-fitting problem of matrix factorization. The testing RMSE of both our methods varies only slightly when the number of the given rank increases, while that of the other matrix factorization methods changes dramatically. This further means that our methods perform much more robust than them in terms of the given ranks. More importantly, both our methods under all of the rank settings consistently outperform the other methods in terms of prediction accuracy. This confirms that our Bi-tr or Tri-tr quasi-norm regularized models can provide a good estimation of a low-rank matrix. Note that IRLS and IRNN could not run on the three larger datasets due to runtime exceptions. Moreover, our methods are much faster than LRMF, NNLS, ALT, IRLS and IRNN on all these datasets, and are comparable in speed with LMaFit. This shows that our methods have very good scalability and can solve large-scale problems.
\vspace{-2mm}

\subsection{Text Separation}
We conducted an experiment on artificially generated data to separate some text from an image. The ground-truth image is of size $256\!\times\!256$ with rank equal to 10. Figure \ref{fig_sim3}(a) shows the input image together with the original image. The input data are generated by setting 10\% of the randomly selected pixels as missing entries. We compare our Bi-tr+$\ell_{1}$, Tri-tr+$\ell_{1}$ and Bi-tr+$\ell_{1/2}$ methods (see Supplementary Materials for the details) to three state-of-the-art methods, including PCP~\cite{candes:rpca}, LRMF+$\ell_{1}$~\cite{cabral:nnbf} and $\textup{S}_{p}$+$\ell_{p}$~\cite{nie:rmc} with $0\!<\!p\!\leq\!1$. For fairness, we set the rank of all methods to 15, and $\varepsilon\!=\!10^{-4}$ for all these algorithms.

The results of different methods are shown in Figure \ref{fig_sim3}, where the text detection accuracy (the score Area Under the receiver operating characteristic Curve, AUC) and the RSE of low-rank component recovery are reported. Note that we present the best performance results of $\textup{S}_{p}$+$\ell_{p}$ with all choices of $p$ in $\{0.1,0.2,\ldots,0.9\}$. For both low-rank component recovery and text separation, our Bi-tr+$\ell_{1/2}$ method is significantly better than the other methods, not only visually but also quantitatively. In addition, our Bi-tr+$\ell_{1}$ and Tri-tr+$\ell_{1}$ methods have very similar performance to the $\textup{S}_{p}$+$\ell_{p}$ method, and all these three methods outperform PCP and LRMF+$\ell_{1}$ in terms of AUC and RSE. Moreover, the running time of PCP, LRMF+$\ell_{1}$, $\textup{S}_{p}$+$\ell_{p}$, Tri-tr+$\ell_{1}$, Bi-tr+$\ell_{1}$ and Bi-tr+$\ell_{1/2}$ is 31.57sec, 6.91sec, 163.65sec, 0.96sec, 0.57sec and 1.62sec, respectively. In other words, our three methods are at least 7, 12 and 4 times faster than the other methods, respectively. This is a very impressive result as our three methods are nearly 170, 290 or 100 times faster than the most related $\textup{S}_{p}$+$\ell_{p}$ method, which further confirms that our methods have good scalability.

\section{Conclusions}
In this paper, we defined two tractable Schatten quasi-norm formulations, and then proved that they are in essence the Schatten-${1/2}$ and ${1/3}$ quasi-norms, respectively. By applying the two defined quasi-norms to various rank minimization problems, such as MC and RPCA, we achieved some challenging non-smooth and non-convex problems. Then we designed two classes of efficient PALM and LADM algorithms to solve our problems with smooth and non-smooth loss functions, respectively. Finally, we established that each bounded sequence generated by our algorithms converges to a critical point, and also provided the recovery performance guarantees for our algorithms. Experiments on real-world data sets showed that our methods outperform the state-of-the-art methods in terms of both efficiency and effectiveness. For future work, we are interested in analyzing the recovery bound for our algorithms to solve the Bi-tr or Tri-tr quasi-norm regularized problems with non-smooth loss functions.

\subsubsection*{Acknowledgements}
We thank the reviewers for their valuable comments. The authors are supported by the Hong Kong GRF 2150851. The project is funded by Research Committee of CUHK.

\bibliographystyle{unsrt}
\bibliography{aistats}

\begin{thebibliography}{10}

\bibitem{candes:emc}
E.~Cand\`{e}s and B.~Recht.
\newblock Exact matrix completion via convex optimization.
\newblock {\em Found. Comput. Math.}, 9(6):717--772, 2009.

\bibitem{candes:rpca}
E.~Cand\`{e}s, X.~Li, Y.~Ma, and J.~Wright.
\newblock Robust principal component analysis?
\newblock {\em J. ACM}, 58(3):1--37, 2011.

\bibitem{liu:lrr}
G.~Liu, Z.~Lin, and Y.~Yu.
\newblock Robust subspace segmentation by low-rank representation.
\newblock In {\em ICML}, pages 663--670, 2010.

\bibitem{hsieh:nnm}
C.~Hsieh and P.~A. Olsen.
\newblock Nuclear norm minimization via active subspace selection.
\newblock In {\em ICML}, pages 575--583, 2014.

\bibitem{argyriou:mtl}
A.~Argyriou, C.~A. Micchelli, M.~Pontil, and Y.~Ying.
\newblock A spectral regularization framework for multi-task structure
  learning.
\newblock In {\em NIPS}, pages 25--32, 2007.

\bibitem{fazel:rmh}
M.~Fazel, H.~Hindi, and S.~P. Boyd.
\newblock A rank minimization heuristic with application to minimum order
  system approximation.
\newblock In {\em ACC}, pages 4734--4739, 2001.

\bibitem{recht:nnm}
B.~Recht, M.~Fazel, and P.~A. Parrilo.
\newblock Guaranteed minimum-rank solutions of linear matrix equations via
  nuclear norm minimization.
\newblock {\em SIAM Rev.}, 52:471--501, 2010.

\bibitem{fan:ve}
J.~Fan and R.~Li.
\newblock Variable selection via nonconcave penalized likelihood and its
  {O}racle properties.
\newblock {\em J. Am. Statist. Assoc.}, 96:1348--1361, 2001.

\bibitem{lu:irsvm}
Z.~Lu and Y.~Zhang.
\newblock Schatten-$p$ quasi-norm regularized matrix optimization via iterative
  reweighted singular value minimization.
\newblock {\em arXiv:1401.0869v2}, 2015.

\bibitem{lu:lrm}
C.~Lu, J.~Tang, S.~Yan, and Z.~Lin.
\newblock Generalized nonconvex nonsmooth low-rank minimization.
\newblock In {\em CVPR}, pages 4130--4137, 2014.

\bibitem{nie:rmc}
F.~Nie, H.~Wang, X.~Cai, H.~Huang, and C.~Ding.
\newblock Robust matrix completion via joint {S}chatten $p$-norm and
  ${L}_{p}$-norm minimization.
\newblock In {\em ICDM}, pages 566--574, 2012.

\bibitem{majumdar:mri}
A.~Majumdar and R.~K. Ward.
\newblock An algorithm for sparse {MRI} reconstruction by {S}chatten $p$-norm
  minimization.
\newblock {\em Magn. Reson. Imaging}, 29:408--417, 2011.

\bibitem{mohan:mrm}
K.~Mohan and M.~Fazel.
\newblock Iterative reweighted algorithms for matrix rank minimization.
\newblock {\em J. Mach. Learn. Res.}, 13:3441--3473, 2012.

\bibitem{lai:irls}
M.~Lai, Y.~Xu, and W.~Yin.
\newblock Improved iteratively rewighted least squares for unconstrained
  smoothed $\ell_{p}$ minimization.
\newblock {\em SIAM J. Numer. Anal.}, 51(2):927--957, 2013.

\bibitem{marjanovic:mc}
G.~Marjanovic and V.~Solo.
\newblock On $\ell_{p}$ optimization and matrix completion.
\newblock {\em IEEE Trans. Signal Process.}, 60(11):5714--5724, 2012.

\bibitem{nie:lrmr}
F.~Nie, H.~Huang, and C.~Ding.
\newblock Low-rank matrix recovery via efficient {S}chatten $p$-norm
  minimization.
\newblock In {\em AAAI}, pages 655--661, 2012.

\bibitem{zhang:ncmr}
M.~Zhang, Z.~Huang, and Y.~Zhang.
\newblock Restricted $p$-isometry properties of nonconvex matrix recovery.
\newblock {\em IEEE Trans. Inform. Theory}, 59(7):4316--4323, 2013.

\bibitem{shang:snm}
F.~Shang, Y.~Liu, and J.~Cheng.
\newblock Scalable algorithms for tractable {S}chatten quasi-norm minimization.
\newblock In {\em AAAI}, pages 2016--2022, 2016.

\bibitem{srebro:mmmf}
N.~Srebro, J.~Rennie, and T.~Jaakkola.
\newblock Maximum-margin matrix factorization.
\newblock In {\em NIPS}, pages 1329--1336, 2004.

\bibitem{mitra:lsmf}
K.~Mitra, S.~Sheorey, and R.~Chellappa.
\newblock Large-scale matrix factorization with missing data under additional
  constraints.
\newblock In {\em NIPS}, pages 1642--1650, 2010.

\bibitem{aravkin:rpca}
A.~Aravkin, R.~Kumar, H.~Mansour, B.~Recht, and F.~J. Herrmann.
\newblock Fast methods for denoising matrix completion formulations, with
  applications to robust seismic data interpolation.
\newblock {\em SIAM J. Sci. Comput.}, 36(5):S237--S266, 2014.

\bibitem{foucart:lp}
S.~Foucart and M.~Lai.
\newblock Sparsest solutions of underdetermined linear systems via
  $\ell_{q}$-minimization for $0\!<\!q\!\leq\! 1$.
\newblock {\em Appl. Comput. Harmon. Anal.}, 26:397--407, 2009.

\bibitem{liu:nnr}
Y.~Liu, F.~Shang, H.~Cheng, and J.~Cheng.
\newblock A {G}rassmannian manifold algorithm for nuclear norm regularized
  least squares problems.
\newblock In {\em UAI}, pages 515--524, 2014.

\bibitem{bian:ipa}
W.~Bian, X.~Chen, and Y.~Ye.
\newblock Complexity analysis of interior point algorithms for non-{L}ipschitz
  and nonconvex minimization.
\newblock {\em Math. Program.}, 149:301--327, 2015.

\bibitem{shang:rpca}
F.~Shang, Y.~Liu, J.~Cheng, and H.~Cheng.
\newblock Robust principal component analysis with missing data.
\newblock In {\em CIKM}, pages 1149--1158, 2014.

\bibitem{shang:rbf}
F.~Shang, Y.~Liu, J.~Cheng, and H.~Cheng.
\newblock Recovering low-rank and sparse matrices via robust bilateral
  factorization.
\newblock In {\em ICDM}, pages 965--970, 2014.

\bibitem{jaggi:fw}
M.~Jaggi.
\newblock Revisiting {F}rank-{W}olfe: Projection-free sparse convex
  optimization.
\newblock In {\em ICML}, pages 427--435, 2013.

\bibitem{bolte:palm}
J.~Bolte, S.~Sabach, and M.~Teboulle.
\newblock Proximal alternating linearized minimization for nonconvex and
  nonsmooth problems.
\newblock {\em Math. Program.}, 146:459--494, 2014.

\bibitem{yang:adm}
J.~Yang and X.~Yuan.
\newblock Linearized augmented {L}agrangian and alternating direction methods
  for nuclear norm minimization.
\newblock {\em Math. Comp.}, 82:301--329, 2013.

\bibitem{cai:svt}
J.~Cai, E.~Cand\`{e}s, and Z.~Shen.
\newblock A singular value thresholding algorithm for matrix completion.
\newblock {\em SIAM J. Optim.}, 20(4):1956--1982, 2010.

\bibitem{daubechies:it}
I.~Daubechies, M.~Defrise, and C.~DeMol.
\newblock An iterative thresholding algorithm for linear inverse problems with
  a sparsity constraint.
\newblock {\em Commun. Pur. Appl. Math.}, 57(11):1413--1457, 2004.

\bibitem{lin:ladmm}
Z.~Lin, R.~Liu, and Z.~Su.
\newblock Linearized alternating direction method with adaptive penalty for
  low-rank representation.
\newblock In {\em NIPS}, pages 612--620, 2011.

\bibitem{attouch:cr}
H.~Attouch and J.~Bolte.
\newblock On the convergence of the proximal algorithm for nonsmooth functions
  involving analytic features.
\newblock {\em Math. Program.}, 116:5--16, 2009.

\bibitem{oymak:lrm}
S.~Oymak, K.~Mohan, M.~Fazel, and B.~Hassibi.
\newblock A simplified approach to recovery conditions for low rank matrices.
\newblock In {\em ISIT}, pages 2318--2322, 2011.

\bibitem{negahban:rsc}
S.~Negahban, P.~Ravikumar, M.~J. Wainwright, and B.~Yu.
\newblock A unified framework for highdimensional analysis of {M}-estimators
  with decomposable regularizers.
\newblock In {\em NIPS}, pages 1348--1356, 2009.

\bibitem{rohde:lrm}
A.~Rohde and A.~B. Tsybakov.
\newblock Estimation of high-dimensional low-rank matrices.
\newblock {\em Ann. Statist.}, 39(2):887--930, 2011.

\bibitem{candes:mcn}
E.~Cand\`{e}s and Y.~Plan.
\newblock Matrix completion with noise.
\newblock {\em Proc. IEEE}, 98(6):925--936, 2010.

\bibitem{jain:svp}
P.~Jain, R.~Meka, and I.~Dhillon.
\newblock Guaranteed rank minimization via singular value projection.
\newblock In {\em NIPS}, pages 937--945, 2010.

\bibitem{keshavan:mc}
R.~Keshavan, A.~Montanari, and S.~Oh.
\newblock Matrix completion from a few entries.
\newblock {\em IEEE Trans. Inform. Theory}, 56(6):2980--2998, 2010.

\bibitem{toh:apg}
K.-C. Toh and S.~Yun.
\newblock An accelerated proximal gradient algorithm for nuclear norm
  regularized least squares problems.
\newblock {\em Pac. J. Optim.}, 6:615--640, 2010.

\bibitem{kdd:cup}
KDDCup.
\newblock {ACM} {SIGKDD} and {N}etflix.
\newblock In {\em Proc. {KDD} {C}up and {W}orkshop}, 2007.

\bibitem{wen:nsor}
Z.~Wen, W.~Yin, and Y.~Zhang.
\newblock Solving a low-rank factorization model for matrix completion by a
  nonlinear successive over-relaxation algorithm.
\newblock {\em Math. Prog. Comp.}, 4(4):333--361, 2012.

\bibitem{cabral:nnbf}
R.~Cabral, F.~Torre, J.~Costeira, and A.~Bernardino.
\newblock Unifying nuclear norm and bilinear factorization approaches for
  low-rank matrix decomposition.
\newblock In {\em ICCV}, pages 2488--2495, 2013.

\bibitem{mazumder:sr}
R.~Mazumder, T.~Hastie, and R.~Tibshirani.
\newblock Spectral regularization algorithms for learning large incomplete
  matrices.
\newblock {\em J. Mach. Learn. Res.}, 11:2287--2322, 2010.

\bibitem{bertsekas:np}
D.~P. Bertsekas.
\newblock {\em Nonlinear Programming}.
\newblock The 2nd edition, Athena Scientific, Belmont, 2004.

\bibitem{yue:lrmr}
M.~C. Yue and A.~M.~C. So.
\newblock A perturbation inequality for concave functions of singular values
  and its applications in low-rank matrix recovery.
\newblock {\em Appl. Comput. Harmon. Anal.}, 40(2):396--416, 2016.

\bibitem{wang:mfcf}
Y.~Wang and H.~Xu.
\newblock Stability of matrix factorization for collaborative filtering.
\newblock In {\em ICML}, pages 417--424, 2012.

\bibitem{krishnan:hlp}
D.~Krishnan and R.~Fergus.
\newblock Fast image deconvolution using hyper-{L}aplacian priors.
\newblock In {\em NIPS}, pages 1033--1041, 2009.

\bibitem{zeng:l12}
J.~Zeng, S.~Lin, Y.~Wang, and Z.~Xu.
\newblock ${L}_{1/2}$ regularization: Convergence of iterative half
  thresholding algorithm.
\newblock {\em IEEE Trans. Signal Process.}, 62(9):2317--2329, 2014.

\bibitem{larsen:svd}
R.~Larsen.
\newblock {PROPACK}-software for large and sparse {SVD} calculations.
\newblock {\em Available from \url{http://sun.stanford.edu/srmunk/PROPACK/}},
  2005.

\end{thebibliography}

\twocolumn[

\aistatstitle{Supplementary Materials for ``Tractable and Scalable Schatten Quasi-Norm Approximations for Rank Minimization"}

\aistatsauthor{ Fanhua Shang \And Yuanyuan Liu \And James Cheng }
\aistatsaddress{ Department of Computer Science and Engineering, The Chinese University of Hong Kong }~\\]

In this supplementary material, we give the detailed proofs of some lemmas, properties and theorems, as well as some additional experimental results on synthetic data and four recommendation system data sets.

\setcounter{page}{11}

\section{More Notations}
$\mathbb{R}^{n}$ denotes the $n$-dimensional Euclidean space, and the set of all $m\times n$ matrices with real entries is denoted by $\mathbb{R}^{m\times n}$. Given matrices $X$ and $Y\in\mathbb{R}^{m\times n}$, the inner product is defined by $\langle X, Y\rangle:=\textrm{Tr}(X^{T}Y)$, where $\textrm{Tr}(\cdot)$ denotes the trace of a matrix. $\|X\|_{2}$ is the spectral norm and is equal to the maximum singular value of $X$. $I$ denotes an identity matrix.

For any vector $x\in\mathbb{R}^{n}$, its $\ell_{p}$ quasi-norm for $0<p<1$ is defined as
\begin{displaymath}
\|x\|_{p}=\left(\sum_{i}|x_{i}|^{p}\right)^{1/p}.
\end{displaymath}
In addition, the $\ell_{1}$-norm and the $\ell_{2}$-norm of $x$ are $\|x\|_{1}=\sum_{i}|x_{i}|$ and $\|x\|_{2}=\sqrt{\sum_{i}x^{2}_{i}}$, respectively.

For any matrix $X\in\mathbb{R}^{m\times n}$, we assume the singular values of $X$ are ordered as $\sigma_{1}(X)\geq \sigma_{2}(X)\geq\cdots\geq\sigma_{r}(X)>\sigma_{r+1}(X)=\cdots=\sigma_{\min(m,n)}(X)=0$, where $r=\textrm{rank}(X)$. By writing $X=U\Sigma V^{T}$ in its standard singular value decomposition (SVD), we can extend $X=U\Sigma V^{T}$ to the following definitions.

The Schatten-$p$ quasi-norm ($0<p<1$) of a matrix $X\in\mathbb{R}^{m\times n}$ is defined as follows:
\begin{displaymath}
\|X\|_{S_{p}}=\left(\sum^{\min(m,n)}_{i=1}\left(\sigma_{i}(X)\right)^{p}\right)^{1/p}.
\end{displaymath}

The trace norm (also called the nuclear norm or the Schatten-1 norm) of $X$ is defined as
\begin{displaymath}
\|X\|_{\textup{tr}}=\sum^{\min(m,n)}_{i=1}\sigma_{i}(X).
\end{displaymath}

The Frobenius norm (also called the Schatten-2 norm) of $X$ is defined as
\begin{displaymath}
\|X\|_{F}=\sqrt{\textrm{Tr}\left(X^{T}X\right)}=\sqrt{\sum^{\min(m,n)}_{i=1}\sigma^{2}_{i}(X)}.
\end{displaymath}

\setcounter{equation}{15}
\setcounter{lemma}{1}
\setcounter{definition}{2}
~\\
\section{Proof of Theorem 1}
In the following, we will first prove that the bi-trace norm $\|\!\cdot\!\|_{\textup{Bi-tr}}$ is a quasi-norm.

\begin{proof}
By the definition of the bi-trace norm, for any $a$, $a_{1}$, $a_{2}\in \mathbb{R}$ and $a=a_{1}a_{2}$, we have
\begin{displaymath}
\begin{split}
\|aX\|_{\textup{Bi-tr}}&=\min_{aX=(a_{1}U)(a_{2}V^{T})}\|a_{1}U\|_{\textup{tr}}\|a_{2}V\|_{\textup{tr}}\\
&=\min_{X=UV^{T}}|a|\,\|U\|_{\textup{tr}}\|V\|_{\textup{tr}}\\
&=|a|\min_{X=UV^{T}}\|U\|_{\textup{tr}}\|V\|_{\textup{tr}}\\
&=|a|\,\|X\|_{\textup{Bi-tr}}.
\end{split}
\end{displaymath}

By Lemma 1 of the main paper, i.e., $\|X\|_{\textup{tr}}\!=\!\min_{X=UV^{T}}\|U\|_{F}\|V\|_{F}$, and Lemma 6 in~\cite{mazumder:sr}, there exist both matrices $\widehat{U}\!=\!U_{(d)}\Sigma^{1/2}_{(d)}$ and $\widehat{V}\!=\!V_{(d)}\Sigma^{1/2}_{(d)}$ (which are constructed in the same way as $U_{\star}$ and $V_{\star}$ in Section 5.1) such that $\|X\|_{\textup{tr}}\!=\!\|\widehat{U}\|_{F}\|\widehat{V}\|_{F}$ with the SVD of $X$, i.e., $X\!=\!U\Sigma V^{T}$. By the fact that $\|X\|_{\textup{tr}}\leq\sqrt{\textrm{rank}(X)}\|X\|_{F}$, we have
\begin{displaymath}
\begin{split}
\|X\|_{\textup{Bi-tr}}&=\min_{X=UV^{T}}\|U\|_{\textup{tr}}\|V\|_{\textup{tr}}\\
&\leq \|\widehat{U}\|_{\textup{tr}}\|\widehat{V}\|_{\textup{tr}}\\
&\leq \sqrt{\textrm{rank}(X)}\sqrt{\textrm{rank}(X)}\|\widehat{U}\|_{F}\|\widehat{V}\|_{F}\\
&\leq \textrm{rank}(X)\|X\|_{\textup{tr}}.
\end{split}
\end{displaymath}
If $X\neq0$, then $\textrm{rank}(X)\geq 1$. On the other hand, we also have
\begin{displaymath}
\|X\|_{\textup{tr}}\leq \|X\|_{\textup{Bi-tr}}.
\end{displaymath}

\onecolumn

By the above properties, there exists a constant $\alpha\geq 1$ such that the following holds for all $X,Y\in\mathbb{R}^{m\times n}$
\begin{displaymath}
\begin{split}
\|X+Y\|_{\textup{Bi-tr}}&\leq \alpha\|X+Y\|_{\textup{tr}}\\
&\leq \alpha(\|X\|_{\textup{tr}}+\|Y\|_{\textup{tr}})\\
&\leq \alpha(\|X\|_{\textup{Bi-tr}}+\|Y\|_{\textup{Bi-tr}}).
\end{split}
\end{displaymath}

$\forall X\in\mathbb{R}^{m\times n}$ and $X=UV^{T}$, we have
\begin{displaymath}
\|X\|_{\textup{Bi-tr}}=\min_{X=UV^{T}}\|U\|_{\textup{tr}}\|V\|_{\textup{tr}}\geq0.
\end{displaymath}
Moreover, if $\|X\|_{\textup{Bi-tr}}=0$, we have $\|X\|_{\textup{tr}}\leq\|X\|_{\textup{Bi-tr}}=0$, i.e., $\|X\|_{\textup{tr}}=0$.
Hence, $X=0$. In short, the bi-trace norm $\|\cdot\|_{\textup{Bi-tr}}$ is a quasi-norm.
\end{proof}
~\\
Before giving a complete proof for Theorem 1, we first present and prove the following lemmas.

\begin{lemma}[Jensen's inequality]\label{lem2}
Assume that the function $g:\mathbb{R}^{+}\rightarrow \mathbb{R}^{+}$ is a continuous concave function on $[0,+\infty)$. Then, for any $t_{i}\geq 0$, $\sum_{i}t_{i}=1$, and any $x_{i}\in \mathbb{R}^{+}$ for $i=1,\ldots,n$, we have
\begin{equation}\label{dtreq1}
g\left(\sum^{n}_{i=1}t_{i}x_{i}\right)\geq \sum^{n}_{i=1}t_{i}g(x_{i}).
\end{equation}
\end{lemma}

\begin{lemma}\label{lem3}
Suppose $Z\in\mathbb{R}^{m\times n}$ is of rank $r\leq\min(m,n)$, and denote its SVD by $Z=L\Sigma_{Z}R^{T}$, where $L\in \mathbb{R}^{m\times r}$, $R\in \mathbb{R}^{n\times r}$ and $\Sigma_{Z}\in \mathbb{R}^{r\times r}$. For any unitary matrix $A$ satisfying $AA^{T}=A^{T}A=I_{r\times r}$, then $(A\Sigma_{Z}A^{T})_{k,k}\geq 0$ for any $k=1,\ldots,r$, and
\begin{equation*}
\textup{Tr}^{1/2}(A\Sigma_{Z}A^{T})\geq \textup{Tr}^{1/2}(\Sigma_{Z})=\|Z\|^{1/2}_{S_{1/2}},
\end{equation*}
where $\textup{Tr}^{1/2}(B)=\sum_{i}B^{1/2}_{ii}$.
\end{lemma}

\begin{proof}
For any $k\in\{1,\ldots,r\}$, we have
\begin{equation*}
(A\Sigma_{Z}A^{T})_{k,k}=\sum_{i}a^{2}_{ki}\sigma_{i}\geq 0,
\end{equation*}
where $\sigma_{i}\geq 0$ is the $i$-th singular value of $Z$. Then
\begin{equation}\label{smeq1}
\textup{Tr}^{1/2}(A\Sigma_{Z}A^{T})=\sum_{k}(\sum_{i}a^{2}_{ki}\sigma_{i})^{1/2}.
\end{equation}
Let $t_{i}\geq 0$, $\sum_{i}t_{i}=1$, and $x_{i}\geq 0$. Since the function $g(x)=x^{1/2}$ is concave on $\mathbb{R}^{+}$, and by Lemma \ref{lem2}, we have
\begin{equation*}
g\left(\sum_{i}t_{i}x_{i}\right)\geq \sum_{i}t_{i}g(x_{i}).
\end{equation*}

Due to the fact that $\sum_{i}a^{2}_{ki}=1$ for any $k\in\{1,\ldots,r\}$, we have
\begin{equation*}
\left(\sum_{i}a^{2}_{ki}\sigma_{i}\right)^{1/2}\geq \sum_{i}a^{2}_{ki}\sigma^{1/2}_{i}.
\end{equation*}
According to the above inequality and the fact that $\sum_{k}a^{2}_{ki}=1$ for any $i\in\{1,\ldots,r\}$, \eqref{smeq1} can be written as follows:
\begin{equation*}
\begin{split}
\textup{Tr}^{1/2}(A\Sigma_{Z}A^{T})&=\sum_{k}(\sum_{i}a^{2}_{ki}\sigma_{i})^{1/2}\\
&\geq \sum_{k}\sum_{i}a^{2}_{ki}\sigma^{1/2}_{i}\\
&=\sum_{i}\sigma^{1/2}_{i}\\
&=\textup{Tr}^{1/2}(\Sigma_{Z})=\|Z\|^{1/2}_{S_{1/2}}.
\end{split}
\end{equation*}
This completes the proof.
\end{proof}

\subsection*{Proof of Theorem 1:}
\begin{proof}
Assume that $U\!=\!L_{U}\Sigma_{U}R^{T}_{U}$ and $V\!=\!L_{V}\Sigma_{V}R^{T}_{V}$ are the thin SVDs of $U$ and $V$, respectively, where $L_{U}\in \mathbb{R}^{m\times d}$, $L_{V}\in \mathbb{R}^{n\times d}$, and $R_{U},\Sigma_{U},R_{V},\Sigma_{V}\in \mathbb{R}^{d\times d}$. Without loss of generality, we set $X\!=\!L_{X}\Sigma_{X}R^{T}_{X}$, where the columns of $L_{X}\in \mathbb{R}^{m\times d}$ and $R_{X}\in \mathbb{R}^{n\times d}$ are the left and right singular vectors associated with the top $d$ singular values of $X$ with rank at most $r$ $(r\!\leq\! d)$, and $\Sigma_{X}=\textup{diag}([\sigma_{1}(X),\!\cdots\!,\sigma_{r}(X),0,\!\cdots\!,0])\in\!\mathbb{R}^{d\times d}$.

By $X=UV^{T}$, which means that $L_{X}\Sigma_{X}R^{T}_{X}=L_{U}\Sigma_{U}R^{T}_{U}R_{V}\Sigma_{V}L^{T}_{V}$, then $\exists O_{1},\widehat{O}_{1}\in \mathbb{R}^{d\times d}$ satisfy $L_{X}=L_{U}O_{1}$ and $L_{U}=L_{X}\widehat{O}_{1}$. Since $O_{1}=L^{T}_{U}L_{X}$ and $\widehat{O}_{1}=L^{T}_{X}L_{U}$, then $O^{T}_{1}=\widehat{O}_{1}$. Indeed, since $L_{X}=L_{U}O_{1}=L_{X}\widehat{O}_{1}O_{1}$, we immediately have $\widehat{O}_{1}O_{1}=O^{T}_{1}O_{1}=I_{d\times d}$. In addition, we obviously have $O_{1}\widehat{O}_{1}=O_{1}O^{T}_{1}=I_{d\times d}$. Similarly, $\exists O_{2}\in \mathbb{R}^{d\times d}$ satisfying $O_{2}O^{T}_{2}=O^{T}_{2}O_{2}=I_{d\times d}$ such that $R_{X}=L_{V}O_{2}$. Therefore, we have
\begin{displaymath}
\begin{split}
\Sigma_{X}O^{T}_{2}=\Sigma_{X}R^{T}_{X}L_{V}=L^{T}_{X}L_{U}\Sigma_{U}R^{T}_{U}R_{V}\Sigma_{V}=O^{T}_{1}\Sigma_{U}R^{T}_{U}R_{V}\Sigma_{V}.
\end{split}
\end{displaymath}

Let $O_{3}=O_{2}O^{T}_{1}\in \mathbb{R}^{d\times d}$, then we have $O_{3}O^{T}_{3}=O^{T}_{3}O_{3}=I_{d\times d}$, i.e., $\sum_{i}(O_{3})^{2}_{ij}=\sum_{j}(O_{3})^{2}_{ij}=1$ for $\forall i,j\in \{1,2,\ldots,d\}$, where $a_{i,j}$ denotes the element of the matrix $A$ in the $i$-th row and the $j$-th column. Furthermore, let $O_{4}=R^{T}_{U}R_{V}$, we have $\sum_{i}(O_{4})^{2}_{ij}\leq 1$ and $\sum_{j}(O_{4})^{2}_{ij}\leq 1$ for $\forall i,j\in \{1,2,\ldots,d\}$.

By the above analysis, then we have $O_{2}\Sigma_{X}O^{T}_{2}=O_{2}O^{T}_{1}\Sigma_{U}O_{4}\Sigma_{V}=O_{3}\Sigma_{U}O_{4}\Sigma_{V}$. Let $\varrho_{i}$ and $\tau_{j}$ denote the $i$-th and the $j$-th diagonal elements of $\Sigma_{V}$ and $\Sigma_{U}$, respectively. By Lemma \ref{lem3}, we can derive that
\begin{equation*}
\begin{split}
\|X\|_{S_{1/2}}&\leq\left(\textup{Tr}^{1/2}(O_{2}\Sigma_{X}O^{T}_{2})\right)^{2}=\left(\textup{Tr}^{1/2}(O_{2}O^{T}_{1}\Sigma_{U}O_{4}\Sigma_{V})\right)^{2}=\left(\textup{Tr}^{1/2}(O_{3}\Sigma_{U}O_{4}\Sigma_{V})\right)^{2}\\
&=\left(
\sum^{d}_{i=1}\sqrt{\sum^{d}_{j=1}\tau_{j}(O_{3})_{ij}(O_{4})_{ji}\varrho_{i}}\right )^{2}=\left(
\sum^{d}_{i=1}\sqrt{\varrho_{i}\sum^{d}_{j=1}\tau_{j}(O_{3})_{ij}(O_{4})_{ji}}\right )^{2}\\
&^{a}\!\!\!\leq\sum^{d}_{i=1}\varrho_{i}\sum^{d}_{i=1}\sum^{d}_{j=1}(\tau_{j}(O_{3})_{ij}(O_{4})_{ji})\\
&^{b}\!\!\!\leq\sum^{d}_{i=1}\varrho_{i}\sum^{d}_{i=1}\sum^{d}_{j=1}\frac{((O_{3})^{2}_{ij}\tau_{j}+(O_{4})^{2}_{ji}\tau_{j})}{2}\\
&^{c}\!\!\!\leq\sum^{d}_{i=1}\varrho_{i}\sum^{d}_{j=1}\tau_{j}\\
&=\|U\|_{\textup{tr}}\|V\|_{\textup{tr}}\leq\left(\frac{\|U\|_{\textup{tr}}+\|V\|_{\textup{tr}}}{2}\right)^{2},
\end{split}
\end{equation*}
where the inequality $^{a}\!\!\leq$ holds due to the Cauchy$-$Schwartz inequality, the inequality $^{b}\!\!\leq$ follows from the basic inequality $xy\leq \frac{x^{2}+y^{2}}{2}$ for any real numbers $x$ and $y$, and the inequality $^{c}\!\!\leq$ relies on the fact that $\sum_{i}(O_{3})^{2}_{ij}=1$ and $\sum_{i}(O_{4})^{2}_{ji}\leq 1$. Thus, we obtain
\begin{equation*}
\|X\|_{S_{1/2}}\leq\|U\|_{\textup{tr}}\|V\|_{\textup{tr}}\leq\left(\frac{\|U\|_{\textup{tr}}+\|V\|_{\textup{tr}}}{2}\right)^{2}\;\textup{and}\;\|X\|_{S_{1/2}}\leq\|U\|_{\textup{tr}}\|V\|_{\textup{tr}}\leq\frac{\|U\|^{2}_{\textup{tr}}+\|V\|^{2}_{\textup{tr}}}{2}.
\end{equation*}

On the other hand, set $U_{\star}=L_{X}\Sigma^{1/2}_{X}$ and $V_{\star}=R_{X}\Sigma^{1/2}_{X}$, then we have $X=U_{\star}V^{T}_{\star}$ and
\begin{equation*}
\begin{split}
\|X\|_{S_{1/2}}&=[\textup{Tr}^{1/2}(\Sigma_{X})]^{2}=\|L_{X}\Sigma^{1/2}_{X}\|_{\textup{tr}}\|R_{X}\Sigma^{1/2}_{X}\|_{\textup{tr}}=\|U_{\star}\|_{\textup{tr}}\|V_{\star}\|_{\textup{tr}}=\frac{\|L_{X}\Sigma^{1/2}_{X}\|^{2}_{\textup{tr}}+\|R_{X}\Sigma^{1/2}_{X}\|^{2}_{\textup{tr}}}{2}\\
&=\frac{\|U_{\star}\|^{2}_{\textup{tr}}+\|V_{\star}\|^{2}_{\textup{tr}}}{2}=\left(\frac{\|L_{X}\Sigma^{1/2}_{X}\|_{\textup{tr}}+\|R_{X}\Sigma^{1/2}_{X}\|_{\textup{tr}}}{2}\right)^{2}=\left(\frac{\|U_{\star}\|_{\textup{tr}}+\|V_{\star}\|_{\textup{tr}}}{2}\right)^{2}. \end{split}
\end{equation*}
Therefore, under the constraint $X=UV^{T}$, we have
\begin{displaymath}
\|X\|_{S_{1/2}}=\min_{X=UV^{T}}\!\left(\frac{\|U\|_{\textup{tr}}+\|V\|_{\textup{tr}}}{2}\right)^{2}=\min_{X=UV^{T}}\!\frac{\|U\|^{2}_{\textup{tr}}+\|V\|^{2}_{\textup{tr}}}{2}=\min_{X=UV^{T}}\!\|U\|_{\textup{tr}}\|V\|_{\textup{tr}}=\|X\|_{\textup{Bi-tr}}.
\end{displaymath}
This completes the proof.
\end{proof}

\section{Proof of Theorem 3}
Before giving the proof of Theorem 3, we first prove the boundedness of multipliers and some variables of Algorithm 1. To prove the boundedness, we first give the following lemmas.

\begin{lemma}[\cite{recht:nnm}] \label{Conv L1}
Let $\mathcal{H}$ be a real Hilbert space endowed with an inner product $\langle\cdot,\cdot\rangle$ and a corresponding norm $\|\!\cdot\!\|$, and $y\in \partial\|x\|$, where $\partial f(x)$ denotes the subgradient of $f(x)$. Then $\|y\|^{*}=1$ if $x\neq 0$, and $\|y\|^{*}\leq 1$ if $x=0$, where $\|\!\cdot\!\|^{*}$ is the dual norm of $\|\!\cdot\!\|$. For instance, the dual norm of the trace norm is the spectral norm, $\|\!\cdot\!\|_{2}$, i.e., the largest singular value.
\end{lemma}

\begin{lemma}[\cite{bertsekas:np}] \label{Conv L2}
Assume that $\nabla\! g$ is Lipschitz continuous on \textup{dom}$(g):=\{X|g(X)<\infty\}$ satisfying the following condition: $\|\nabla\! g(X)-\nabla\! g(Y)\|_{F}\leq L_{g}\|X-Y\|_{F},\;\forall X, Y\in\textup{dom}(g)$, with a Lipschitz constant $L_{g}$. Then
\begin{displaymath}
g(X)\leq g(Y)+\langle\nabla\! g(Y),\,X-Y\rangle+\frac{L_{g}}{2}\|X-Y\|^{2}_{F},\;\;\forall X, Y\in\textup{dom}(g).
\end{displaymath}
\end{lemma}

\begin{lemma}
Let $\lambda_{k+1}=\lambda_{k}+\beta_{k}(\mathcal{A}(U_{k+1}V^{T}_{k+1})-b-e_{k+1})$, then the sequences $\{U_{k},V_{k}\}$, $\{e_{k}\}$ and $\{\lambda_{k}\}$ produced by Algorithm 1 are all bounded.
\end{lemma}

\begin{proof}
By the first-order optimality condition of the Lagrangian function $\mathcal{L}(U,V,e,\lambda,\beta)$ with respect to $e$, we have
\begin{displaymath}
0\in \partial_{e}\mathcal{L}(U_{k+1},V_{k+1},e,\lambda_{k},\beta_{k}),
\end{displaymath}
which equivalently states that
\begin{displaymath}
\beta_{k}\!\left(\mathcal{A}(U_{k+1}V^{T}_{k+1})-e_{k+1}-b\right)+\lambda_{k}\in \frac{1}{\mu}\partial\|e_{k+1}\|_{1}.
\end{displaymath}

Recalling $\lambda_{k+1}=\lambda_{k}+\beta_{k}(\mathcal{A}(U_{k+1}V^{T}_{k+1})-b-e_{k+1})$, we have $\lambda_{k+1}\in \frac{1}{\mu}\partial\|e_{k+1}\|_{1}$. By Lemma~\ref{Conv L1}, we have
\begin{displaymath}
\|\lambda_{k+1}\|_{\infty}\leq \frac{1}{\mu},
\end{displaymath}
where $\|\!\cdot\!\|_{\infty}$ is the dual norm of $\|\!\cdot\!\|_{1}$. Thus, the sequence $\{\lambda_{k}\}$ is bounded.

By the definition of the linearization function $\widehat{\varphi}_{k}(U,U_{k})$ in (8), we have $\widehat{\varphi}_{k}(U_{k},U_{k})=\varphi_{k}(U_{k})$. Since $U_{k+1}$ is the optimal solution of (9), and by Lemma \ref{Conv L2}, then we can derive that
\begin{equation*}
\begin{split}
&\mathcal{L}(U_{k+1},V_{k},e_{k},\lambda_{k},\beta_{k})\\
=&\frac{1}{2}\|U_{k+1}\|_{\textup{tr}}+\frac{\beta_{k}}{2}\varphi_{k}(U_{k+1})+c\\
\leq&\frac{1}{2}\|U_{k+1}\|_{\textup{tr}}+\frac{\beta_{k}}{2}\widehat{\varphi}_{k}(U_{k+1},U_{k})+c\\ \leq&\frac{1}{2}\|U_{k}\|_{\textup{tr}}+\frac{\beta_{k}}{2}\varphi_{k}(U_{k})+c=\mathcal{L}(U_{k},V_{k},e_{k},\lambda_{k},\beta_{k}),
\end{split}
\end{equation*}
where $c$ is a constant independent of both $U_{k}$ and $U_{k+1}$. Similarly, we have
\begin{equation*}
\mathcal{L}(U_{k+1},V_{k+1},e_{k},\lambda_{k},\beta_{k})\leq\mathcal{L}(U_{k+1},V_{k},e_{k},\lambda_{k},\beta_{k}).
\end{equation*}
Furthermore, by the iteration procedure of Algorithm 1, we obtain
\begin{equation*}
\begin{split}
&\mathcal{L}(U_{k+1},V_{k+1},e_{k+1},\lambda_{k},\beta_{k})\\
\leq&\mathcal{L}(U_{k+1},V_{k+1},e_{k},\lambda_{k},\beta_{k})\leq\mathcal{L}(U_{k},V_{k},e_{k},\lambda_{k},\beta_{k})\\
=&\mathcal{L}(U_{k},V_{k},e_{k},\lambda_{k-1},\beta_{k-1})+\alpha_{k}\|\lambda_{k}-\lambda_{k-1}\|^2_{2},
\end{split}
\end{equation*}
where $\alpha_{k}=\frac{\beta_{k}+\beta_{k-1}}{2\beta^{2}_{k-\!1}}$.

Since
\begin{displaymath}
\sum^\infty_{k=1}\alpha_{k}=\frac{\rho(\rho+1)}{2\beta_{0}(\rho-1)}<+\infty,
\end{displaymath}
and recall the boundedness of $\{\lambda_{k}\}$, we have that $\{\mathcal{L}(U_{k},V_{k},e_{k},\lambda_{k-1},\beta_{k-1})\}$ is upper-bounded.

Note that $\lambda_{k}=\lambda_{k-1}+\beta_{k-1}(\mathcal{A}(U_{k}V^{T}_{k})-b-e_{k})$. Then we have
\begin{displaymath}
\frac{1}{2}(\|U_{k}\|_{\textup{tr}}+\|V_{k}\|_{\textup{tr}})+\frac{1}{\mu}\|e_{k}\|_{1}=\mathcal{L}(U_{k},V_{k},e_{k},\lambda_{k-1},\beta_{k-1})-\frac{\|\lambda_{k}\|^2_{2}-\|\lambda_{k-1}\|^2_{2}}{2\beta_{k-1}},
\end{displaymath}
which is also upper-bounded. Thus the sequences $\{e_{k}\}$, $\{U_{k}\}$ and $\{V_{k}\}$ are all bounded.
\end{proof}

\subsection*{Proof of Theorem 3:}
\begin{proof}
$\textbf{(I)}$ By $\mathcal{A}(U_{k+1}V^{T}_{k+1})-e_{k+1}-b=(\lambda_{k+1}-\lambda_{k})/\beta_{k}$, the boundedness of $\{\lambda_{k}\}$, and $\lim_{k\rightarrow\infty}\beta_{k}=\infty$, we have
\begin{equation*}
\lim_{k\rightarrow\infty}\|\mathcal{A}(U_{k+1}V^{T}_{k+1})-e_{k+1}-b\|_{2}=0.
\end{equation*}
Hence, $\{(U_{k},\,V_{k},\,e_{k})\}$ approaches to a feasible solution.

In the following, we will prove that the sequences $\{U_{k}\}$, $\{V_{k}\}$ and $\{e_{k}\}$ are Cauchy sequences.

By the boundedness of $\{\lambda_{k}\}$, $\{e_{k}\}$, $\{U_{k}\}$ and $\{V_{k}\}$, then both $\nabla\! \varphi_{k}(U_{k})$ and $t^{\varphi}_{k}$ are bounded. Furthermore, $\exists P_{k+1}\in\partial\|U_{k+1}\|_{\textup{tr}}$ satisfies the following first-order optimality condition of (9)
\begin{equation}\label{Th31}
\frac{1}{2}P_{k+1}+\beta_{k}t^{\varphi}_{k}\left[U_{k+1}-U_{k}+\frac{1}{t^{\varphi}_{k}}\nabla\! \varphi_{k}(U_{k})\right]=0.
\end{equation}
By Lemma \ref{Conv L1}, we have $\|P_{k+1}\|_{2}\leq 1$, which implies that $\{P_{k+1}\}$ is bounded.
\begin{equation}\label{Th312}
\nabla\! \varphi_{k}(U_{k})=\mathcal{A}^{*}[\mathcal{A}(U_{k}V^{T}_{k})-e_{k}-b+\lambda_{k}/\beta_{k}]V_{k}=\frac{\mathcal{A}^{*}((\rho+1)\lambda_{k}-\rho\lambda_{k-1})V_{k}}{\beta_{k}}.
\end{equation}

Substituting \eqref{Th312} into \eqref{Th31}, it is easy to see that
\begin{equation*}
\|U_{k+1}-U_{k}\|_{F}=\frac{\|\frac{1}{2}P_{k+1}+\beta_{k}\nabla\! \varphi_{k}(U_{k})\|_{F}}{\beta_{k}t^{\varphi}_{k}}=\frac{\|P_{k+1}+2\mathcal{A}^{*}((\rho+1)\lambda_{k}-\rho\lambda_{k-1})V_{k}\|_{F}}{2\beta_{k}t^{\varphi}_{k}}.
\end{equation*}
Consequently, if $m>n$,
\begin{equation*}
\begin{split}
\|U_{n}\!-\!U_{m}\|_{F}\!\!&\leq\!\|U_{n}-U_{n+1}\|_{F}+\|U_{n+1}-U_{n+2}\|_{F}+\ldots+\|U_{m-1}-U_{m}\|_{F}\\
&\!=\!\!\frac{\|\!P\!_{n\!+\!1}\!\!+\!\!2\!\mathcal{A}^{*}\!(\!(\!\rho\!+\!1\!)\!\lambda_{n}\!\!-\!\!\rho\!\lambda_{n\!-\!1}\!)\!V_{n}\!\|_{F}\!\!}{2\beta_{n}t^{\varphi}_{n}}\!\!+\!\!\frac{\|
\!P\!_{n\!+\!2}\!\!+\!\!2\!\mathcal{A}^{*}\!(\!(\!\rho\!+\!1\!)\!\lambda_{n\!+\!1}\!\!-\!\!\rho\!\lambda_{n}\!)\!V_{n\!+\!1}\!\|_{F}\!\!}{2\beta_{n+\!1}t^{\varphi}_{n\!+\!1}}\!\!+\!\!\ldots\!\!+\!\!
\frac{\|\!P\!_{m}\!\!+\!\!2\!\mathcal{A}^{*}\!(\!(\!\rho\!+\!1\!)\!\lambda_{m\!-\!1}\!\!-\!\!\rho\!\lambda_{m\!-\!2}\!)\!V_{m\!-\!1}\!\|_{F}\!\!}{2\beta_{m-\!1}t^{\varphi}_{m-\!1}}\\
&\!\leq\!\delta_{C}(\frac{1}{\beta_{n}}+\frac{1}{\beta_{n+1}}+\ldots+\frac{1}{\beta_{m-1}})=\frac{\delta_{C}}{\beta_{n}}(1+\frac{1}{\rho}+\ldots+\frac{1}{\rho^{m-n-1}})<\frac{\rho\delta_{C}}{(\rho-1)\beta_{n}},
\end{split}
\end{equation*}
where $\delta_{C}\!=\!\max\{\frac{\|\!P_{n\!+\!1}+2\!\mathcal{A}^{*}\!((\!\rho\!+\!1\!)\lambda_{n}\!-\!\rho\lambda_{n\!-\!1}\!)\!V_{n}\!\|_{F}}{2t^{\varphi}_{n}},\frac{\|
\!P_{n\!+\!2}+2\!\mathcal{A}^{*}\!((\!\rho\!+\!1\!)\lambda_{n\!+\!1}\!-\!\rho\lambda_{n}\!)\!V_{n\!+\!1}\!\|_{F}}{2t^{\varphi}_{n\!+\!1}},\ldots,\frac{\|\!P_{m}+2\!\mathcal{A}^{*}\!((\!\rho\!+\!1\!)\lambda_{m\!-\!1}\!-\!\rho\lambda_{m\!-\!2}\!)\!V_{m\!-\!1}\!\|_{F}}{2t^{\varphi}_{m-\!1}}\}$.
Since $\frac{\rho\delta_{C}}{(\rho-1)\beta_{n}}\!\rightarrow\! 0$, it follows that indeed $\{U_{k}\}$ is a Cauchy sequence.

Similarly, $\{V_{k}\}$ and $\{e_{k}\}$ are also Cauchy sequences.

$\textbf{(II)}$ Let $(U_{*},V_{*},e_{*})$ be a stationary point of (6), then the Karush-Kuhn-Tucker (KKT) conditions for (6) are formulated as follows:
\begin{equation*}
\begin{split}
0&\in \partial \|U_{*}\|_{\textup{tr}}+2\mathcal{A}^{*}(\lambda_{*})V_{*},\\
0&\in \partial \|V_{*}\|_{\textup{tr}}+2(\mathcal{A}^{*}(\lambda_{*}))^{T}U_{*},\\
0&\in \frac{1}{\mu}\partial \|e_{*}\|_{1}-\lambda_{*},\\
e_{*}&\!=\mathcal{A}(U_{*}V^{T}_{*})-b,
\end{split}
\end{equation*}
where $\lambda_{*}$ is the associated Lagrangian multiplier. The first-order optimality condition of each subproblem at the $(k\!+\!1)$-th iteration is given by
\begin{equation}\label{Th32}
\begin{split}
&0\in\partial\|U_{k+1}\|_{\textup{tr}}+2\beta_{k}t^{\varphi}_{k}\!\left[U_{k+1}-U_{k}+\frac{1}{t^{\varphi}_{k}}\nabla\! \varphi_{k}(U_{k})\right],\\
&0\in\partial\|V_{k+1}\|_{\textup{tr}}+2\beta_{k}t^{\psi}_{k}\!\left[V_{k+1}-V_{k}+\frac{1}{t^{\psi}_{k}}\nabla\! \psi_{k}(V_{k})\right],\\
&0\in \frac{1}{\mu}\partial \|e_{k+1}\|_{1}-\beta_{k}\!\left[\mathcal{A}(U_{k+1}V^{T}_{k+1})-e_{k+1}-b+\lambda_{k}/\beta_{k}\right].
\end{split}
\end{equation}

Since $\{U_{k}\}$, $\{V_{k}\}$ and $\{e_{k}\}$ are Cauchy sequences, then $\|U_{k+1}-U_{k}\|_{F}\rightarrow 0$, $\|V_{k+1}-V_{k}\|_{F}\rightarrow 0$ and $\|e_{k+1}-e_{k}\|_{2}\rightarrow 0$. Let $U_{\infty}$, $V_{\infty}$ and $e_{\infty}$ be their limit points, respectively, and $\lambda_{\infty}$ be the associated Lagrangian multiplier. By the assumption $\|\lambda_{k+1}-\lambda_{k}\|_{2}\rightarrow 0$, and $(U_{k},V_{k},e_{k})$ approaches a feasible solution, then $\beta_{k}\nabla\! \varphi_{k}(U_{k})=\beta_{k}\mathcal{A}^{*}[\mathcal{A}(U_{k}V^{T}_{k})-b-e_{k}+\lambda_{k}/\beta_{k}]V_{k}\rightarrow \mathcal{A}^{*}(\lambda_{\infty})V_{\infty}$ and $\beta_{k}\nabla\! \psi_{k}(V_{k})=\beta_{k}\{\mathcal{A}^{*}[\mathcal{A}(U_{k+1}V^{T}_{k})-b-e_{k}+\lambda_{k}/\beta_{k}]\}^{T}U_{k+1}\rightarrow [\mathcal{A}^{*}(\lambda_{\infty})]^{T}U_{\infty}$. Therefore, with $k\rightarrow \infty$, the following holds
\begin{equation*}
\begin{split}
0&\in \partial\|U_{\infty}\|_{\textup{tr}}+2\mathcal{A}^{*}(\lambda_{\infty})V_{\infty},\\
0&\in \partial\|V_{\infty}\|_{\textup{tr}}+2[\mathcal{A}^{*}(\lambda_{\infty})]^{T}U_{\infty},\\
0&\in \frac{1}{\mu}\partial\|e_{\infty}\|_{1}-\lambda_{\infty},\\
e_{\infty}&\!=\mathcal{A}(U_{\infty}V^{T}_{\infty})-b.
\end{split}
\end{equation*}
Hence, the accumulation point $(U_{\infty},U_{\infty},e_{\infty})$ of the sequence $\{(U_{k},U_{k},e_{k})\}$ generated by Algorithm 1 satisfies the KKT conditions for the problem (6).
\end{proof}

\section{Proof of Theorem 4}
To solve the bi-trace quasi-norm regularized problem (4) with the squared loss $\|\!\cdot\!\|^{2}_{2}$, the proposed algorithm is based on the proximal alternating linearized minimization (PALM) method for solving the following non-convex problem:
\begin{equation}\label{smeq2}
\min_{x,y}\,Q(x,y):=F(x)+G(y)+H(x,y),
\end{equation}
where $F(x)$ and $G(y)$ are proper lower semi-continuous functions, and $H(x,y)$ is a smooth function with Lipschitz continuous gradients on any bounded set.

In Section 4.2 of the main paper, we stated that our PALM algorithm alternates between two blocks of variables, $U$ and $V$. We establish the global convergence of our PALM algorithm by transforming the problem (4) into a standard form \eqref{smeq2}, and show that the transformed problem satisfies the condition needed to establish the convergence. First, the minimization problem (4) can be expressed in the form of \eqref{smeq2} by setting
\begin{equation*}
\begin{cases}
F(U):=\frac{1}{2}\|U\|_{\textup{tr}};\\
G(V):=\frac{1}{2}\|V\|_{\textup{tr}};\\
H(U,V):=\frac{1}{2\mu}\|\mathcal{A}(UV^{T})-b\|^{2}_{2}.
\end{cases}
\end{equation*}

The conditions for global convergence of the PALM algorithm proposed in~\cite{bolte:palm} are shown in the following lemma.
\begin{lemma}\label{lem7}
Let $\{(x_{k},y_{k})\}$ be a sequence generated by the PALM algorithm proposed in~\cite{bolte:palm}. This sequence converges to a critical point of \eqref{smeq2}, if the following conditions hold:
\begin{enumerate}
  \item $Q(x,y)$ is a Kurdyka-{\L}ojasiewicz (KL) function;
  \item $\nabla H(x,y)$ has Lipschitz constant on any bounded set;
  \item $\{(x_{k},y_{k})\}$ is a bounded sequence.
\end{enumerate}
\end{lemma}

As stated in Lemma \ref{lem7}, the first condition requires that the objective function satisfies the KL property (For more details, see~\cite{bolte:palm}). It is known that any proper closed semi-algebraic function is a KL function as such a function satisfies the KL property for all points in $\textup{dom}f$ with $\phi(s)=cs^{1-\theta}$ for some $\theta\in[0,1)$ and some $c>0$. Therefore, we first give the following definitions of semi-algebraic sets and functions, and then prove that the proposed problem (4) with the squared loss $\|\!\cdot\!\|^{2}_{2}$ is also semi-algebraic.

\begin{definition} [\cite{bolte:palm}]
A subset $S\subset\mathbb{R}^{n}$ is a real semi-algebraic set if there exists a finite number of real polynomial functions $g_{ij},\,h_{ij}:\mathbb{R}^{n}\rightarrow \mathbb{R}$ such that
\begin{equation*}
S=\bigcup_{j}\bigcap_{i} \left\{u\in \mathbb{R}^{n}:g_{ij}(u)=0,\;h_{ij}(u)<0\right\}.
\end{equation*}
Moreover, a function $g(u)$ is called semi-algebraic if its graph $\{(u,t)\in \mathbb{R}^{n+1}: g(u)=t\}$ is a semi-algebraic set.
\end{definition}

Semi-algebraic sets are stable under the operations of finite union, finite intersections, complementation and Cartesian product. The following are the semi-algebraic functions or the property of semi-algebraic functions used below:
\begin{itemize}
\item Real polynomial functions.
\item Finite sums and product of semi-algebraic functions.
\item Composition of semi-algebraic functions.
\end{itemize}

\begin{lemma}
Each term in the proposed problem (4) with the squared loss $\|\!\cdot\!\|^{2}_{2}$ is a semi-algebraic function, and thus the function (4) is also semi-algebraic.
\end{lemma}

\begin{proof}
It is easy to notice that the sets $\mathcal{U}=\{U\in \mathbb{R}^{m\times d}:\|U\|_{\infty}\leq D_{1}\}$ and $\mathcal{V}=\{V\in \mathbb{R}^{n\times d}:\|V\|_{\infty}\leq D_{2}\}$ are both semi-algebraic sets, where $D_{1}$ and $D_{2}$ denote two pre-defined upper-bounds for all entries of $U$ and $V$, respectively.

For both terms $F(U)=\frac{1}{2}\|U\|_{\textup{tr}}$ and $G(V)=\frac{1}{2}\|V\|_{\textup{tr}}$: According to~\cite{bolte:palm}, we can know that the $\ell_{1}$-norm is a semi-algebraic function. Since the trace norm is equivalent to the $\ell_{1}$-norm on singular values of the associated matrix, it is natural that the trace norm is also semi-algebraic.

For the third term $H(U,V)=\frac{1}{2\mu}\|\mathcal{A}(UV^{T})-b\|^{2}_{2}$, it is a real polynomial function, and thus is a semi-algebraic function~\cite{bolte:palm}. Therefore, the proposed problem (4) with the squared loss $\|\!\cdot\!\|^{2}_{2}$ is semi-algebraic due to the fact that a finite sum of semi-algebraic functions is also semi-algebraic.
\end{proof}

For the second condition in Lemma \ref{lem7}, $H(U,V)=\frac{1}{2\mu}\|\mathcal{A}(UV^{T})-b\|^{2}_{2}$ is a smooth polynomial function, and $\nabla H(U,V)=(\frac{1}{\mu}\mathcal{A}^{*}[\mathcal{A}(UV^{T})-b]V,\,\frac{1}{\mu}\{\mathcal{A}^{*}[\mathcal{A}(UV^{T})-b]\}^{T}U)$. It is natural that $\nabla H(U,V)$ has Lipschitz constant on any bounded set~\cite{toh:apg}.

For the final condition in Lemma \ref{lem7}, $U_{k}\in\mathcal{U}$ and $V_{k}\in\mathcal{V}$ for any $k=1,2,\ldots$, which implies the sequence $\{(U_{k},V_{k})\}$ is bounded.

In short, we can know that three similar conditions as in Lemma \ref{lem7} hold for our PALM algorithm. In other words, our PALM algorithm shares the same convergence property as in Lemma \ref{lem7}.~\\

\section{Proof of Theorem 6}
In order to prove Theorem 6, we first introduce the following Lemma (i.e., the Lemma 11 in~\cite{oymak:lrm} or the Theorem 1 in~\cite{yue:lrmr}).

\begin{lemma}\label{AppC1}
Let $A, B\in \mathbb{R}^{n_{1}\times n_{2}}$, for any $p\in (0,1]$, then we have
\begin{equation*}
\sum^{n}_{i=1}|\sigma^{p}_{i}(A)-\sigma^{p}_{i}(B)|\leq\sum^{n}_{i=1}\sigma^{p}_{i}(A-B),
\end{equation*}
where $n=\min(n_{1},n_{2})$.
\end{lemma}

\subsection*{Proof of Theorem 6:}
\begin{proof}
($\Longrightarrow$) By the definitions of $U_{\star}$ and $V_{\star}$ and Theorem 1 of the main paper, we have
\begin{equation}\label{Th41}
X_{0}=U_{\star}V^{T}_{\star},\;\;\|X_{0}\|^{1/2}_{S_{1/2}}=\frac{\|U_{\star}\|_{\textup{tr}}+\|V_{\star}\|_{\textup{tr}}}{2},
\end{equation}
and $\textup{rank}(U_{\star})=\textup{rank}(X_{0})\leq r$ and $\textup{rank}(V_{\star})=\textup{rank}(X_{0})\leq r$.
$\forall Z=U_{\star}W^{T}_{2}+W_{1}V^{T}_{\star}+W_{1}W^{T}_{2}$, we then obtain
\begin{equation}\label{AppNSP1}
\begin{split}
X_{0}+Z&=U_{\star}V^{T}_{\star}+U_{\star}W^{T}_{2}+W_{1}V^{T}_{\star}+W_{1}W^{T}_{2}\\
&=(U_{\star}+W_{1})(V_{\star}+W_{2})^{T}.
\end{split}
\end{equation}

Recall that $\mathcal{A}(U_{\star}V^{T}_{\star})=\mathcal{A}(X_{0})=b$. Then for all $Z\in \mathcal{N}(\mathcal{A})\setminus \{\mathbf{0}\}$, we get
\begin{equation*}
\mathcal{A}\left((U_{\star}+W_{1})(V_{\star}+W_{2})^{T}\right)=\mathcal{A}(X_{0}+Z)=b.
\end{equation*}

Thus all feasible solutions to (3) can be represented as $X_{0}+Z$ with $Z\in \mathcal{N}(\mathcal{A})$. To prove that $X_{0}=U_{\star}V^{T}_{\star}$ is uniquely recovered by (3), we need to show that any feasible solution $X_{1}=U_{1}V^{T}_{1}$ ($X_{1}\neq X_{0}$) to (3) satisfies $\|X_{1}\|^{1/2}_{S_{1/2}}>\|X_{0}\|^{1/2}_{S_{1/2}}$, where $U_{1}\in\mathbb{R}^{m\times d}$ and $V_{1}\in\mathbb{R}^{n\times d}$. Let $Z_{1}=X_{1}-X_{0}$, then $Z_{1}\in \mathcal{N}(\mathcal{A})\setminus \{\mathbf{0}\}$. Applying Theorem 1 of the main paper and by \eqref{AppNSP1}, the following holds for some $W_{1}$ and $W_{2}$:
\begin{equation}\label{Th42}
\|X_{1}\|^{1/2}_{S_{1/2}}=\|X_{0}+Z_{1}\|^{1/2}_{S_{1/2}}=\left(\|U_{\star}+W_{1}\|_{\textup{tr}}+\|V_{\star}+W_{2}\|_{\textup{tr}}\right)\!/2,
\end{equation}
where $U_{\star}+W_{1}=U_{X_{1}}\!\Sigma^{1/2}_{X_{1}}$, $V_{\star}+W_{2}=V_{X_{1}}\!\Sigma^{1/2}_{X_{1}}$, and $U_{X_{1}}\!\Sigma_{X_{1}}\!V^{T}_{X_{1}}$ is the same SVD form of $X_{1}$ as that of $X_{0}$.

According to \eqref{Th41}, \eqref{Th42}, Theorem 1 of the main paper, and Lemma \ref{AppC1} with $p=1$, we have
\begin{equation*}
\begin{split}
\|X_{1}\|^{1/2}_{S_{1/2}}=&\|X_{0}+Z_{1}\|^{1/2}_{S_{1/2}}=\frac{1}{2}\left(\|U_{\star}+W_{1}\|_{\textup{tr}}+\|V_{\star}+W_{2}\|_{\textup{tr}}\right)\\
=&\frac{1}{2}\left(\sum^{d}_{i=1}\sigma_{i}(U_{\star}+W_{1})+\sum^{d}_{i=1}\sigma_{i}(V_{\star}+W_{2})\right)\\
\geq&\frac{1}{2}\left(\sum^{d}_{i=1}|\sigma_{i}(U_{\star})-\sigma_{i}(W_{1})|+\sum^{d}_{i=1}|\sigma_{i}(V_{\star})-\sigma_{i}(W_{2})|\right)\\
\geq&\frac{1}{2}\left[\sum^{r}_{i=1}\sigma_{i}(U_{\star})-\sum^{r}_{i=1}\sigma_{i}(W_{1})+\sum^{d}_{i=r+1}\sigma_{i}(W_{1})+\sum^{r}_{i=1}\sigma_{i}(V_{\star})-\sum^{r}_{i=1}\sigma_{i}(W_{2})+\sum^{d}_{i=r+1}\sigma_{i}(W_{2})\right]\\
>&\frac{1}{2}\left(\sum^{r}_{i=1}\sigma_{i}(U_{\star})+\sum^{r}_{i=1}\sigma_{i}(V_{\star})\right)\\
=&\frac{1}{2}\left(\|U_{\star}\|_{\textup{tr}}+\|V_{\star}\|_{\textup{tr}}\right)=\|X_{0}\|^{1/2}_{S_{1/2}},
\end{split}
\end{equation*}
which confirms that $X_{0}=U_{\star}V^{T}_{\star}$ is uniquely recovered by (3).

($\Longleftarrow$) Conversely if (14) does not hold for some $W_{1}$ and $W_{2}$, i.e.,
\begin{equation} \label{Th43}
\sum^{r}_{i=1}(\sigma_{i}(W_{1})+\sigma_{i}(W_{2}))\geq\sum^{d}_{i=r+1}(\sigma_{i}(W_{1})+\sigma_{i}(W_{2})),
\end{equation}
then we can find $W_{1}$ and $W_{2}$ such that $(W_{1})_{r}=-U_{\star}$ and $(W_{2})_{r}=-V_{\star}$, where $(W_{1})_{r}$ and $(W_{2})_{r}$ denote the matrices induced by setting all but largest $r$ singular values of $W_{1}$ and $W_{2}$ to 0, respectively. Using Theorem 1 of the main paper, then we can derive that
\begin{equation*}
\begin{split}
&\|X_{0}+Z\|^{1/2}_{S_{1/2}}\\
\leq &\frac{1}{2}\left(\|U_{\star}+W_{1}\|_{\textup{tr}}+\|V_{\star}+W_{2}\|_{\textup{tr}}\right)\\
=&\frac{1}{2}\left(\sum^{d}_{i=r+1}\sigma_{i}(W_{1})+\sum^{d}_{i=r+1}\sigma_{i}(W_{2})\right)\\
\leq &\frac{1}{2}\left(\sum^{r}_{i=1}\sigma_{i}(W_{1})+\sum^{r}_{i=1}\sigma_{i}(W_{2})\right)\\
=&\frac{1}{2}\left(\sum^{r}_{i=1}\sigma_{i}(U_{\star})+\sum^{r}_{i=1}\sigma_{i}(V_{\star})\right)\\
=&\|X_{0}\|^{1/2}_{S_{1/2}},
\end{split}
\end{equation*}
i.e., $\|X_{0}+Z\|^{1/2}_{S_{1/2}}\leq \|X_{0}\|^{1/2}_{S_{1/2}}$. This shows that $X_{0}=U_{\star}V^{T}_{\star}$ is not the unique minimizer.
\end{proof}

\section{Proof of Theorem 7}
\begin{proof}
With the squared loss, $\|\!\cdot\!\|^{2}_{2}$, (4) can be reformulated as follows:
\begin{equation}\label{Th71}
\min_{U,V}\left\{\frac{\|U\|_{\textup{tr}}\!+\!\|V\|_{\textup{tr}}}{2}+ \frac{1}{2\mu}\|\mathcal{A}(UV^{T})-b\|^{2}_{2}\right\}.
\end{equation}

Given $\widehat{V}$, the first-order optimality condition for the problem \eqref{Th71} with respect to $U$ is given by
\begin{equation}\label{Th72}
\frac{1}{\mu}\mathcal{A}^{*}(b-\mathcal{A}(\widehat{U}\widehat{V}^{T}))\widehat{V}\in \frac{1}{2}\partial\|\widehat{U}\|_{\textup{tr}}.
\end{equation}

According to Lemma \ref{Conv L1} and \eqref{Th72}, we can know that
\begin{displaymath}
\frac{1}{\mu}\|\mathcal{A}^{*}(b-\mathcal{A}(\widehat{U}\widehat{V}^{T}))\widehat{V}\|_{2}\leq \frac{1}{2},
\end{displaymath}
where $\|A\|_{2}$ is the spectral norm of a matrix $A$ and the dual norm of the trace norm. Recall that  $\mathcal{A}^{*}(b-\mathcal{A}(\widehat{U}\widehat{V}^{T}))\widehat{V}\in\mathbb{R}^{m\times d}$ and $\textrm{rank}(\mathcal{A}^{*}(b-\mathcal{A}(\widehat{U}\widehat{V}^{T}))\widehat{V})\leq d$, then we obtain
\begin{equation}\label{Th73}
\|\mathcal{A}^{*}(b-\mathcal{A}(\widehat{U}\widehat{V}^{T}))\widehat{V}\|_{F}\leq \sqrt{d}\|\mathcal{A}^{*}(b-\mathcal{A}(\widehat{U}\widehat{V}^{T}))\widehat{V}\|_{2} \leq \frac{\mu\sqrt{d}}{2}.
\end{equation}

Let $\hat{X}=\widehat{U}\widehat{V}^{T}$. By the RSC assumption and \eqref{Th73}, we have
\begin{displaymath}
\begin{split}
\frac{\|X_{0}-\hat{X}\|_{F}}{\sqrt{m n}}&\leq\frac{\|\mathcal{A}(X_{0}-\hat{X})\|_{2}}{\kappa(\mathcal{A})\sqrt{lmn}}\\
&\leq \frac{\|\mathcal{A}(X_{0})-b\|_{2}}{\kappa(\mathcal{A})\sqrt{lmn}}+\frac{\|b-\mathcal{A}(\hat{X})\|_{2}}{\kappa(\mathcal{A})\sqrt{lmn}}\\
&=\frac{\|e\|_{2}}{\kappa(\mathcal{A})\sqrt{lmn}}+\frac{\|\mathcal{A}^{*}(b-\mathcal{A}(\hat{X}))\hat{V}\|_{F}}{C_{1}\kappa(\mathcal{A})\sqrt{lmn}}\\
&\leq\frac{\epsilon}{\kappa(\mathcal{A})\sqrt{lmn}}+\frac{\mu\sqrt{d}}{2C_{1}\kappa(\mathcal{A})\sqrt{lmn}}.
\end{split}
\end{displaymath}
\end{proof}

\subsection*{Lower bound on $C_{1}$}
Next we discuss the lower boundedness of $C_{1}$, that is, it is lower bounded by a positive constant. By the characterization of the subdifferentials of the trace norm, we can know that
\begin{equation}\label{Th74}
\partial\|X\|_{\textup{tr}}=\left\{Y\,|\,\langle Y,\,X\rangle=\|X\|_{\textup{tr}},\,\|Y\|_{2}\leq 1\right\}.
\end{equation}
Let $\Xi=\mathcal{A}^{*}(b-\mathcal{A}(\widehat{U}\widehat{V}^{T}))\widehat{V}$, and by \eqref{Th72}, we have that $\Xi\in \frac{\mu}{2}\partial\|\widehat{U}\|_{\textup{tr}}$. By \eqref{Th74}, we obtain
\begin{displaymath}
\left\langle \frac{2}{\mu}\Xi, \,\widehat{U}\right\rangle=\|\widehat{U}\|_{\textup{tr}}.
\end{displaymath}
Note that $\|X\|_{\textup{tr}}\geq \|X\|_{F}$ and $\langle X,Y\rangle\leq\|X\|_{F}\|Y\|_{F}$ for any same-sized matrices $X$ and $Y$. Then
\begin{displaymath}
\frac{2}{\mu}\|\Xi\|_{F}\|\widehat{U}\|_{F}\geq\left\langle \frac{2}{\mu}\Xi, \,\widehat{U}\right\rangle=\|\widehat{U}\|_{\textup{tr}}\geq \|\widehat{U}\|_{F}.
\end{displaymath}
Recall that $\|\widehat{U}\|_{F}>0$ and $\mu\neq0$, thus we obtain
\begin{displaymath}
\|\mathcal{A}^{*}(b-\mathcal{A}(\widehat{U}\widehat{V}^{T}))\widehat{V}\|_{F}=\|\Xi\|_{F}\geq \frac{\mu}{2}.
\end{displaymath}

$\widehat{U}$ is the optimal solution of the problem \eqref{Th71} with given $\widehat{V}$, then
\begin{equation*}
\frac{1}{2\mu}\|\mathcal{A}(\widehat{U}\widehat{V}^{T})-b\|^{2}_{2}<\frac{1}{2\mu}\|\mathcal{A}(\widehat{U}\widehat{V}^{T})-b\|^{2}_{2}+\frac{1}{2}\|\widehat{U}\|_{\textup{tr}}\leq\frac{1}{2\mu}\|b\|^{2}_{2}=\nu,
\end{equation*}
where $\nu>0$ is a constant. Hence,
\begin{displaymath}
C_{1}=\frac{\|\mathcal{A}^{*}(b-\mathcal{A}(\hat{X}))\hat{V}\|_{F}}{\|b-\mathcal{A}(\hat{X})\|_{2}}>\frac{\sqrt{\mu}}{2\sqrt{2\nu}}.
\end{displaymath}

\section{Proof of Theorem 8}
According to Theorem 4 of the main paper, we can know that $(\widehat{U},\widehat{V})$ is a critical point of the problem (15). To prove Theorem 8, we first give the following lemma~\cite{wang:mfcf}.

\begin{lemma}\label{lem10}
Let $\mathcal{L}(X)=\frac{1}{\sqrt{mn}}\|X-\widehat{X}\|_{F}$ and $\hat{\mathcal{L}}(X)=\frac{1}{\sqrt{|\Omega|}}\|\mathcal{P}_{\Omega}(X-\widehat{X})\|_{F}$ be the actual and empirical loss function respectively, where $X,\;\widehat{X}\in \mathbb{R}^{m\times n}\;(m\geq n)$. Furthermore, assume entry-wise constraint $\max_{i,j}|X_{ij}|\leq \delta$. Then for all rank-$r$ matrices $X$, with probability greater than $1-2\exp(-m)$, there exists a fixed constant $C$ such that
\begin{displaymath}
\sup_{X\in S_{r}}|\hat{\mathcal{L}}(X)-\mathcal {L}(X)|\leq C\delta \left(\frac{mr\log(m)}{|\Omega|}\right)^{1/4},
\end{displaymath}
where $S_{r}=\{X\in \mathbb{R}^{m\times n}: rank(X)\leq r, \|X\|_{F}\leq \sqrt{mn}\delta\}$.
\end{lemma}

Suppose that $M=\max_{i,j}(X_{ij}-\widehat{X}_{ij})^{2}\leq(2\delta)^{2}$ and $\epsilon=9\delta$ as in~\cite{wang:mfcf}. According to Theorem 2 in~\cite{wang:mfcf}, thus we have
\begin{equation*}
\begin{split}
&\sup_{X\in S_{r}}|\hat{\mathcal{L}}(X)-\mathcal{L}(X)|\\
\leq&\frac{2\epsilon}{\sqrt{|\Omega|}}+\left(\frac{M^{2}}{2}\frac{2mr\log(9\delta m/\epsilon)}{|\Omega|}\right)^{1/4}\\
\leq& \frac{18\beta_{1}}{\sqrt{|\Omega|}}+2\delta\left(\frac{mr\log(m)}{|\Omega|}\right)^{1/4}\\
=&\left(2+\frac{18}{(|\Omega|mr\log(m))^{1/4}}\right)\delta\left(\frac{mr\log(m)}{|\Omega|}\right)^{1/4}.
\end{split}
\end{equation*}
Therefore, the constant $C$ can be set to $2+\frac{18}{(|\Omega|mr\log(m))^{1/4}}$.\\

\subsection*{Proof of Theorem 8:}
\begin{proof}
\begin{equation*}
\begin{split}
&\frac{\|D-\widehat{U}\widehat{V}^{T}\|_{F}}{\sqrt{mn}}\\
\leq&\left|\frac{\|D-\widehat{U}\widehat{V}^{T}\|_{F}}{\sqrt{mn}}-\frac{\|\mathcal{P}_{\Omega}(D-\widehat{U}\widehat{V}^{T})\widehat{V}\|_{F}}{C_{3}\sqrt{|\Omega|}}\right|+\frac{\|\mathcal{P}_{\Omega}(D-\widehat{U}\widehat{V}^{T})\widehat{V}\|_{F}}{C_{3}\sqrt{|\Omega|}}\\
=&\left|\frac{\|D-\widehat{U}\widehat{V}^{T}\|_{F}}{\sqrt{mn}}-\frac{\|\mathcal{P}_{\Omega}(D-\widehat{U}\widehat{V}^{T})\|_{F}}{\sqrt{|\Omega|}}\right|+\frac{\|\mathcal{P}_{\Omega}(D-\widehat{U}\widehat{V}^{T})\widehat{V}\|_{F}}{C_{3}\sqrt{|\Omega|}}.
\end{split}
\end{equation*}
Let $\tau(\Omega):=\left|\frac{1}{\sqrt{mn}}\|D-\widehat{U}\widehat{V}^{T}\|_{F}-\frac{1}{\sqrt{|\Omega|}}\|\mathcal{P}_{\Omega}(D-\widehat{U}\widehat{V}^{T})\|_{F}\right|$, then we need to bound $\tau(\Omega)$. It is clear that rank$(\widehat{U}\widehat{V}^{T})\leq d$, and thus $\widehat{U}\widehat{V}^{T}\in S_{d}$. According to Lemma \ref{lem10}, then with probability greater than $1-2\exp(-m)$, then there exists a fixed constant $C_{2}=2+\frac{18}{(|\Omega|mr\log(m))^{1/4}}$ such that
\begin{equation}\label{Th81}
\begin{split}
\sup_{\hat{U}\hat{V}^{T}\in S_{d}}\tau(\Omega)=&\left|\frac{\|\widehat{U}\widehat{V}^{T}-D\|_{F}}{\sqrt{mn}}-\frac{\|\mathcal{P}_{\Omega}(\widehat{U}\widehat{V}^{T})-\mathcal{P}_{\Omega}(D)\|_{F}}{\sqrt{|\Omega|}}\right|\\
\leq & C_{2} \delta\left(\frac{md\log(m)}{|\Omega|}\right)^{\frac{1}{4}}.
\end{split}
\end{equation}

We also need to bound $\|\mathcal{P}_{\Omega}(\widehat{U}\widehat{V}^{T}-D)\widehat{V}\|_{F}$. Given $\widehat{V}$, the optimization problem with respect to $U$ is formulated as follows:
\begin{equation}\label{Th82}
\min_{U} \frac{\|U\|_{\textup{tr}}}{2}+\frac{1}{2\mu}\|\mathcal{P}_{\Omega}(U\widehat{V}^{T})-\mathcal{P}_{\Omega}(D)\|^{2}_{F}.
\end{equation}

Since $(\widehat{U}, \widehat{V})$ is a KKT point of the problem (15), the first-order optimality condition for the problem \eqref{Th82} is given by
\begin{equation}\label{Th83}
\mathcal{P}_{\Omega}(D-\widehat{U}\widehat{V}^{T})\widehat{V}\in \frac{\mu}{2}\partial\|\widehat{U}\|_{\textup{tr}}.
\end{equation}

Using Lemma \ref{Conv L1}, we obtain
\begin{displaymath}
\|\mathcal{P}_{\Omega}(\widehat{U}\widehat{V}^{T}-D)\widehat{V}\|_{2}\leq \frac{\mu}{2},
\end{displaymath}
where $\|X\|_{2}$ is the spectral norm of a matrix $X$. Recall that $\textrm{rank}(\mathcal{P}_{\Omega}(\widehat{U}\widehat{V}^{T}-D)\widehat{V})\leq d$, we have
\begin{equation}\label{Th84}
\|\mathcal{P}_{\Omega}(\widehat{U}\widehat{V}^{T}-D)\widehat{V}\|_{F}\leq \sqrt{d}\|\mathcal{P}_{\Omega}(\widehat{U}\widehat{V}^{T}-D)\widehat{V}\|_{2} \leq \frac{\sqrt{d}\mu}{2}.
\end{equation}
By \eqref{Th81} and \eqref{Th84}, we have
\begin{equation*}
\begin{split}
\frac{\|X_{0}-\widehat{U}\widehat{V}^{T}\|_{F}}{\sqrt{mn}}\leq & \frac{\|E\|_{F}}{\sqrt{mn}}+\frac{\|D-\widehat{U}\widehat{V}^{T}\|_{F}}{\sqrt{mn}}\\
\leq & \frac{\|E\|_{F}}{\sqrt{mn}}+\tau(\Omega)+\frac{\|\mathcal{P}_{\Omega}(D-\widehat{U}\widehat{V}^{T})\widehat{V}\|_{F}}{C_{3}\sqrt{|\Omega|}}\\
\leq & \frac{\|E\|_{F}}{\sqrt{mn}}+C_{2}\delta\left(\frac{md\log(m)}{|\Omega|}\right)^{\frac{1}{4}}+\frac{\sqrt{d}\mu}{2C_{3}\sqrt{|\Omega|}}.
\end{split}
\end{equation*}
This completes the proof.
\end{proof}

\subsection*{Lower bound on $C_{3}$}
Finally, we also discuss the lower boundedness of $C_{3}$, that is, it is lower bounded by a positive constant. Let $Q=\mathcal{P}_{\Omega}(D-\widehat{U}\widehat{V}^{T})\widehat{V}$, and by \eqref{Th83}, we have that $Q\in \frac{\mu}{2}\partial\|\widehat{U}\|_{\textup{tr}}$. By \eqref{Th74}, we obtain
\begin{displaymath}
\left\langle \frac{2}{\mu}Q, \,\widehat{U}\right\rangle=\|\widehat{U}\|_{\textup{tr}}.
\end{displaymath}
Note that $\|A\|_{\textup{tr}}\geq \|A\|_{F}$ and $\langle A,B\rangle\leq\|A\|_{F}\|B\|_{F}$ for any matrices $A$ and $B$ of the same size.
\begin{displaymath}
\frac{2}{\mu}\|Q\|_{F}\|\widehat{U}\|_{F}\geq\left\langle \frac{2}{\mu}Q, \,\widehat{U}\right\rangle=\|\widehat{U}\|_{\textup{tr}}\geq \|\widehat{U}\|_{F}.
\end{displaymath}
Recall that $\|\widehat{U}\|_{F}>0$ and $\mu\neq0$, thus we obtain
\begin{displaymath}
\|\mathcal{P}_{\Omega}(D-\widehat{U}\widehat{V}^{T})\widehat{V}\|_{F}=\|Q\|_{F}\geq \frac{\mu}{2}.
\end{displaymath}

Since $\widehat{U}$ is the optimal solution of the problem \eqref{Th82} with given $\widehat{V}$, then
\begin{equation*}
\frac{1}{2\mu}\|\mathcal{P}_{\Omega}(D-\widehat{U}\widehat{V}^{T})\|^{2}_{F}<\frac{1}{2\mu}\|\mathcal{P}_{\Omega}(D-\widehat{U}\widehat{V}^{T})\|^{2}_{F}+\frac{1}{2}\|\widehat{U}\|_{\textup{tr}}\leq\frac{1}{2\mu}\|\mathcal{P}_{\Omega}(D)\|^{2}_{F}=\gamma,
\end{equation*}
where $\gamma>0$ is a constant. Hence,
\begin{displaymath}
C_{3}=\frac{\|\mathcal{P}_{\Omega}(D-\widehat{U}\widehat{V}^{T})\widehat{V}\|_{F}}{\|\mathcal{P}_{\Omega}(D-\widehat{U}\widehat{V}^{T})\|_{F}}>\frac{\sqrt{\mu}}{2\sqrt{2\gamma}}.
\end{displaymath}

In fact, the value of $C_{3}$ is much greater than its lower bound, $\frac{\sqrt{\mu}}{2\sqrt{2\gamma}}$, as shown in Figure \ref{fig_sim1}, where the ordinate is the average results over 100 independent runs, and the abscissa denotes the sampling ratio, which is chosen from $\{0.002, 0.005,0.01,0.05,0.1,\ldots,0.95,0.99,0.995,0.999\}$. Moreover, the regularization parameter $\mu$ is set to 5 and 100 for noisy matrices ($nf\!=\!0.1$) and noiseless matrices, respectively.

\begin{figure}[t]
\centering
\subfigure[Matrices of size $100\!\times\!100$]{\includegraphics[width=0.46\linewidth]{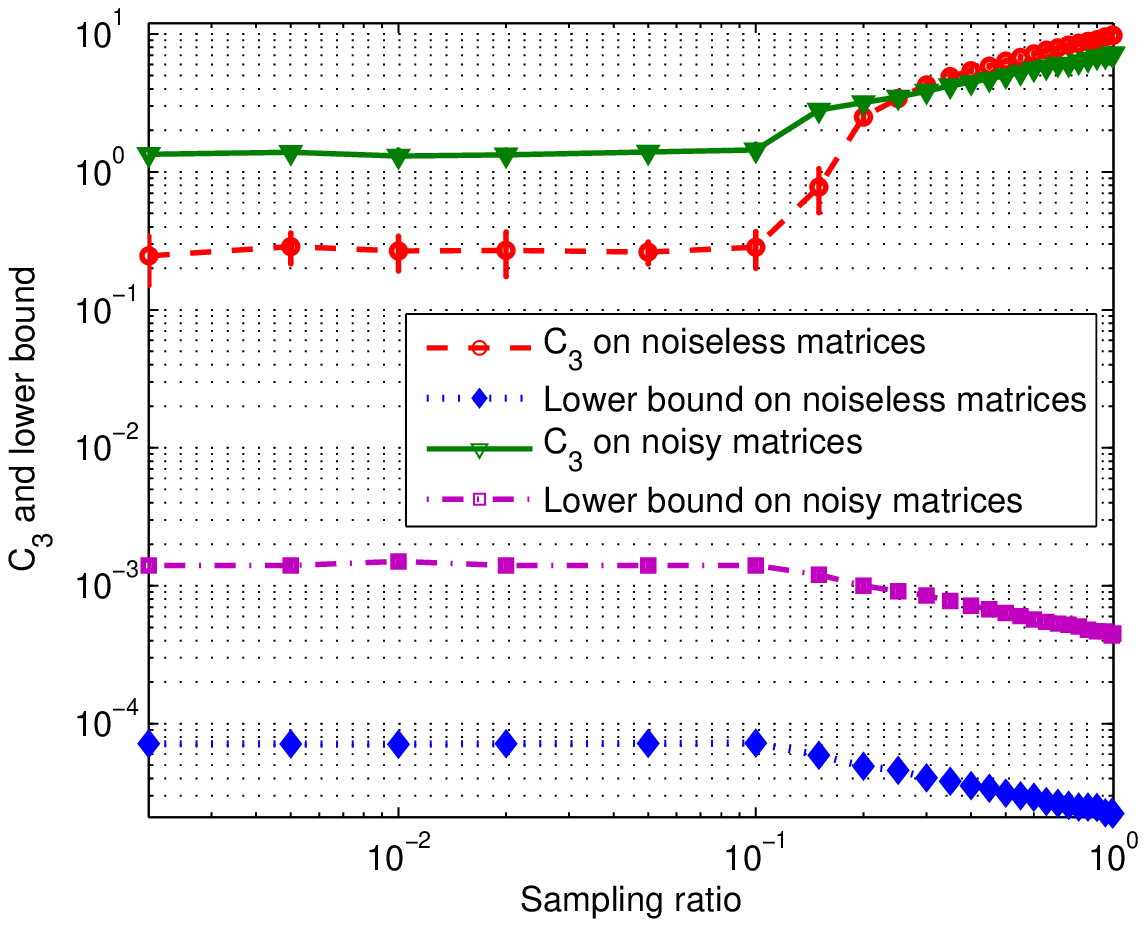}}\;\;
\subfigure[Matrices of size $200\!\times\!200$]{\includegraphics[width=0.46\linewidth]{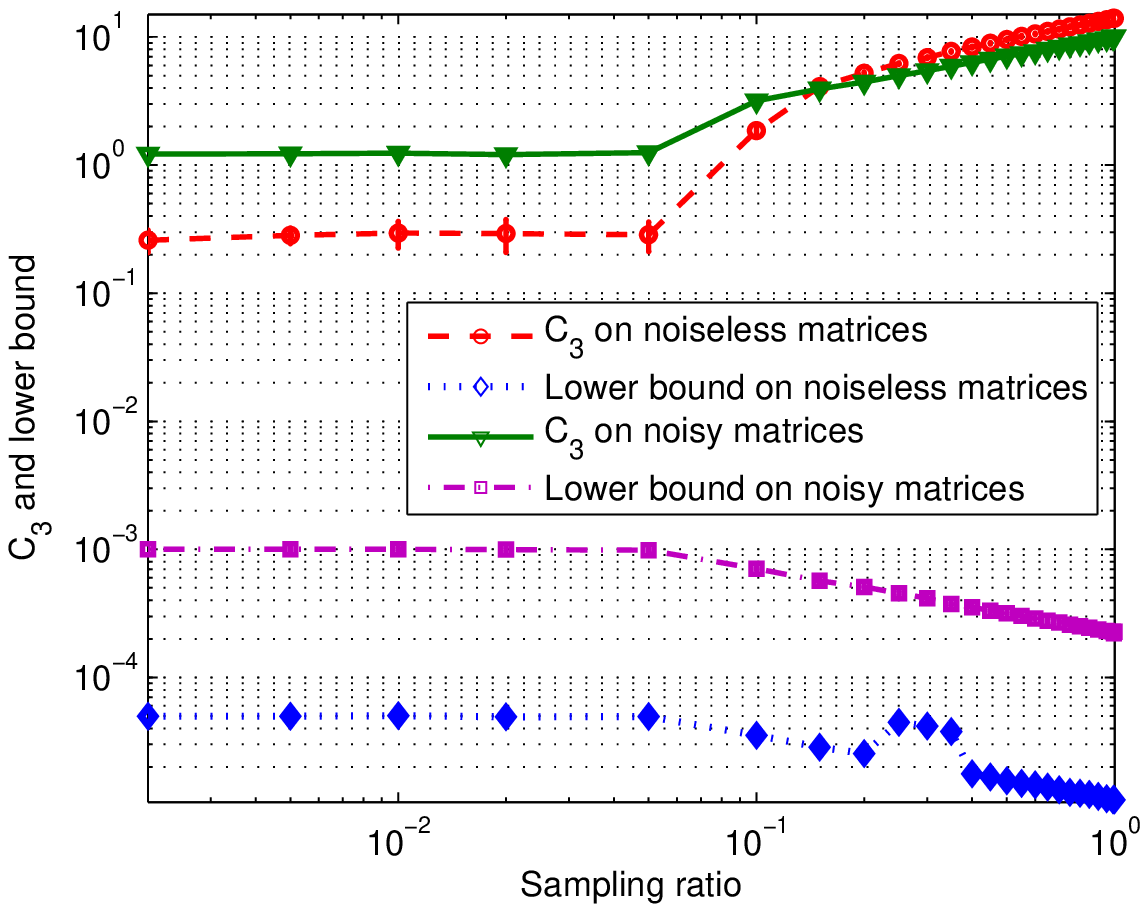}}
\caption{Average results (mean and std.) of $C_{3}$ vs.\ sampling ratio. For a fixed sampling ratio, we can observe that although $C_{3}$ is not a constant, the value of $C_{3}$ is stable with respect to the projection operator $\mathcal{P}_{\Omega}$ (Best viewed zoomed in).}
\label{fig_sim1}
\end{figure}

\section{Complexity Analysis}
For MC and RPCA problems, the running time of our PALM and LADM algorithms is mainly consumed in performing some matrix multiplications. The time complexity of some multiplications operators is $O(mnd)$. In addition, the time complexity of performing SVD on matrices of the same sizes as $U_{k}$ and $V_{k}$ is $O(md^{2}+nd^{2})$. In short, the total time complexity of our PALM and LADM algorithms is $O(nmd)$ $(d\ll m,n)$. Moreover, it is known that the parallel matrix multiplication on multicore architectures can be efficiently implemented. Thus, in practice our PALM and LADM algorithms are fast and scales well to handle large-scale problems.

\section{More Experimental Results}
For the MC problem, e.g., synthetic matrix completion and collaborative filtering, we propose an efficient proximal alternating linearized minimization (PALM) algorithm to solve (15), and then extend it to solve the Tri-tr quasi-norm regularized matrix completion problem.

In the following, we present our bi-trace quasi-norm minimization models for the RPCA problems (e.g., the text separation task):
\begin{equation}\label{exp1}
\min_{U,\,V}\, \frac{1}{2}(\|U\|_{\textup{tr}}+\|V\|_{\textup{tr}})+\frac{1}{\mu}\|\mathcal{P}_{\Omega}(D-UV^{T})\|_{1},
\end{equation}
and
\begin{equation}\label{exp2}
\min_{U,\,V}\, \frac{1}{2}(\|U\|_{\textup{tr}}+\|V\|_{\textup{tr}})+\frac{1}{\mu}\|\mathcal{P}_{\Omega}(D-UV^{T})\|^{1/2}_{1/2},
\end{equation}
where $\mathcal{P}_{\Omega}$ denotes the linear projection operator, i.e., $\mathcal{P}_{\Omega}(D)_{ij}=D_{ij}$ if $(i,j)\in\Omega$, and $\mathcal{P}_{\Omega}(D)_{ij}=0$ otherwise. Similar to \eqref{exp1}, the Tri-tr quasi-norm penalty can also be used to the RPCA problem.

To efficiently solve the RPCA problems \eqref{exp1} and \eqref{exp2}, we also need to introduce an auxiliary variable $E$ (as the same role as $e$), and can assume, without loss of generality, that the unknown entries of $D$ are simply set as zeros, i.e., $D_{\Omega^{C}}={0}$, and $E_{\Omega^{C}}$ may be any values such that $\mathcal{P}_{\Omega^{C}}(D)=\mathcal{P}_{\Omega^{C}}(UV^{T})+\mathcal{P}_{\Omega^{C}}(E)$. Therefore, the RPCA problems \eqref{exp1} and \eqref{exp2} are reformulated as follows:
\begin{equation}\label{exp3}
\min_{U,\,V,\,E}\, \frac{1}{2}(\|U\|_{\textup{tr}}+\|V\|_{\textup{tr}})+\frac{1}{\mu}\|\mathcal{P}_{\Omega}(E)\|_{1},\;\textup{s.t.},\,UV^{T}+E=D,
\end{equation}
\begin{equation}\label{exp4}
\min_{U,\,V,\,E}\, \frac{1}{2}(\|U\|_{\textup{tr}}+\|V\|_{\textup{tr}})+\frac{1}{\mu}\|\mathcal{P}_{\Omega}(E)\|^{1/2}_{1/2},\;\textup{s.t.},\,UV^{T}+E=D.
\end{equation}

In fact, we can apply directly Algorithm 1 with the soft-thresholding operator~\cite{daubechies:it} to solve \eqref{exp3}. In contrast, for solving \eqref{exp4} we can update $E$ via solving the following problem with Lagrange multiplier $Y_{k}$,
\begin{equation}\label{exp5}
\min_{E}\,\frac{1}{\mu}\|\mathcal{P}_{\Omega}(E)\|^{1/2}_{1/2}+\frac{\beta_{k}}{2}\|E-M_{k+1}\|^{2}_{F},
\end{equation}
where $M_{k+1}=D-U_{k+1}V^{T}_{k+1}-Y_{k}/\beta_{k}$. In general, the $\ell_{p}$ ($0\!<\!p\!<\!1$) quasi-norm leads to a non-convex, non-smooth, and non-Lipschitz optimization problem~\cite{bian:ipa}. Fortunately, we introduce the following half-thresholding operator in~\cite{krishnan:hlp, zeng:l12} to solve \eqref{exp5}.

\begin{lemma}\label{lem11}
Let $y=(y_{1},y_{2},\ldots,y_{n})^{T}$, and $x^{*}=(x^{*}_{1},x^{*}_{2},\ldots,x^{*}_{n})^{T}$ be an $\ell_{1/2}$ quasi-norm solution of the following minimization
\begin{equation*}
\min_{x}\,\|y-x\|^{2}_{2}+\lambda\|x\|^{1/2}_{1/2},
\end{equation*}
then the solution $x^{*}$ can be given by $x^{*}=H_{\lambda}(y)$, where the half-thresholding operator $H_{\lambda}(\cdot)$ is defined as
\begin{equation*}
H_{\lambda}(y_{i})=\left\{ \begin{array}{ll}
\frac{2}{3}y_{i}[1+\cos(\frac{2\pi}{3}-\frac{2\phi_{\lambda}(y_{i})}{3})], &|y_{i}|> \frac{\sqrt[3]{54}}{4}\lambda^{\frac{2}{3}},\\
0, &\textup{otherwise},
\end{array}\right.
\end{equation*}
where $\phi_{\lambda}(y_{i})=\arccos(\frac{\lambda}{8}({|y_{i}|}/{3})^{-{3}/{2}})$.
\end{lemma}
By Lemma \ref{lem11}, the closed-form solution of \eqref{exp5} is given by
\begin{equation*}
(E_{k+1})_{i,j}=\left\{ \begin{array}{ll}
H_{2/{(\mu\beta_{k})}}((M_{k+1})_{i,j}),& (i,j)\in\Omega,\\
(M_{k+1})_{i,j}, & \textup{otherwise}.\\
\end{array}\right.
\end{equation*}

\subsection{Implementation Details for Comparison}
We present the implementation detail for all other algorithms in our comparison. The code for NNLS is downloaded from~\url{http://www.math.nus.edu.sg/~mattohkc/NNLS.html}. For ALT, the given rank is set to $d=\lfloor1.25r\rfloor$ for synthetic data and 50 for four recommendation system datasets, and the maximum number of iterations is set to 100 and 50, respectively, and its code is downloaded from~\url{http://www.cs.utexas.edu/~cjhsieh/}. For LRMF, the code is downloaded from~\url{http://ttic.uchicago.edu/~ssameer/#code}, and the IRLS code is downloaded from~\url{http://www.math.ucla.edu/~wotaoyin/papers/improved_matrix_lq.html}. The given rank of LRMF and IRLS is set to the same value as our algorithms, e.g., $d=\lfloor1.25r\rfloor$ for synthetic data. In addition, the code for IRNN is downloaded from~\url{https://sites.google.com/site/canyilu/}. For LMaFit, the code is downloaded from~\url{http://lmafit.blogs.rice.edu/}, and the $\textup{S}_{p}$+$\ell_{p}$ code is downloaded from~\url{https://sites.google.com/site/feipingnie/publications}. Note that the regularization parameter $\mu$ is generally set to $\sqrt{\max(m,n)}$ as suggested in~\cite{candes:rpca}. All the experiments were conducted on an Intel Xeon E7-4830V2 2.20GHz CPU with 64G RAM.

\begin{figure}[t]
\begin{center}
\subfigure[Noisy matrices of size $100\!\times\!100$]{\includegraphics[width=0.432\columnwidth]{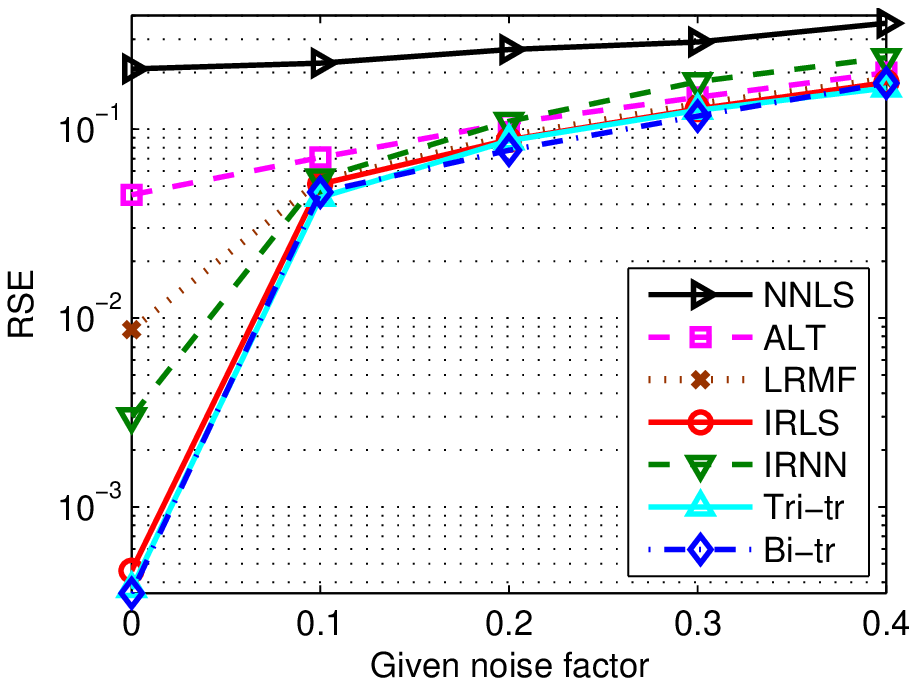}}\;\;
\subfigure[Noisy matrices of size $200\!\times\!200$]{\includegraphics[width=0.432\columnwidth]{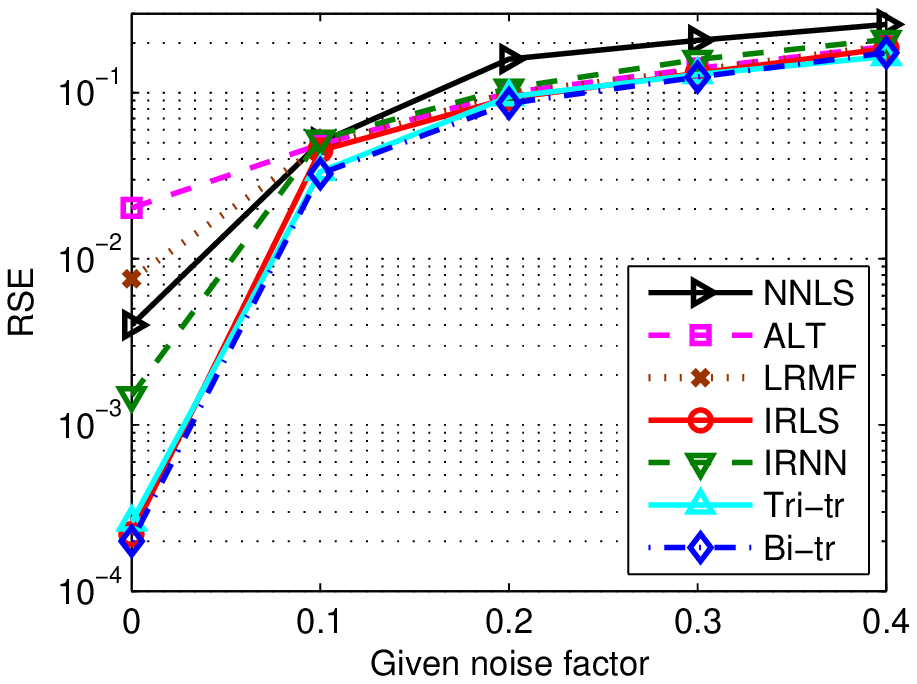}}
\caption{The recovery RSE results of NNLS, ALT, LRMF, IRLS, IRNN, and our Tri-tr and Bi-tr methods on noisy random matrices with different noise levels.}
\label{fig_sim2}
\end{center}
\end{figure}

\begin{figure}[t]
\begin{center}
{\includegraphics[width=0.50\columnwidth]{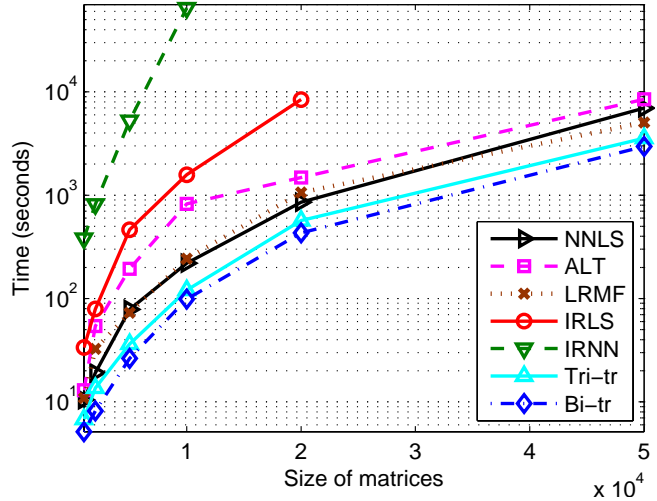}}
\caption{The running time of NNLS, ALT, LRMF, IRLS, IRNN, and our Tri-tr and Bi-tr methods as the size of noisy random matrices increases.}
\label{fig_sim3}
\end{center}
\end{figure}

\subsection{Synthetic Data}
In order to evaluate the robustness of our methods against noise, we generated the noisy input by the following
procedure~\cite{shang:snm}:
\begin{displaymath}
b=\mathcal{A}(X_{0}+nf*\Theta),
\end{displaymath}
where the elements of the noise matrix $\Theta$ are i.i.d.\ standard Gaussian random variables, and $nf$ is the given noise factor.

We also conduct some experiments on noisy matrices of size $100\times100$ or $200\times200$ with different noise factors, and report the RSE results of all algorithms with 20\% SR in Figure \ref{fig_sim2}. It is clear that ALT and LRMF have very similar performance, and usually outperform NNLS in terms of RSE. Moreover, the recovery performance of both our methods are similar to that of IRLS and IRNN, and they consistently perform better than the other methods. Moreover, we also present the running time of all those methods with 20\% SR as the size of noisy random matrices increases, as shown in Figure \ref{fig_sim3}. We can observe that the running time of IRNN and IRLS increases dramatically when the size of matrices increases, and they could not yield experimental results within 48 hours when the size of matrices is $50,000\times50,000$. On the contrary, both our methods are much faster than the other methods. This further justifies that both our methods have very good scalability and can address large-scale problems. As NNLS uses the PROPACK package~\cite{larsen:svd} to compute a partial SVD in each iteration, it usually runs slightly faster than ALT.

\begin{table}[t]
\centering
\caption{Characteristics of the recommendation datasets.}
\begin{tabular}{l|rrr}
\hline
{Dataset} & {\# row}    &{\# column} & {\# rating}\\
\hline
{MovieLens1M}   &6,040   &3,906    &1,000,209\\
{MovieLens10M}  &71,567  &10,681   &10,000,054\\
{MovieLens20M}  &138,493 &27,278   &20,000,263\\
{Netflix}       &480,189 &17,770   &100,480,507\\
\hline
\end{tabular}
\label{tab_sim1}
\end{table}

\begin{table}[t]
\centering
\caption{Regularization parameter settings for different algorithms.}
\begin{tabular}{l|cccccc}
\hline
\multirow{2}{*}{Datasets} &{NNLS}  &{ALT} &{LRMF}  &{LMaFit}  &{IRLS} &{Ours}\\
                          &{$\mu$} &{$\lambda$} &{$\lambda$} &{$\lambda$}  &{$\lambda$} &{$\mu$}\\
\hline
{MovieLens1M}     &1.70  &50   &5     &--    &1e-6     &100\\
{MovieLens10M}    &4.80  &100  &5     &--    &1e-6     &100\\
{MovieLens20M}    &6.63  &150  &5     &--    &1e-6     &100\\
{Netflix}         &16.76 &150  &5     &--    &1e-6     &100\\
\hline
\end{tabular}
\label{tab_sim2}
\end{table}

\begin{figure}[t]
\begin{center}
{\includegraphics[width=0.65\columnwidth]{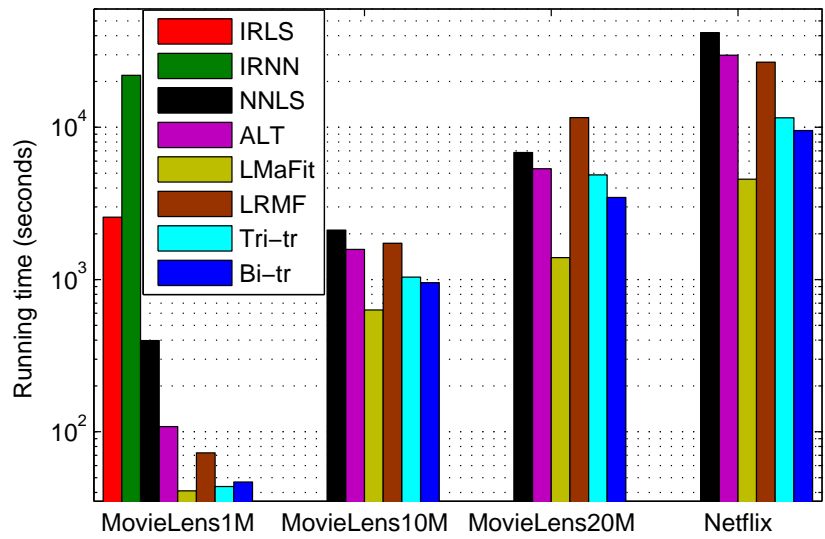}}
\caption{Running time (seconds) for comparison on the four data sets (Best viewed in color).}
\label{fig_sim4}
\end{center}
\end{figure}

\subsection{Real-World Recommendation System Data}
In this part, we present the detailed descriptions for four real-world recommendation system data sets and the detailed regularization parameter settings for different algorithms, as shown in Table \ref{tab_sim1} and Table \ref{tab_sim2}, respectively. For IRNN, the regularization parameter $\lambda$ is dynamically decreased by $\lambda_{k}\!=\!0.7\lambda_{k-1}$, where $\lambda_{0}\!=\!10\|\mathcal{P}_{\Omega}(D)\|_{\infty}$. We also report the running time of all these algorithms on the four data sets, as shown in Figure \ref{fig_sim4}, from which it is clear that both our methods are much faster than the other methods, except LMaFit. Compared with LMaFit, both our methods remain competitive in speed, but achieve much lower RMSE than LMaFit on all the data sets (as shown in Figure~2 in the main paper).


\end{document}